\documentclass[11pt, oneside]{article}  	

\usepackage[margin=1in]{geometry}        		
\usepackage{graphicx,multirow}				
\usepackage[page]{appendix}								
\usepackage{amssymb,amsmath,amsfonts,amssymb,hyperref}

\usepackage[svgnames]{xcolor}
\hypersetup{
    colorlinks,
    citecolor=SeaGreen,
    filecolor=black,
    linkcolor=Indigo,
    urlcolor=black
}

\usepackage{dsfont}
\usepackage{natbib}
\usepackage{booktabs}
\usepackage{setspace}
\usepackage{tikz}

\usepackage{algorithm2e}
\RestyleAlgo{ruled}

\usepackage[position=bottom]{subfig}
\usepackage{url}
\usepackage[table]{xcolor}
\makeatletter
\newcounter{peq}
\renewcommand{\thepeq}{P\arabic{peq}}

  {\tag*{\thepeq}\end{equation}}

\g@addto@macro{\UrlBreaks}{\UrlOrds}
\makeatother

\DeclareMathOperator*{\argmax}{\arg\!\max}
\definecolor{lightgray}{gray}{0.9}

\usepackage{url}
\usepackage{graphicx}
\usepackage{amssymb,amsmath,enumerate}
\usepackage{amsthm}
\graphicspath{{pics/}{figs/}}
\usepackage{afterpage}
\usepackage{textcomp}
\RequirePackage{amsmath} \RequirePackage{xspace}

\theoremstyle{plain}
\newtheorem{theorem}{{Theorem}}[section] 
\newtheorem*{theorem*}{{Theorem}}
\newtheorem{proposition}[theorem]{Proposition}
\newtheorem*{proposition*}{Proposition}

\newtheorem*{corollary*}{Corollary}
\newtheorem{lemma}[theorem]{Lemma}
\newtheorem*{lemma*}{Lemma}
\newtheorem{assumption}[theorem]{Assumption}
\newtheorem*{assumption*}{Assumption}

\newtheorem*{definition*}{Definition}

\theoremstyle{remark}
\newtheorem*{notation*}{Notation}
\newtheorem*{remark*}{Remark}

\def\argmax{\mathop{\mathrm{arg\,max}}} 
\def\argmin{\mathop{\mathrm{arg\,min}}} 
\RequirePackage{amsmath} \RequirePackage{xspace}
\RequirePackage{bbm}

\def\XS{\xspace}
\DeclareMathAlphabet{\mathb}{OML}{cmm}{b}{it}
\def\sbm#1{\ensuremath{\mathb{#1}}}                
\def\sbmm#1{\ensuremath{\boldsymbol{#1}}}          

\def\Ab{{\sbm{A}}\XS}  \def\ab{{\sbm{a}}\XS}
  \def\bb{{\sbm{b}}\XS}

\def\Eb{{\sbm{E}}\XS}  
\def\Fb{{\sbm{F}}\XS}

\def\Ib{{\sbm{I}}\XS}  
\def\Jb{{\sbm{J}}\XS}

\def\Mb{{\sbm{M}}\XS}
\def\mb{{\sbm{m}}\XS}

  \def\qb{{\sbm{q}}\XS}
  \def\rb{{\sbm{r}}\XS}
\def\ESSb{{\sbm{S}}\XS}  
\def\Tb{{\sbm{T}}\XS}  
  \def\ub{{\sbm{u}}\XS}

\def\gammab      {{\sbmm{\gamma}}\XS}

\def\thetab      {{\sbmm{\theta}}\XS}

\def\lambdab     {{\sbmm{\lambda}}\XS}    \def\Lambdab   {{\sbmm{\Lambda}}\XS}
\def\mub         {{\sbmm{\mu}}\XS}
\def\nub         {{\sbmm{\nu}}\XS}
\def\xib         {{\sbmm{\xi}}\XS}                 
\def\pib         {{\sbmm{\pi}}\XS}

      \def\Sigmab    {{\sbmm{\Sigma}}\XS}

        \def\Psib       {{\sbmm{\Psi}}\XS}

\def\PM{\kern0pt^{\textrm{{\scriptsize PM}}}\kern0pt}
\def\MMAP{\kern1pt^{\textrm{{\tiny MMAP}}}\kern-1pt}

\newcommand{\KL}{\mathrm{KL}} 
\newcommand{\yv}{\ensuremath{\mathbf{y}}} 

\def\Rset{\mathbb{R}} 
\def\wp1{\mathrm{w.p.} 1}  

\usepackage{titletoc}

\title{Natural Variational Annealing for Multimodal Optimization}
\author{T\^am Le Minh$^{1}$, Julyan Arbel$^{1}$, Thomas M\"ollenhoff$^{2}$, \\ Mohammad Emtiyaz Khan$^{2}$, Florence Forbes$^{1}$}
\date{}							

\begin{document}
\maketitle

\begin{center}
\textit{\noindent
$^{1}$ Univ. Grenoble Alpes, Inria, CNRS, Grenoble INP, LJK, 38000 Grenoble, France \\
$^{2}$ RIKEN Center for Advanced Intelligence Project, 1-4-1 Nihonbashi, Chuo-ku, Tokyo 103-0027, Japan
}
\end{center}

\paragraph{Abstract.} We introduce a new multimodal optimization approach called Natural Variational Annealing (NVA) that combines the strengths of three foundational concepts to simultaneously search for multiple global and local modes of black-box nonconvex objectives. First, it implements a simultaneous search by using variational posteriors, such as, mixtures of Gaussians. Second, it applies annealing to gradually trade off exploration for exploitation. Finally, it learns the variational search distribution using natural-gradient learning where updates resemble well-known and easy-to-implement algorithms. The three concepts come together in NVA giving rise to new algorithms and also allowing us to incorporate ``fitness shaping'', a core concept from evolutionary algorithms. We assess the quality of search on simulations and compare them to methods using gradient descent and evolution strategies. We also provide an application to a real-world inverse problem in planetary science.


\section{Introduction}

Multimodal or multisolution optimization involves finding all global optima of an objective function $\ell$, as well as high-quality local optima.
We consider $\ell : \mathbb{R}^d \rightarrow \mathbb{R}$ to be a smooth and bounded function, where $d$ is a positive integer and assume that it has a finite number $I$ of global modes (maxima), denoted by  $\xib_{1}^*, \dots, \xib^*_{I}$ that satisfy
\begin{equation}
    \{\xib^*_1, \dots, \xib^*_I\} = \argmax_{\xib \in \mathbb{R}^d}~\ell(\xib).
    \tag{P1}
    \label{eq:simple_problem}
\end{equation}
An explicit expression for $\ell$ or its gradient may not always be available, that is, the objective can be a black-box, but we assume that it is possible to evaluate $\ell(\xib)$ at arbitrary locations $\xib \in \mathbb{R}^d$. Many real-world problems can be formulated as multimodal optimization \citep{li2016seeking} when there are benefits obtained by finding multiple modes, as opposed to just one mode. The diversity of modes can help us improve the decision-making process and also reveal hidden characteristics of the problem.

Classical optimization methods are not specifically designed for multimodal optimization, including those used in machine learning to navigate non-concave landscapes through stochastic mini-batching \citep{robbins1951stochastic, duchi2011adaptive, tieleman2012lecture, kingma2014adam}. These can still be used for multimodal optimization by using multiple restarts \citep{rinnooy1987stochastic, lee2011variable, regis2013quasi, custodio2015glods, larson2018asynchronously}, for example, by restarting with new initial conditions to avoid previously explored regions, or by running multiple parallel instances that communicate to avoid search overlaps. Unfortunately, such solutions often rely on ad hoc heuristics which can be disconnected from principles of exploration-exploitation and also often lack theoretical foundations.

A principled way to explore is through a search distribution in stochastic optimization, for example, through the use of a Gibbs distribution \citep{geman1984stochastic} and simulated annealing to target global modes by slowly decreasing the temperature \citep{kirkpatrick1983optimization}. However, a Gibbs distribution is often infeasible to compute, and sampling from it can be expensive. Cheap alternatives rely on variational approximations \citep{staines2012variational, gershman2012nonparametric}, annealing of which is referred to as \emph{variational annealing} \citep{katahira2008deterministic, yoshida2010bayesian, mandt2016variational, huang2018improving, tao2019variational, dangelo2021annealed, sanokowski2023variational}. However, such works often focused on Bayesian models and applications to multimodal optimization are rare, if any.

Evolutionary algorithms are another mechanism to explore \citep{back1997handbook, hansen2001completely, wierstra2014natural, ollivier2017information}. In particular, particle swarms \citep{kennedy1995particle, brits2007locating,Liang2006,li2007multimodal, qu2012distance, yamanaka2021simple} have been widely tested in multimodal optimization. Evolutionary algorithms iteratively generate populations of particles to explore the search space. They simulate processes of natural evolution like mutation and selection. As a result, the particles retain memories of previously explored regions and use these to explore new regions. State-of-the-art methods combine many different heuristics, most notably niching techniques, which partition the search space to maintain exploration diversity \citep{ahrari2017multimodal, maree2018real, denobel2024avoiding}. 

Our goal in this paper is to propose a framework that enables us to combine the strengths of the various approaches discussed above. Specifically we look for three desirable properties: (1) We want an algorithm that can simultaneously search different regions, for example, by using a search distribution; (2) We also want to be able to trade off exploration vs exploitation by using mechanisms like annealing; (3) But, we want the algorithm to be cheap and easy to implement, similarly to algorithms such as Stochastic Gradient Descent (SGD), Adam and Newton's method, used in machine learning for single-mode optimization. We aim to propose a new framework that helps us achieve all the above properties.

\paragraph{Contributions.\label{contrib}} We introduce a new multimodal optimization approach called \textit{Natural Variational Annealing} (NVA) that realizes the three desirable properties by combining the strengths of three foundational concepts: 
\begin{enumerate}
\itemsep0mm 
    \item[1.~] \textit{Variational approximation} to enable simultaneous search for multiple modes. This is achieved by using a multimodal search distribution $q_\lambdab(\xib)$ and aiming for a relaxed objective $\mathbb{E}_{q_\lambdab}[\ell(\xib)]$ over the variational parameters $\lambdab$. Simultaneous search is made possible by using mixture distributions where each mixing component focuses on a different region in the search space.
    \item[2.~] \textit{Entropy annealing} to introduce exploration-exploitation trade-off by optimizing 
        \begin{equation}
        \mathcal{L}_\omega(\lambdab) = \mathbb{E}_{q_\lambdab}[\ell(\xib)] + \omega \mathcal{H}(q_\lambdab),
        \label{eq:variational_objective_intro}
    \end{equation}
    where $\mathcal{H}(q_\lambdab) = - \int q_\lambdab(\xib) \log q_\lambdab(\xib) d\xib$ is the entropy of the search distribution $q_\lambdab$. Similarly to simulated annealing, larger values of the temperature parameter $\omega>0$ facilitate more exploration.  
    \item[3.~] \textit{Natural-gradient ascent} method of \cite{khan2023bayesian} called the Bayesian learning rule to obtain cheap and easy-to-implement algorithms to optimize \eqref{eq:variational_objective_intro}:
    \begin{equation}
        \lambdab_{t+1} = \lambdab_t + \rho_t \Fb(\lambdab_t)^{-1} \nabla_\lambdab \mathcal{L}_{\omega_t}(\lambdab_t).
    \end{equation}
    Here, $\lambdab_t$ denotes the natural-parameters of an exponential-family distribution $q_{\lambdab_t}$ in iteration $t$, $\Fb(\lambdab_t)$ is the Fisher Information Matrix (FIM), $\rho_t>0$ is a learning rate, and $\omega_t > 0$ is an annealing schedule. For mixture distributions, we use the variants proposed by \cite{lin2019fast}. 
\end{enumerate}
While existing methods have combined pairs of these concepts, like variational approximation with natural-gradient optimization \citep{wierstra2014natural, ollivier2017information, lin2019fast, osawa2019practical, khan2023bayesian} or with entropy annealing \citep{katahira2008deterministic, yoshida2010bayesian, mandt2016variational, huang2018improving, tao2019variational, sanokowski2023variational}, to our knowledge, the three concepts have not previously been used simultaneously. Apart from combining the strengths of three foundational concepts, an additional advantage of NVA is that it allows us to incorporate the concept of \textit{fitness shaping} used in evolutionary algorithms \citep{wierstra2014natural} and Information-Geometric Optimization (IGO, \cite{ollivier2017information}).

We propose variants of NVA with mixture distributions (referred to as NVA-M). Specifically, we use the variant with Gaussian mixtures (referred to as NVA-GM) to identify high-quality local solutions. We typically choose the number of components to be much larger than the expected number of modes in the objective, and intentionally let some components get stuck at local solutions. This helps us to identify multiple good local solutions and understand the properties of the landscape. We show several such use cases of our method through simulations and a real-world example. A Python implementation of NVA-GM is freely available\footnote{\href{https://github.com/tam-leminh/natural-variational-annealing}{github.com/tam-leminh/natural-variational-annealing}}.

\paragraph{Outline.} In Section~\ref{sec:gmm_vi}, we present Gibbs measures as solutions to variational problems, and we explain how they can be used to tackle the multimodal optimization problem~\eqref{eq:simple_problem}. Section~\ref{sec:optimization} describes the NVA approach, building first upon the natural gradient ascent approach to solving variational problems, then incorporating annealing. Section~\ref{sec:mego} discusses the specific case of Gaussian mixtures search distributions, giving rise to the NVA-GM algorithm. Section~\ref{sec:fs_mego} shows the incorporation of fitness shaping and presents a variant FS-NVA-GM (fitness-shaped NVA-GM). In Section~\ref{sec:simulations}, we present simulation results and compare them to Stochastic Gradient Ascent (SGA) and Covariance Matrix Adaptation Evolution Strategy (CMA-ES). Section~\ref{sec:illustration} shows an application to a real-world inverse problem in planetary science.

\section{Variational Annealing of Gibbs measures}
\label{sec:gmm_vi}

We start by discussing the variational formulation to anneal Gibbs measures and use it to capture multiple modes of the objective function. Throughout, we will make standard assumptions regarding the function $\ell : \mathbb{R}^d \rightarrow \mathbb{R}$. That is, we assume it to be non-concave with $I$ global modes at locations $\xib^*_1, \dots, \xib^*_I$, twice continuously differentiable, and also that its exponential is bounded, that is, $\int \exp(\ell(\xib)/\omega) d\xib < \infty$, for all $\omega > 0$.  

A Gibbs measure for the function $\ell$ takes the following form,
\begin{equation}
    g_\omega(\xib) = \frac{\exp(\ell(\xib)/\omega)}{\int_{\mathbb{R}^d} \exp(\ell(\xib)/\omega) d\xib}.
    \label{eq:gibbs_measure}
\end{equation}
which is also the solution to the following entropy-regularized variational problem
\begin{equation}
    \argmax_{q \in \mathcal{P}(\mathbb{R}^d)}~\mathbb{E}_q[\ell(\xib)] + \omega \mathcal{H}(q),
    \tag{P2}
    \label{eq:problem_fixed_omega}
\end{equation} 
The result follows due to the strict concavity of the entropy function (Proposition~\ref{prop:kl_convexity} in Appendix~\ref{app:kl_convexity}). A formal statement is given below and a proof is included in Appendix~\ref{app:gibbs_variational_proof}.
\begin{proposition}[\cite{kullback1959information, donsker1976asymptotic}]
    The solution of \eqref{eq:problem_fixed_omega} is given by the Gibbs measure $g_\omega$ defined by~\eqref{eq:gibbs_measure}.
    \label{prop:solution_fixed_omega}
\end{proposition}
When $\omega = 0$, the problem becomes degenerate, as it admits an \emph{infinite} set of solutions given by the probability measures of the form $q^* = \sum_{i=1}^I c_i \delta_{\xib^*_i}$, where the coefficients $c_1, \dots, c_I$ are non-negative and sum to 1, and $\delta_{\xib}$ is the Dirac measure centered at $\xib$. The temperature $\omega$ controls the trade-off between the term $\mathbb{E}_{q}[\ell(\xib)]$, which favors a sharp $q^*$ concentrating at the modes, and the entropy term $\mathcal{H}(q)$, which favors the opposite.

The annealing of the Gibbs measure through the variational formulation in \eqref{eq:problem_fixed_omega} is valid because the measure converges and concentrates around the global modes as $\omega \to 0$.
\begin{theorem}[Annealed Gibbs measure]
    Suppose that $\ell$ is twice continuously differentiable and all the global modes of $\ell$ are non-degenerate, meaning that for all $i\in[I]$, the Hessian matrix $\nabla^2 \ell(\xib^*_i)$ is negative definite. Then, we have
    \begin{equation*}
        g_\omega \underset{\omega \rightarrow 0}{\rightharpoonup} g_0 := \sum_{i=1}^I \tilde{c}_i \delta_{\xib^*_i}, \,\text{ with }\, \forall i \in [I],\, \tilde{c}_i = \frac{\det(- \nabla^2 \ell(\xib^*_i))^{-1/2}}{\sum_{i'=1}^I \det(- \nabla^2 \ell(\xib^*_{i'}))^{-1/2}},
    \end{equation*}
    where $\rightharpoonup$ denotes weak convergence and $\delta_{\xib}$ is the Dirac measure centered at $\xib$.
    \label{th:annealing_limit}
\end{theorem}
This result is similar to the one of \cite{hwang1980laplace}, which applies to Gibbs posteriors (see Appendix~\ref{subsec:gibbs_posteriors}), but in our case, the prior distribution might not always be present in $\ell$. A detailed proof, relying on Laplace's method, is included in Appendix~\ref{app:proof_annealing_limit}.
Note that, even though at $\omega=0$ the set of solutions is infinite, the solution is unique for $\omega>0$. Theorem~\ref{th:annealing_limit} shows that the sequence $(g_\omega)_{\omega}$ follows a specific path when $\omega \rightarrow 0$ and eventually reaches the solution set of the degenerate problem. Therefore, the limit of the solution $g_\omega$ is unique as $\omega \rightarrow 0$. Thus, we can approximate $g_0$ as closely as desired (in the sense of weak convergence) by solving successive versions of~\eqref{eq:problem_fixed_omega} with decreasing $\omega$.

The ``annealed'' Gibbs measure $g_0$ is a mixture of Dirac distributions located on the global modes of $\ell$, with specific weights $\tilde{c}_i$ depending on the Hessian of $\ell$. This motivates to restrict the optimization in \eqref{eq:problem_fixed_omega} to a set of mixture distributions, which we will discuss next. The weights $\tilde{c}_i$ quantify the curvature of $\ell$ at its modes, estimating their ``flatness'', which can be insightful (see detailed discussions in Appendix~\ref{sub:interpretation_weights}).
In our approach, we aim to exploit the variational formulation~\eqref{eq:problem_fixed_omega} to find $g_0$, which is similar but somewhat different to the other use of Gibbs measures in simulated annealing through sampling~\citep{kirkpatrick1983optimization, geman1984stochastic} and global optimization as \textit{nascent} distributions \citep{luo2018minima, zhang2023progo, serre2025stein}. 
The variational problem can also be viewed as a \textit{generalized Bayes} problem \citep{bissiri2016general,knoblauch2022optimization}. A short discussion of this in our context is included in Appendix~\ref{sec:bayesian_link}.

\section{Natural variational annealing with mixtures}
\label{sec:optimization}

In the previous section, we presented the annealed Gibbs measures and their interpretation as the limit of a sequence of solutions to variational problems, where the entropy penalty is annealed. To effectively solve the variational problems with mixture distributions, we use natural gradient ascent, which is the third foundational concept of the NVA approach.

The natural gradient procedure is inspired by the \textit{Bayesian Learning Rule} framework \citep{khan2023bayesian}, which employs natural gradients to solve variational problems, with variational families contained in the exponential family. Natural gradients allow faster convergence (see Appendix~\ref{subsec:natural_gradients} for a justification). However, for multimodal optimization, the use of entropy annealing, combined with mixture distributions, is also necessary.

The initial step in reaching $g_0$ involves finding a way to approximate the Gibbs measure $g_\omega$ with a mixture for a fixed $\omega > 0$. Its explicit form remains unknown due to the general lack of knowledge about $\ell$. This section focuses on this crucial step. We show how natural gradient ascent can be employed to approximate $g_\omega$ with a mixture and discuss its practical implementation. Annealing can then be implemented by defining an annealing schedule $(\omega_t)_{t \ge 1}$ modifying the value of $\omega$ at each iteration.

\subsection{Natural parameterization of mixtures}

We adopt the following parameterization for mixtures of exponential family distributions, expressed as:
\begin{equation*}
    q_\Lambdab(\xib) = \sum_{k = 1}^K \pi_k q_{\lambdab_k}(\xib),
\end{equation*}
where $\pib = (\pi_1, \dots, \pi_{K})$ are the mixture weights, which are non-negative and sum to 1, and $q_{\lambdab_k}$ are the components' density functions of the form
    \begin{equation}
        q_{\lambdab_k}(\xib) = h_k(\xib) \exp \left( \left< \lambdab_k, \Tb_k(\xib) \right> - A(\lambdab_k) \right).
        \label{eq:exponential_family_density}
    \end{equation}
The exponential family natural parameters are $\lambdab_k$, while $\Tb_k(\xib)$ are their sufficient statistics, and $h_k$ and $A_k$ are respectively called their base measure and log partition function. The whole mixture parameters are denoted by $\Lambdab := \left(\log(\pi_1/\pi_K), \dots, \log(\pi_{K-1}/\pi_K), \lambdab_1, \dots, \lambdab_K \right)$. 
Although mixtures of exponential family distributions do not belong to the exponential family, $\Lambdab$ can still be considered as the \textit{natural parameters} of the mixture $q_\Lambdab$, viewed as a \textit{minimal conditional exponential-family distribution} (MCEF, \cite{lin2019fast}). Indeed, a variable $\xib$ with mixture distribution can be represented hierarchically with a mixing random variable $Z$ that takes values in $[K]$ such that
\begin{equation}
Z \sim \mathcal{M}(1, \pib), \quad
    \xib \mid Z \sim q_{\lambdab_Z},
\label{eq:mcef_gaussian_mixture}
\end{equation}
where $\mathcal{M}(1, \pib)$ is the multinoulli distribution (multinomial distribution with $1$ trial) with probability vector $\pib$. In this form, $\Lambdab$ combines the traditional natural parameters of the multinomial distribution governing $Z$ with the traditional natural parameters of each component in the conditional distribution $\xib \mid Z = k$ (see \cite{lin2019fast} for further details).

Thus, the restriction of~\eqref{eq:problem_fixed_omega} to mixtures of exponential family distributions can be reformulated as
\begin{equation}
    \Lambdab^{*,\omega} \in \argmax_\Lambdab~\mathbb{E}_{q_\Lambdab}[\ell(\xib)] + \omega \mathcal{H}(q_\Lambdab) = \argmax_\Lambdab~\mathcal{L}_\omega(\Lambdab),
    \tag{P3}
    \label{eq:problem_lambda_omega}
\end{equation} 
where $\mathcal{L}_\omega(\Lambdab) := \mathbb{E}_{q_\Lambdab}[\ell(\xib) - \omega \log(q_\Lambdab(\xib))]$.

\subsection{Natural gradient update rules}
\label{sub:computation_natural_gradients}

Now, we consider solving~\eqref{eq:problem_lambda_omega}. Appendix~\ref{subsec:natural_gradients} motivates the use of natural gradients~\citep{amari1998natural}, instead of (vanilla) gradients, for solving this problem. Natural gradients for a mixture of exponential family distributions are defined by
\begin{equation}
    \tilde{\nabla}_\Lambdab \mathcal{L}_\omega(\Lambdab) = \Fb(\Lambdab)^{-1} \nabla_\Lambdab \mathcal{L}_\omega(\Lambdab),
    \label{eq:natural_gradient_def}
\end{equation}
where $\Fb(\Lambdab) := \mathbb{E}[\nabla_\Lambdab \log q_\Lambdab(\xib, Z) \nabla_\Lambdab \log q_\Lambdab(\xib, Z)^T]$ is the Fisher Information Matrix (FIM) of $q_\Lambdab(\xib, Z)$, which is non-singular \citep{lin2019fast}. 

A common challenge in applying natural gradients is the computation of the inverse FIM, which is typically required at each iteration. For certain distributions, such as those within the exponential family, the FIM has a simple closed form. However, in many cases, estimating, storing, and inverting this matrix can be computationally expensive, especially for models like Gaussian mixtures, where the FIM may contain $O(K^2 d^4)$ elements. Moreover, there is no guarantee that the estimated matrix will be invertible.

Fortunately, \cite{lin2019fast} provided a method to avoid directly estimating the FIM for mixtures, based on the fact that the FIM of the joint MCEF distribution is invertible. They derived natural gradient updates for mixtures by exploiting the following relationship between natural and expectation parameters:
\begin{equation}
    \tilde{\nabla}_\Lambdab \mathcal{L}_\omega(\Lambdab) = \nabla_\Mb \mathcal{L}_\omega(\Lambdab),
    \label{eq:duality_identity}
\end{equation}
where $\Mb = (\pi_1, \dots, \pi_{K-1}, \Mb_1, \dots, \Mb_K)$ are the expectation parameters of the mixture, with $\Mb_k = \mathbb{E}_{q_{\lambdab_k}}[\Tb_k(\xib)]$, for $k \in [K]$. 

Using~\eqref{eq:duality_identity}, the natural gradient ascent update rule
\begin{equation}
    \Lambdab_{t+1} = \Lambdab_t + \rho_t \tilde{\nabla}_\Lambdab \mathcal{L}_\omega(\Lambdab)|_{\Lambdab_t},
    \label{eq:natural_gradient_update}
\end{equation}
becomes 
\begin{align}
   v_{k,t+1} &= v_{k,t} + \rho_t \nabla_{\pi_k} \mathcal{L}_\omega(\Lambdab)|_{\Lambdab_t},  \label{eq:natural_gradient_update_pi} \\
     \lambdab_{k, t+1} &= \lambdab_{k, t} + \rho_t \nabla_{\Mb_k} \mathcal{L}_\omega(\Lambdab)|_{\Lambdab_t}, \label{eq:natural_gradient_update_lambda}
\end{align}
for $k \in [K]$, where $v_{k,t} := \log(\pi_{k,t}/\pi_{K,t})$. Note that a related mirror-descent interpretation of mixture-weight and component updates has been established in a variational-inference setting by \citep{daudel2021mixture, daudel2023monotonic}. In order to use these update rules, at each iteration $t$, we need to estimate the gradients $\nabla_{\pi_k} \mathcal{L}_\omega(\Lambdab)|_{\Lambdab_t}$ and $\nabla_{\Mb_k} \mathcal{L}_\omega(\Lambdab)|_{\Lambdab_t}$. 

\subsection{Estimation of the natural gradients} 

The gradient expression for the weight update~\eqref{eq:natural_gradient_update_pi} is provided in equation~\eqref{eq:grad_pi} in  Appendix~\ref{app:expression_grad_pi}, for which a Monte Carlo estimator $\widehat{\gamma}^{(\pi_k)}_{\omega,\Lambdab_t,B} := \widehat{\nabla_{\pi_k}\mathcal{L}}_\omega(\Lambdab)|_{\Lambdab_t}$ using $B$ samples is defined by~\eqref{eq:est_grad_pi_0} in  Appendix~\ref{app:estimation_gradients}. 

For the update of the individual component natural parameters~\eqref{eq:natural_gradient_update_lambda}, the gradient expressions depend on the specific distribution considered. Since the exact form of $\ell$ is typically unknown, it is not possible to derive closed-form expressions for these gradients. However, these gradients can be expressed as expectations of functions of $\xib$ when $\xib$ follows $q_{\lambdab_k}$, for some $k \in [K]$. Indeed, from the chain rule,  Proposition~\ref{prop:score_theorem}, and equation~\eqref{eq:exponential_family_density}, we have
\begin{equation}
\begin{split}
    \tilde{\nabla}_{\lambdab_k} \mathcal{L}_\omega(\Lambdab)|_{\Lambdab_t} = \nabla_{\Mb_k} \mathcal{L}_\omega(\Lambdab)|_{\Lambdab_t} &= \Jb_{\lambdab_k}(\Mb_k)^T \nabla_{\lambdab_k} \mathbb{E}_{q_{_\Lambdab}}[f_{\omega}(\xib; \Lambdab)]|_{\Lambdab_t} \\
    &= \mathbb{E}_{q_{\lambdab_k}}[\Jb_{\lambdab_k}(\Mb_k)^T (\Tb_k(\xib) - \nabla_{\lambdab_k} A_k(\lambdab_{k,t}))f_{\omega}(\xib; \Lambdab_t)],
\end{split}
\label{eq:gradients_as_expectations}
\end{equation}
where $\Jb_{\lambdab_k}(\Mb_k)$ is the Jacobian matrix of $\lambdab_k$ with respect to $\Mb_k$, as a function of $\Mb_k$, and $f_\omega(\xib; \Lambdab) := \ell(\xib) - \omega \log(q_\Lambdab(\xib))$. 

Equation~\eqref{eq:gradients_as_expectations} provides a general estimator for the natural gradients $\tilde{\nabla}_{\lambdab_k} \mathcal{L}_\omega(\Lambdab)|_{\Lambdab_t}$, applicable to all mixtures of exponential family distributions. The function $f_\omega$ can be evaluated at any $\xib \in \mathbb{R}^d$, enabling Monte Carlo approximations for estimating the expectations. For certain distributions, reparameterization tricks may offer alternative estimators involving higher-order derivatives of $\ell$ (see Section~\ref{subsec:up_rules_gm} for the example of Gaussian mixtures). These alternative estimators may perform better due to different variance properties.

\subsection{Stochastic natural gradient ascent}
\label{sub:basic_algorithm}

As discussed above, using the estimator $\widehat{\gamma}^{(\pi_k)}_{\omega,\Lambdab_t,B}$ alongside with some estimator $\widehat{\nabla_{\Mb_k} \mathcal{L}}_{\omega}(\Lambdab_t)$ for $\tilde{\nabla}_{\lambdab_k} \mathcal{L}_\omega(\Lambdab)|_{\Lambdab_t} = \nabla_{\Mb_k} \mathcal{L}_\omega(\Lambdab)|_{\Lambdab_t}$, we can derive a mini-batch stochastic natural gradient ascent algorithm (SNGA), as described by Algorithm~\ref{alg:snga} in Appendix~\ref{app:snga}.

The hyperparameters include the maximum number of iterations $T$ (additionally, other stopping criteria can be defined), the mini-batch size $B$ determining the number of samples used in the Monte Carlo gradient approximations, the number of components $K$ in the mixture, the initial mixture parameters $\Lambdab_0$, and the learning rate sequence $(\rho_t)_{t \in [T]}$ associated with the natural gradient ascent updates.

SNGA is the basis for the NVA-M algorithms, which incorporate annealing, an essential feature for solving the multimodal optimization problem. Given the non-convex nature of the variational problem within the constrained search space of these mixtures, there is a risk that the algorithm converges to a local solution of problem~\eqref{eq:problem_lambda_omega} if $\omega$ is too small. SNGA does not provide any mechanism to remediate this issue, while the annealing process in NVA-M leverages Theorem~\ref{th:annealing_limit}, mitigating this problem by approximating increasingly sharper Gibbs measures $g_\omega$ converging to the annealed Gibbs measure $g_0$.

\subsection{The NVA-M algorithm}
\label{sub:general_annealing_algorithm}

An annealing schedule, represented by a sequence of values $(\omega_t)_{t \ge 1}$, sets the temperature $\omega$ at each iteration $t$, thus changing the objective function $\mathcal{L}_{\omega_t}(\Lambdab)$ and affecting the gradients in the update rule~\eqref{eq:natural_gradient_update}. 
Algorithm~\ref{alg:snga_annealing} gives the NVA-M algorithm, where changes compared to the basic SNGA with fixed $\omega$ (Algorithm~\ref{alg:snga}) are colored in blue. 

\LinesNumbered
\begin{algorithm}
	\caption{Natural Variational Annealing with Mixtures (NVA-M)}   \label{alg:snga_annealing}
 \textsc{Given} a function $\ell$. \\
		 \textsc{Set} $T$, $B$, $K$, $\Lambdab_0$, $(\rho_t)_{t \in [T]}$, \textcolor{teal}{$(\omega_t)_{t \in [T]}$}.\\
   \textsc{Compute} $(v_{k,0})_{k \in [K-1]} = (\log(\pi_{k,0}/\pi_{K,0}))_{k \in [K-1]}.$ \\
		 \For{$t=0\!:\!(T-1)$}{
				   \For {$k=1\!:\!K$}{
				    \textsc{Sample} $\xib^{(k)}_b \overset{\text{i.i.d.}}{\sim} q_{\lambdab_{k,t}},  \quad \text{for  $b=1\!:\!B$}. $ \\
      \textsc{Compute} \textcolor{teal}{$\widehat{\nabla_{\Mb_k} \mathcal{L}}_{\omega_t}(\Lambdab_t)$}.\\
 \textsc{Update} $\lambdab_{k,t+1} =  \lambdab_{k,t}  + \rho_t  \textcolor{teal}{\widehat{\nabla_{\Mb_k} \mathcal{L}}_{\omega_t}(\Lambdab_t)}.$ } 
\For {$k=1\!:\!(K\!-\!1)$}{
\textsc{Compute} \textcolor{teal}{$\widehat{\gamma}^{(\pi_k)}_{\omega_t,\Lambdab_t,B}$}.  \\
\textsc{Update} $v_{k, t+1} =  v_{k, t}  + \rho_t  \textcolor{teal}{\widehat{\gamma}^{(\pi_k)}_{\omega_t,\Lambdab_t,B}}$.
}
}
\textsc{Compute} $(\pi_{k,T})_{k \in [K]}$ from $(v_{k,T})_{k \in [K-1]}$. \\
		\Return $\Lambdab_T.$ 
\end{algorithm}

While the NVA-M algorithm is flexible and can accommodate any mixture of exponential family distribution, in the rest of the paper, we will be considering Gaussian mixtures. 

\section{Special case of Gaussian mixtures (NVA-GM)}
\label{sec:mego}

In the previous section, we have derived the NVA-M algorithm to approximate the annealed Gibbs measure using a mixture of exponential family distributions. Here, we demonstrate how using Gaussian mixtures (GM) in the NVA framework allows direct tracking of the modes of $\ell$ as $\omega$ goes to 0, through the means of the Gaussian components. This motivates the NVA-GM algorithm. Additionally, we discuss the importance of the annealing schedule.

\subsection{Annealing properties of Gaussian mixtures variational approximation}
\label{sub:asymptotic_behavior}

Gaussian mixture variational families have already been used in earlier works on multimodal approximation \citep{gershman2012nonparametric, lin2019fast, lin2020handling, daudel2021mixture, arenz2023unified, daudel2023monotonic, petit2025variational}. In this section, we show that the Gaussian mixture approximation of the annealed Gibbs measure leads to vanishing covariance matrices as $\omega \rightarrow 0$. To characterize this asymptotic behavior, it is more convenient to parameterize Gaussian distributions using their means and covariance matrices $\thetab = (\mub, \Sigmab)$, rather than their natural parameters $\lambdab$, which involve diverging precision matrices $\ESSb = \Sigmab^{-1}$. Thus, a Gaussian mixture with $K$ components is expressed as 
\begin{equation*}
    q_\thetab = \sum_{k = 1}^K \pi_k q_{\thetab_k},
\end{equation*}
where the weights $\pib = (\pi_1, \dots, \pi_{K})$ sum to $1$, and each component is a Gaussian $q_{\thetab_k} = \mathcal{N}(\mub_k, \Sigmab_k)$. We aggregate these parameters into a vector, for all $k \in [K]$, $\thetab_k := (\mub_k, \Sigmab_k)$ and
\begin{equation}
    \thetab := (\pi_1, \dots, \pi_{K-1}, \mub_1, \dots, \mub_K, \Sigmab_1, \dots, \Sigmab_K).
    \label{eq:parameterization_theta}
\end{equation}
There is a one-to-one correspondence between $\lambdab$ and $\thetab$. The variational approximation problem~\eqref{eq:problem_lambda_omega} writes,
\begin{equation}
     \thetab^{*, \omega} \in \argmax_{\thetab}~\mathbb{E}_{q_\thetab}[\ell(\xib)] + \omega \mathcal{H}(q_\thetab).
     \tag{P4}
     \label{eq:variational_theta}
\end{equation}
Although the solution $\thetab^{*, \omega}$ is not necessarily unique, we can analyze the behavior of $\thetab^{*, \omega}$ as $\omega \rightarrow 0$. Consider a sequence $(\omega_t)_{t \ge 1}$ such that $\omega_t \xrightarrow[]{} 0$ as $t \rightarrow \infty$, and a corresponding sequence $(\thetab^{*, \omega_t})_{t \ge 1}$ where, $\thetab^{*, \omega_t}$ is a solution of~\eqref{eq:variational_theta} for  $\omega = \omega_t$. Our primary interest lies in the behavior of the component means $\mub^{*,\omega_t}_k$ and covariance matrices $\Sigmab^{*,\omega_t}_k$ as $\omega_t \rightarrow 0$, under the assumption that $\ell$ has $I = K$ non-degenerate modes (see Assumption~\ref{ass:determininant_positive} in Appendix \ref{app:proof_annealing_limit}, for a precise statement). 

\paragraph{Single Gaussian behavior.} First, we consider the simpler case where $\ell$ has only one mode, implying that the variational approximation is performed with a single Gaussian ($I = K = 1$), solving 
\begin{equation}
    \thetab^{*, \omega}_1 \in \argmax_{\thetab_1}~\mathbb{E}_{q_{\thetab_1}}[\ell(\xib)] + \omega \mathcal{H}(q_{\thetab_1}).
     \tag{P5}
    \label{eq:variational_theta_1}
\end{equation}
We have the following result.
\begin{proposition}
    Suppose that $\ell$ is strictly concave with one global non-degenerate maximum located at $\xib^*_1$. Let $\omega_t \xrightarrow[]{} 0$ as $t \rightarrow \infty$. For all $t \ge 1$, let $\thetab^{*, \omega_t}_1 = (\mub^{*, \omega_t}_1, \Sigmab^{*, \omega_t}_1)$ be a solution of~\eqref{eq:variational_theta_1} for $\omega = \omega_t$.
    Then, we have
    \begin{align*}
        \mub^{*, \omega_t}_1 &\xrightarrow[t \rightarrow \infty]{} \xib^*_1, \\
        \omega_t^{-1} \Sigmab^{*, \omega_t}_1 &\xrightarrow[t \rightarrow \infty]{} (-\nabla^2_\xib \ell(\xib^*_1))^{-1}.
    \end{align*}
    \label{prop:asymptotic_behavior_single_gaussian}
\end{proposition}
The proof for this proposition is given in Appendix~\ref{app:asymptotic_behavior_single_gaussian}. This result ensures that when $\ell$ has only one mode, the Gaussian approximation converges weakly to a Dirac distribution centered at the mode. Furthermore, it establishes a convergence rate for the covariance matrix of the Gaussian distribution. This notably implies that the eigenvalues of the covariance matrix decrease linearly with $\omega$, tending to $0$ as $\omega \rightarrow 0$.

\paragraph{Mixture behavior.} Let us generalize this to the case where $K = I > 1$. According to Theorem~\ref{th:annealing_limit}, $g_\omega$ converges to $g_0$, with $g_0$ being a mixture of Dirac measures located at the modes $(\xib^*_1, \dots, \xib^*_K)$. Since $g_0$ is the limit of a sequence of Gaussian mixtures, where eigenvalues of the covariance matrices vanish, it is reasonable to assume that the sequence $\thetab^{*,\omega}$ converges to, up to a permutation of the component labels, to
\begin{equation*}
    \thetab^{*, 0} = (\tilde{c}_1, \dots, \tilde{c}_{K-1}, \xib^*_{1}, \dots, \xib^*_{K}, 0, \dots, 0),
\end{equation*}
so that in $g_0$, the locations and the weights of the Dirac mixture are respectively retrieved as the limits of the means and the weights of the components of the Gaussian mixture. Proving such a result is beyond the scope of our paper. In Appendix~\ref{app:entropy_approximation}, we may also risk this conjecture to demonstrate how $\Sigmab^{*, \omega}_k \underset{\omega \rightarrow 0}{\sim} \omega (- \nabla^2_\xib \ell(\xib^*_k))^{-1}$ for all $k \in [K]$.

Finally, note that the mixture weights also admit a natural interpretation in this asymptotic regime.
When a component aligns with a non-degenerate mode $\xib_i^*$, its covariance satisfies
$\Sigmab_k \approx \omega (-\nabla^2_\xib \ell(\xib_i^*))^{-1}$. Substituting this local Gaussian
approximation into the entropy-regularized objective $\mathcal{L}_\omega$ shows that the contribution of each component
depends on both its weight $\pi_k$ and the local curvature. As a consequence, the stationary points of
the weight updates recover the curvature-dependent coefficients of the annealed Gibbs limit in
Theorem~\ref{th:annealing_limit}. Thus, although curvature does not appear explicitly in the update
equation~\eqref{eq:natural_gradient_update_pi} for $\pi_k$, it influences the weight dynamics implicitly through the Gaussian asymptotics of
each component.

\subsection{Update rules and gradient estimators for Gaussian mixtures}
\label{subsec:up_rules_gm}

In the previous section, we have shown that the component means of a Gaussian mixture variational approximation of the Gibbs measure $g_\omega$ converge to the global modes of the objective function $\ell$, as $g_\omega$ becomes closer to the annealed Gibbs measure $g_0$. This motivates the use of Gaussian mixtures in NVA-M to extract the modes of $\ell$, as a direct output of the algorithm. More specifically, if $K = I$, up to a permutation of the component labels, we expect the means $\mub_{k,t}$ to converge to the modes $\xib^*_k$, the covariance matrices
$\Sigmab_{k,t}$ to $0$, and the mixture weights $\pi_{k,t}$ to the weight $\tilde{c}_k$. Here, we give more detailed expressions applying NVA-M to Gaussian mixtures, giving rise to the NVA-GM algorithm which directly updates the means and the precision matrices of the Gaussian components. Because the natural gradient update rules apply to the Gaussian natural parameters, it is easier to derive updates on $\ESSb_{k,t} = \Sigmab_{k,t}^{-1}$ rather than on $\Sigmab_{k,t}$. Work on gradient-based updates for Gaussian mixtures has also been advanced in variational-inference settings \cite{petit2025variational}. A comprehensive comparison of design choices for natural-gradient updates of Gaussian mixture components is provided by \cite{arenz2023unified}.

\paragraph{Natural gradient update rules.}

For a more practical use, the chain rule can be applied to express the gradient with respect to the expectation parameters $\Mb_k$ as a combination of gradients with respect to the Gaussian parameters.

For Gaussian mixtures with means $\mub_k$ and precision matrices $\ESSb_k$, for $k \in [K]$, the update~\eqref{eq:natural_gradient_update_lambda} becomes (see derivation in Appendix~\ref{app:natural_gradient_update_rule_gaussian})
\begin{align}
    \ESSb_{k, t+1} &= \ESSb_{k, t} - \frac{2 \rho_t}{\pi_{k,t}} \nabla_{\ESSb_k^{-1}} \mathcal{L}_{\omega_t}(\Lambdab)|_{\Lambdab_t}, \label{eq:natural_gradient_update_s} \\
    \mub_{k, t+1} &= \mub_{k, t} + \frac{\rho_t}{\pi_{k,t}} \ESSb_{k, t+1}^{-1} \nabla_{\mub_k} \mathcal{L}_{\omega_t}(\Lambdab)|_{\Lambdab_t}. \label{eq:natural_gradient_update_mu} 
\end{align}
In this case, at each iteration $t$, we need to estimate the gradients $\nabla_{\pi_k} \mathcal{L}_{\omega_t}(\Lambdab)|_{\Lambdab_t}$, $\nabla_{\ESSb_k^{-1}} \mathcal{L}_{\omega_t}(\Lambdab)|_{\Lambdab_t}$ and $\nabla_{\mub_k} \mathcal{L}_{\omega_t}(\Lambdab)|_{\Lambdab_t}$.

\paragraph{iBLR update rules.}
    The \textit{improved Bayesian Learning Rule} (iBLR) of \cite{lin2020handling} can be used instead of the natural gradient update rule by adding a correction term in the update~\eqref{eq:natural_gradient_update_lambda}. In the case of Gaussian mixtures, the iBLR is simply obtained as it only affects the precision matrix update~\eqref{eq:natural_gradient_update_s}:
    \begin{equation*}
        \ESSb_{k, t+1} = \ESSb_{k, t} - \frac{2\rho_t}{\pi_{k,t}}  \nabla_{\ESSb_k^{-1}} \mathcal{L}_{\omega_t}(\Lambdab)|_{\Lambdab_t} + \frac{2 \rho_t^2}{\pi_{k,t}^2} \nabla_{\ESSb_k^{-1}} \mathcal{L}_{\omega_t}(\Lambdab)|_{\Lambdab_t} \ESSb_{k,t}^{-1} \nabla_{\ESSb_k^{-1}} \mathcal{L}_{\omega_t}(\Lambdab)|_{\Lambdab_t}.
    \end{equation*}
    For Gaussian mixtures, this update rule has two advantages: it does not only improve convergence speed, but it also ensures that the updated precision matrix $\ESSb_{k, t+1}$ remains positive-definite \citep{lin2020handling}. In the rest of the paper, we focus on the natural gradient update rule for simplicity, while keeping in mind that the iBLR rule can be used instead.

\paragraph{Estimation of the natural gradients.}

Typically, for Gaussian mixtures, the gradient needed in~\eqref{eq:natural_gradient_update_s} and~\eqref{eq:natural_gradient_update_mu} can be expressed under several forms using Bonnet and Price's theorems (Appendix~\ref{app:bonnet_price}), which are given by~\eqref{eq:grad_s_0}-\eqref{eq:grad_s_2} in Appendix~\ref{app:expression_grad_mu_s}. Several Monte Carlo estimators are available for these quantities, given by equations~\eqref{eq:est_grad_pi_0} to \eqref{eq:est_grad_s_2} in Appendix~\ref{app:estimation_gradients}. Although all these estimators are unbiased and consistent, they may perform differently since they do not have the same variance. Usually, it is hard to analytically determine which estimators have the lowest variances, as the variance expressions are highly dependent on $\ell$. 
The main factor to consider in selecting these estimators is the computational availability of the gradient $\nabla_\xib \ell(\xib)$ and the Hessian $\nabla^2_\xib \ell(\xib)$. Table~\ref{tab:choice_estimators} lists the possible choices in the case of Gaussian mixtures. In our experiments, estimators using the higher-order derivatives of $\ell$ tend to perform better, often having the lowest variances. In what follows, we will simply write  
$\widehat{\gamma}^{(\pi_k)}_{\omega_t,\Lambdab_t,B}$,  $\widehat{\gammab}^{(\mub_k)}_{\omega_t,\Lambdab_t,B}$ and  $\widehat{\gammab}^{(\ESSb_k^{-1})}_{\omega_t,\Lambdab_t,B}$ to refer to one of these estimators. 

\paragraph{Dimensional scaling.}
    Since these estimators involve matrix products and some update rules require inversion of the full precision matrices, their computational cost scales as $O(d^3)$ per component. Diagonal or isotropic covariance parameterizations avoid these cubic operations and improve scalability in high-dimensional settings (see Sec.~\ref{sec:computational_cost}).

\begin{table}
\begin{tabular}{ l  c c c }
\toprule
 Possible gradient estimators for & $\nabla_{\pi_k} \mathcal{L}_{\omega_t}$ & $\nabla_{\mub_k} \mathcal{L}_{\omega_t}$ &  $\nabla_{\ESSb_k^{-1}} \mathcal{L}_{\omega_t}$ \\ 
\midrule
  If $\ell$ available & \eqref{eq:est_grad_pi_0} & \eqref{eq:est_grad_mu_0} & \eqref{eq:est_grad_s_0} \\   
 If $\nabla_\xib \ell$ available & - & \eqref{eq:est_grad_mu_1} & \eqref{eq:est_grad_s_1} \\  
 If $\nabla^2_\xib \ell$ available & - & - & \eqref{eq:est_grad_s_2}  \\
\bottomrule
\end{tabular}
 \caption{Possible estimators for gradients $\nabla_{\pi_k} \mathcal{L}_{\omega_t}(\Lambdab)|_{\Lambdab_t}$, $\nabla_{\mub_k} \mathcal{L}_{\omega_t}(\Lambdab)|_{\Lambdab_t}$ and $\nabla_{\ESSb_k^{-1}} \mathcal{L}_{\omega_t}(\Lambdab)|_{\Lambdab_t}$, used in the update of the Gaussian mixture parameters, depending on the availability of $\nabla_\xib \ell$ and $\nabla^2_\xib \ell$. Their expressions are given in Appendix~\ref{app:estimation_gradients}.}
\label{tab:choice_estimators}
\end{table}


\subsection{The NVA-GM algorithm} 
\label{subsec:nva_gm_algo}

Algorithm~\ref{alg:snga_annealing_gaussian} implements NVA-M with Gaussian mixtures (NVA-GM). For convenience, instead of updating the natural parameters of the components, the updates are performed on the parameterization $\thetab$ defined by~\eqref{eq:parameterization_theta}. 

The behavior of the mixture parameters in the NVA-GM algorithm can be interpreted in different ways. The component means can be viewed as different particles searching for the modes, but the components themselves can be viewed as separate search distributions in an evolutionary algorithm. These two interpretations are discussed in Appendices~\ref{sub:particle_interpretation} and~\ref{sub:evolutionary_interpretation}. 

\LinesNumbered
\begin{algorithm}
	\caption{Natural Variational Annealing with Gaussian Mixtures (NVA-GM)}   \label{alg:snga_annealing_gaussian}
 \textsc{Given} a function $\ell$. \\
		 \textsc{Set} $T$, $B$, $K$, $\thetab_0$, $(\rho_t)_{t \in [T]}$, $(\omega_t)_{t \in [T]}$.\\
   \textsc{Compute} $(v_{k,0})_{k \in [K-1]} = (\log(\pi_{k,0}/\pi_{K,0}))_{k \in [K-1]}.$ \\
		 \For{$t=0\!:\!(T-1)$}{
				   \For {$k=1\!:\!K$}{
				    \textsc{Sample} $\xib^{(k)}_b \overset{\text{i.i.d.}}{\sim} \mathcal{N}(\mub_{k,t}, \ESSb_{k,t}^{-1}),  \quad \text{for  $b=1\!:\!B$}. $ \\
      \textsc{Compute} $\widehat{\gammab}^{(\mub_k)}_{\omega_t,\Lambdab_t,B}$, $\widehat{\gammab}^{(\ESSb_k^{-1})}_{\omega_t,\Lambdab_t,B}$. \textcolor{brown}{\#gradient estimators for $\nabla_{\mub_k} \mathcal{L}_{\omega_t}$, $\nabla_{\ESSb_k^{-1}} \mathcal{L}_{\omega_t}$} \\
 \textsc{Update} $\ESSb_{k,t+1} =  \ESSb_{k,t}  - \rho_t \widehat{\gammab}^{(\ESSb_k^{-1})}_{\omega_t,\Lambdab_t,B} $, ~\textsc{and}~ $\mub_{k,t+1} = \mub_{k,t}  + \rho_t  \ESSb_{k,t+1}^{-1} \widehat{\gammab}^{(\mub_k)}_{\omega_t,\Lambdab_t,B}$. } 
\For {$k=1\!:\!(K\!-\!1)$}{
\textsc{Compute} $\widehat{\gamma}^{(\pi_k)}_{\omega_t,\Lambdab_t,B}$. \textcolor{brown}{\#gradient estimator for $\nabla_{\pi_k} \mathcal{L}_{\omega_t}$} \\
\textsc{Update} $v_{k, t+1} =  v_{k, t}  + \rho_t  \widehat{\gamma}^{(\pi_k)}_{\omega_t,\Lambdab_t,B}$.
}
}
\textsc{Compute} $(\pi_{k,T})_{k \in [K]}$ from $(v_{k,T})_{k \in [K-1]}$. \\
		\Return $\thetab_T.$ 
\end{algorithm}

In practice, in NVA-GM (and NVA-M), there is no guarantee that all components will converge to distinct modes. The annealing schedule plays an essential role in helping the algorithm avoid this issue, as we will discuss in the next section. 

\subsection{Importance of the annealing schedule}
\label{sub:annealing_schedule}

Since we aim to reach $g_0$, one might question the benefits of using an annealing schedule as opposed to directly solving~\eqref{eq:problem_lambda_omega} for a small, but fixed value of $\omega$, which effectively corresponds to a constant annealing schedule. Intuitively, for a sufficiently small $\omega$, the Gibbs measure $g_\omega$ should approximate $g_0$ closely. 

The rationale for using an annealing schedule lies in the fact that although problem~\eqref{eq:problem_fixed_omega} is strictly convex, problem~\eqref{eq:problem_lambda_omega} is generally not convex and can admit multiple local solutions. As a result, the solution obtained by the natural gradient ascent algorithm may not be the global solution $\Lambdab^{*, \omega}$. In particular, local solutions can occur if several components of the mixture become overly close, leading to their merging and sharing responsibility for the same mode, which can result in the failure to identify the remaining modes. 

We can interpret the component means as particles moving in the search space, as discussed in Appendix~\ref{sub:particle_interpretation} and the temperature $\omega$ as a parameter regulating the exploration-exploitation balance. Appendix~\ref{sub:annealing_schedule_interpretation} analyzes how the value of $\omega$ influences the behavior of the system of particles formed by the component means, particularly in terms of exploration versus exploitation. A discussion on the choice of the annealing schedule is given in Appendix~\ref{sub:hyper_annealing}.

\section{Fitness-shaped NVA-GM}
\label{sec:fs_mego}

NVA-M (and NVA-GM) can be interpreted in the light of evolution strategies, where each component corresponds to a \textit{search distribution} (Appendix~\ref{sub:evolutionary_interpretation}). This section shows that techniques from evolutionary computation can be introduced in the NVA framework to improve its performance, with the example of \textit{fitness shaping} \citep{wierstra2014natural, ollivier2017information} aiming to make the algorithms more robust.

\subsection{Fitness shaping in the IGO framework}

In natural gradient updates, the gradients are highly sensitive to the values of the fitness function $f$. Extreme values of $f$ can distort the gradients, potentially leading to premature convergence or numerical instability \citep{sun2009efficient}. To circumvent this issue, \textit{fitness shaping} is used by the NES \citep{wierstra2014natural} and the IGO framework \cite{ollivier2017information}. Fitness shaping replaces the fitness function $f$ at each iteration $t$ with a \textit{utility function}, which is an adaptive monotonic transformation $W_{\lambdab_t}$ of $f$, based on its quantiles relative to the current search distribution $p_{\lambdab_t}$. In other words, for a given generation of samples, this transformation ranks fitness values in the current population, telling how good an observed fitness value is relative to the fitness distribution induced by the search distribution. This ensures that the gradient flow is invariant under any rank-preserving transformation of the fitness function \citep{ollivier2017information}, which provides robustness to the optimization process. 

The transformation $W_{\lambdab}$ is defined as $W_{\lambdab} : x \mapsto W_{\lambdab}(x) = w\left(\mathbb{P}_{\xib' \sim p_{\lambdab}}\{f(\xib') \le x\}\right)$, where $w : [0,1] \rightarrow \mathbb{R}$ is a non-increasing function known as the \textit{weighting scheme}. The choice of $W_{\lambdab}$ can be made through the choice of $w$. This definition guarantees that $W_{\lambdab}(f(\xib))$ remains invariant under monotonic transformations of $f$. Consequently, $W_{\lambdab}$ can effectively compare the values of $f$ and the modes of $f$ would remain modes of $W_{\lambdab}(f(\cdot))$, i.e.~$\argmax_\xib~f(\xib) \subseteq \argmax_\xib~W_{\lambdab_t}(f(\xib))$. 

The natural gradient ascent update rule can be modified as follows:
\begin{equation*}
  \lambdab_{t+1} = \lambdab_t + \rho_t \tilde{\nabla}_\lambdab {\mathcal L}_{W_{\lambdab_t}}(\lambdab)|_{\lambdab = \lambdab_t},\,\text{ where }\, {\mathcal L}_{W_{\lambdab_t}}(\lambdab) = \mathbb{E}_{p_{\lambdab}}[W_{\lambdab_t}(f(\xib))].
\end{equation*}
If the search distribution is a Gaussian $p_{\lambdab_t} = \mathcal{N}(\mb, \ESSb^{-1})$, then similar to equations~\eqref{eq:grad_s_0} and~\eqref{eq:grad_mu_0}, the gradients can be expressed as:
\begin{align*}
        \nabla_{\mub} \mathcal{L}_{W_{\lambdab_t}}(\lambdab) &= \mathbb{E}_{\mathcal{N}(\mub, \ESSb^{-1})}[\ESSb (\xib - \mub) W_{\lambdab_t}(f(\xib))], \\
        \nabla_{\ESSb^{-1}} \mathcal{L}_{W_{\lambdab_t}}(\lambdab) &= \frac{1}{2} \mathbb{E}_{\mathcal{N}(\mub, \ESSb^{-1})}[(\ESSb (\xib - \mub) (\xib - \mub)^T \ESSb - \ESSb) W_{\lambdab_t}(f(\xib))].
\end{align*}
These gradients are then used to compute the natural gradients, according to the update rules~\eqref{eq:natural_gradient_update_s} and~\eqref{eq:natural_gradient_update_mu}. 
In practice, for a sample $\xib_1, \dots, \xib_B \overset{\text{i.i.d.}}{\sim} p_{\lambdab_t}$ of size $B$, we can estimate
\begin{equation*}
    W_{\lambdab_t}(f(\xib_b)) = w(\mathbb{P}_{\xib' \sim p_{\lambdab_t}}\{f(\xib') \le f(\xib_b)\}) \approx \widehat{w}_b := w\left(\frac{\text{rk}(\xib_b) + 1/2}{B}\right),
\end{equation*}
where $\text{rk}(\xib_b)= \text{Card}(\{ j : f(\xib_j) < f(\xib_b)\})$. According to Theorem 6 of \cite{ollivier2017information}, the following estimators are consistent for the gradients 
\begin{align*}
        \widehat{\nu}^{(\mub)}_{\lambdab_t, B} &:= \frac{1}{B} \ESSb \sum_{b=1}^B (\xib_b - \mub) \widehat{w}_b &\xrightarrow[B \rightarrow \infty]{\mathbb{P}} &\,\,\nabla_{\mub} \mathcal{L}_{W_{\lambdab_t}}(\lambdab), \\
        \widehat{\nu}^{(\ESSb^{-1})}_{\lambdab_t, B} &:= \frac{1}{2 B} \ESSb \sum_{b=1}^B \left((\xib_b - \mub) (\xib_b - \mub)^T \ESSb - \Ib\right) \widehat{w}_b &\xrightarrow[B \rightarrow \infty]{\mathbb{P}} &\,\,\nabla_{\ESSb^{-1}} \mathcal{L}_{W_{\lambdab_t}}(\lambdab).
\end{align*}

If $B$ remains constant throughout the algorithm, then specifying the complete weighting scheme $w$ is unnecessary because only the ordered \textit{utility values} $\ub = (u_b)_{b \in [B]} := (\widehat{w}_{(b)})_{b \in [B]} = \left(w\{(b-1/2)/B\}\right)_{b \in [B]}$ matter. These utility values $\ub$ must satisfy $u_1 \ge \dots \ge u_B$, and can be chosen as hyperparameters of the algorithm.

\subsection{Adaptation to the NVA-M framework}

In the IGO framework, fitness shaping is a crucial technique ensuring the algorithm's robustness, therefore facilitating the convergence of the search distribution. In our method, as discussed in Appendix~\ref{sub:evolutionary_interpretation}, each component of the mixture behaves similarly to a search distribution. Therefore, it is logical to apply fitness shaping independently to each component. Adapting our method, we suggest new updates of the natural parameters $\lambdab_k$, while keeping the updates on the weights of the mixture $\pib = (\pi_1, \dots, \pi_K)$ unchanged.

To incorporate fitness shaping, we define a decreasing function $w$ on $(0,1)$. For each $k \in [K]$, the so-called \textit{quantile rewriting} $W_{\lambdab_k}^{(\omega, \Lambdab)}$ for $f_\omega(\cdot; \Lambdab)$ is defined by
\begin{equation}
    W_{\lambdab_k}^{(\omega, \Lambdab)}(x) =  w\left(\mathbb{P}_{\xib' \sim q_{\lambdab_k}}\{f_\omega(\xib'; \Lambdab) \le x\}\right).
    \label{eq:define_rewriting_multi}
\end{equation}

\paragraph{Natural gradient update with fitness shaping.}

Without fitness shaping, the natural gradient update rule for $\lambdab_k$ at each iteration is
\begin{equation*}
    \lambdab_{k,t+1} = \lambdab_{k,t} + \rho_t \tilde{\nabla}_{\lambdab_k} \mathbb{E}_{q_\Lambdab}[f_{\omega_t}(\xib; \Lambdab)]|_{\Lambdab = \Lambdab_t}.
\end{equation*}
Proposition~\ref{prop:score_theorem} and the definition of  $q_\Lambdab$ as a mixture $\sum_k \pi_k q_{\lambdab_k}$ ensure that 
\begin{equation*}
    \tilde{\nabla}_{\lambdab_k} \mathbb{E}_{q_\Lambdab}[f_{\omega_t}(\xib; \Lambdab)]|_{\Lambdab = \Lambdab_t} = \tilde{\nabla}_{\lambdab_k} \mathbb{E}_{q_\Lambdab}[f_{\omega_t}(\xib; \Lambdab_t)]|_{\Lambdab = \Lambdab_t} = \tilde{\nabla}_{\lambdab_k} \mathbb{E}_{q_{\lambdab_k}}[\pi_k f_{\omega_t}(\xib; \Lambdab_t)]|_{\lambdab_k = \lambdab_{k,t}}.
\end{equation*} 
This reveals that our update rule is equivalent to performing one step of natural gradient ascent to solve
\begin{equation*}
    \lambdab_k^* \in \argmax_{\lambdab_k}~\mathbb{E}_{q_{\lambdab_k}}[f_{\omega_t}(\xib; \Lambdab_t)].
\end{equation*}

Now, given that $W_{\lambdab_{k,t}}^{(\omega_t, \Lambdab_t)}$ is rank-preserving for $f_{\omega_t}(\cdot; \Lambdab_t)$, it is reasonable to replace $f_{\omega_t}(\cdot; \Lambdab_t)$ by $W_{\lambdab_{k,t}}^{(\omega_t, \Lambdab_t)}(f_{\omega_t}(\cdot; \Lambdab_t))$. This yields the following update rule for each $\lambdab_k$, $k \in [K]$:
\begin{equation*}
    \lambdab_{k,t+1} = \lambdab_{k,t} + \rho_t \tilde{\nabla}_{\lambdab_k} \mathcal{L}_{W_{\lambdab_{k,t}}^{(\omega_t, \Lambdab_t)}}(\lambdab_k)|_{\lambdab_{k,t}},\,\text{ where }\, \mathcal{L}_{W_{\lambdab_{k,t}}^{(\omega_t, \Lambdab_t)}}(\lambdab) := \mathbb{E}_{q_{\lambdab_k}}[W_{\lambdab_{k,t}}^{(\omega_t, \Lambdab_t)}(f_{\omega_t}(\xib; \Lambdab_t))].
\end{equation*}

The natural gradients can be estimated as before using Monte Carlo approximations, as equation~\eqref{eq:gradients_as_expectations} becomes
\begin{equation}
    \tilde{\nabla}_{\lambdab_k} \mathcal{L}_\omega(\Lambdab)|_{\Lambdab_t} = \mathbb{E}_{q_{\lambdab_k}}[\Jb_{\lambdab_k}(\Mb_k)^T (\Tb_k(\xib) - \nabla_{\lambdab_k} A_k(\lambdab_{k,t}))W_{\lambdab_{k,t}}^{(\omega, \Lambdab_t)}(f_{\omega}(\xib; \Lambdab_t))].
\end{equation}

\paragraph{FS-NVA-GM algorithm.} 

For simplicity, we only give the fitness-shaped version of NVA-GM which uses Gaussian mixtures. The gradient derivations are given in Appendix~\ref{app:gradient_computation_fsnvagm}, and the estimators $\widehat{\nub}^{(\mub_k)}_{\omega,\Lambdab_t,\ub}$ and $\widehat{\nub}^{(\ESSb_k^{-1})}_{\omega,\Lambdab_t,\ub}$ are defined in equations~\eqref{eq:estimator_nu_mu} and~\eqref{eq:estimator_nu_s}. The estimator $\widehat{\gamma}^{(\pi_k)}_{\omega_t,\Lambdab_t,B}$ is identical to the one used by NVA-GM, and is provided by equation~\eqref{eq:est_grad_pi_0}. FS-NVA-GM algorithm (fitness-shaped NVA-GM, Algorithm~\ref{alg:snga_annealing_utility}) incorporates fitness shaping via the free parameter utility values $\ub$ (Appendix~\ref{sub:utility_values} provides some examples of utility values). Note that changes compared to NVA-GM (Algorithm~\ref{alg:snga_annealing_gaussian}) are indicated in blue.

\LinesNumbered
\begin{algorithm}
	\caption{Fitness-shaped NVA-GM (FS-NVA-GM)}   \label{alg:snga_annealing_utility}
 \textsc{Given} a function $\ell$. \\
		 \textsc{Set} $T$, $B$, $K$, $\thetab_0$, $(\rho_t)_{t \in [T]}$, $(\omega_t)_{t \in [T]}$, \textcolor{teal}{$\ub = (u_b)_{b \in [B]}$}.\\
   \textsc{Compute} $(v_{k,0})_{k \in [K-1]} = (\log(\pi_{k,0}/\pi_{K,0}))_{k \in [K-1]}.$ \\
		 \For{$t=0\!:\!(T-1)$}{
				   \For {$k=1\!:\!K$}{
				    \textsc{Sample} $\xib^{(k)}_b \overset{\text{i.i.d.}}{\sim} \mathcal{N}(\mub_{k,t}, \ESSb_{k,t}^{-1}),  \quad \text{for  $b=1\!:\!B$}. $ \\
				    \textcolor{teal}{\textsc{Sort} $(\xib^{(k)}_b)_{b \in [B]}$ in decreasing order for $f_{\omega_t}(\cdot, \Lambdab_t).$} \\
      \textsc{Compute} $\textcolor{teal}{\widehat{\nub}^{(\mub_k)}_{\omega_t,\Lambdab_t,\ub}}$, $\textcolor{teal}{\widehat{\nub}^{(\ESSb_k^{-1})}_{\omega_t,\Lambdab_t,\ub}}$. \textcolor{brown}{\#gradient estimators for $\nabla_{\mub_k} \mathcal{L}_{\omega_t}$, $\nabla_{\ESSb_k^{-1}} \mathcal{L}_{\omega_t}$} \\
 \textsc{Update} $\ESSb_{k,t+1} =  \ESSb_{k,t}  - \rho_t  \textcolor{teal}{\widehat{\nub}^{(\ESSb_k^{-1})}_{\omega_t,\Lambdab_t,\ub} }$ ~\textsc{and}~ $\mub_{k,t+1} = \mub_{k,t}  + \rho_t  \ESSb_{k,t+1}^{-1} \textcolor{teal}{\widehat{\nub}^{(\mub_k)}_{\omega_t,\Lambdab_t,\ub}}$.} 
\For {$k=1\!:\!(K\!-\!1)$}{
\textsc{Compute} $\widehat{\gamma}^{(\pi_k)}_{\omega_t,\Lambdab_t,B}$. \textcolor{brown}{\#gradient estimator for $\nabla_{\pi_k} \mathcal{L}_{\omega_t}$} \\
\textsc{Update} $v_{k, t+1} =  v_{k, t}  + \rho_t  \widehat{\gamma}^{(\pi_k)}_{\omega_t,\Lambdab_t,B}$.
}
}
\textsc{Compute} $(\pi_{k,T})_{k \in [K]}$ from $(v_{k,T})_{k \in [K-1]}$. \\
		\Return $\thetab_T.$ 
\end{algorithm}

\section{Simulation study}
\label{sec:simulations}

In this section, we show simulation results to illustrate our algorithms' ability to answer our initial optimization problem. First, we check their mode-finding performance. Then, we showcase the ability of the mixture weights to capture information on the flatness of the modes.

\subsection{Mode-finding properties}
\label{sub:simu_mode}
\label{sec:computational_cost}

We apply the Hessian-based version of NVA-GM and FS-NVA-GM (Algorithms~\ref{alg:snga_annealing_gaussian} and~\ref{alg:snga_annealing_utility}) on two test functions: a symmetric Gaussian mixture in two dimensions and the four-dimensional Styblinski--Tang function. For FS-NVA-GM, we use the utility values based on the CMA-ES weights described in Appendix~\ref{sub:utility_values}. For simplicity, the same annealing schedules are used for both algorithms, even though their optimal annealing schedules might differ. Our algorithms fitting a mixture of $K$ components are compared against running $K$ independent instances of traditional Stochastic Gradient Ascent (parallel SGA, denoted by pSGA) and CMA-ES (parallel CMA-ES, denoted by pCMA-ES). 

\paragraph{Computational cost.} For all the algorithms considered, we set comparable numbers of evaluations for $\ell$ (or $\nabla \ell$, $\nabla^2 \ell$). For pSGA and NVA-GM, we use a mini-batch size $\tilde{B}$. For pCMA-ES and FS-NVA-GM, the same number of samples is used to compute the gradients, after truncation selection of level $\eta \in (0,1]$. pSGA requires $\tilde{B}KT$ evaluations of $\nabla \ell$. NVA-GM uses $\tilde{B}KT$ evaluations of $\ell$ for the mixture weight updates, which can also be used to update the means and precision matrices through the black-box estimators. If one wishes to use the derivatives of $\ell$ to compute the gradients, $\tilde{B}KT$ further evaluations of $\nabla \ell$, and possibly $\nabla^2 \ell$ are needed. For pCMA-ES and FS-NVA-GM, we increase the mini-batch size by a factor $\eta^{-1}$ compared to pSGA and NVA-GM, ensuring that the effective number of samples used to compute the gradients after truncation selection remains the same. Therefore, pCMA-ES and FS-NVA-GM require $\eta^{-1}\tilde{B}KT$ evaluations of $\ell$. Overall, the computational cost of all five algorithms is $O(\tilde{B}KT)$ evaluations of $\ell$ or its derivatives.

It is noteworthy that the number of evaluations needed for NVA-GM and FS-NVA-GM can be reduced by a factor of $K$ through importance sampling. This consists of using a single sample per iteration to update all components, leading to algorithm variants that we call NVA-GM-IS and FS-NVA-GM-IS. Typically, to compute $\nabla_{\mub} \mathcal{L}_{W_{\lambdab_k}^{(\omega, \Lambdab_t)}}(\lambdab)|_{\lambdab_{k,t}}$ and $\nabla_{S^{-1}} \mathcal{L}_{W_{\lambdab_k}^{(\omega, \Lambdab_t)}}(\lambdab)|_{\lambdab_{k,t}}$ for all $k \in [K]$, \cite{lin2019fast} used samples $\xib_1, \dots, \xib_B$ from the mixture $q_{\Lambdab_t}$ to update the parameters of all components. Using this technique, the number of calls to $\ell$ or its derivatives for NVA-GM-IS and FS-NVA-GM-IS is reduced to $O(\tilde{B}T)$. Table~\ref{tab:complexity_algorithm} summarizes the number of calls to $\ell$ and its derivatives required by these algorithms. For comparison, the number of calls used by $K$ independent runs of Newton--Raphson's method (parallel Newton, pNewton) is indicated, although it is important to note that Newton--Raphson's method is generally unsuitable for non-convex optimization, as it does not distinguish between minima and maxima.

\paragraph{Runtime.} Tables \ref{tab:runtime_sym}-\ref{tab:runtime_st} report the average runtime per iteration of each method on the two examples considered in this analysis, for various values of $K$. For NVA-GM and FS-NVA-GM, we also include fixed-covariance variants to isolate the cost associated with updating the covariance matrices (Hessian-based in NVA-GM and black-box in FS-NVA-GM). In our implementation, the runtime is dominated by mixture-density evaluations: these appear when computing the entropy (NVA-GM and FS-NVA-GM) and its gradient (NVA-GM) at the samples. As a consequence, both algorithms evaluate the entropy at the $B$ samples drawn per component. However, NVA-GM runtime grows approximately quadratically in $K$, because the entropy gradients require component responsibilities for all $K$ components, which are evaluated at the $B = 4$ samples drawn for each component. In contrast, FS-NVA-GM uses utility values rather than gradients and therefore scales approximately linearly in $K$. In the symmetric-mixture example, FS-NVA-GM exhibits a larger overhead because the target function is evaluated more often ($B=16$ instead of 4) and each call involves a mixture-density computation, whereas this effect is much smaller for the Styblinski--Tang function, whose evaluations are cheaper.

\begin{table}
\centering
\begin{tabular}{lcccc}
\toprule
 & \multicolumn{4}{c}{$K$} \\
Method & 2 & 3 & 4 & 5 \\
\midrule
NVA-GM & 3.7 & 6.6 & 9.8 & 14.0 \\
NVA-GM (fc) & 3.2 & 5.7 & 8.8 & 12.3 \\
FS-NVA-GM & 11.0 & 12.1 & 12.8 & 13.5 \\
FS-NVA-GM (fc) & 10.9 & 11.3 & 12.3 & 13.0 \\
pCMA-ES & 0.66 & 1.0 & 1.3 & 1.5 \\
pSGA & 0.39 & 0.58 & 0.78 & 1.0 \\
\bottomrule
\end{tabular}
\caption{Average runtime per iteration on the symmetric triangle mixture model ($d=2$), in milliseconds, for varying numbers of  mixture components $K$. The fixed-covariance version of the algorithm is denoted by (fc).}
\label{tab:runtime_sym}
\end{table}

\begin{table}
\centering
\begin{tabular}{lcccccccccc}
\toprule
 & \multicolumn{10}{c}{$K$} \\
Method & 2 & 4 & 6 & 8 & 10 & 12 & 14 & 16 & 18 & 20 \\
\midrule
NVA-GM & 2.7 & 8.4 & 17.8 & 30.6 & 45.0 & 64.4 & 86.6 & 111.0 & 141.0 & 173.0 \\
NVA-GM (fc) & 2.1 & 6.8 & 14.1 & 24.1 & 38.3 & 53.1 & 70.6 & 91.1 & 116.0 & 143.0 \\
FS-NVA-GM & 1.3 & 2.2 & 4.0 & 5.3 & 7.1 & 9.2 & 11.6 & 14.4 & 17.0 & 20.5 \\
FS-NVA-GM (fc) & 0.76 & 2.1 & 2.9 & 4.3 & 6.0 & 8.3 & 10.4 & 12.4 & 15.1 & 18.4 \\
pCMA-ES & 0.23 & 0.45 & 0.71 & 0.92 & 1.1 & 1.3 & 1.6 & 1.8 & 2.0 & 2.3 \\
pSGA & 0.0042 & 0.0064 & 0.0095 & 0.014 & 0.019 & 0.021 & 0.025 & 0.029 & 0.032 & 0.036 \\
\bottomrule
\end{tabular}
\caption{Average runtime per iteration on Styblinski--Tang's function ($d\!=\!4$), in milliseconds, for varying numbers of  mixture components $K$. The fixed-covariance version of the algorithm is denoted by (fc).}
\label{tab:runtime_st}
\end{table}

\paragraph{Scalability.} For a mixture of $K$ Gaussians in dimension $d$, the dominant dimensional cost arises from the covariance updates. With full covariance matrices, precision-matrix updates require $O(d^3)$ operations per component due to matrix multiplications. Diagonal and isotropic covariances reduce this to $O(d)$. Thus, in addition to function evaluations, the mixture update cost scales as $O(BKT)$ times these dimensional factors.

The dominant memory usage comes from the storage of the precision matrices and their gradient estimators. With full covariance matrices, this requires $O(K d^2)$ entries, while diagonal and isotropic covariance matrices reduce this to $O(K d)$ and $O(K)$ respectively. Per-sample evaluations add $O(BK)$ temporary memory, which is negligible compared to the cost of storing $d \times d$ matrices when $d$ is large. In higher-dimensional settings, using diagonal and isotropic instead of full covariances improves memory scaling. In our experiments, memory is not a limiting factor in the dimensions considered ($d=2,4$), so we use full covariances. 

\paragraph{Choice of annealing schedule and learning rate.} In all experiments, we used the annealing schedule $\omega_t = \omega_1 t^{-\alpha}$ and the learning-rate $\rho_t = \rho_1 (\omega_1/\omega_t)^{\beta}$ introduced in Appendix~\ref{app:hyperparameters}. We started with $\beta = 0.9$ and selected the initial step size $\rho_1$ from a fixed-$\omega$ run to ensure that the component means exhibit visible updates without instability. The parameters $\omega_1$ and $\alpha$ were then chosen by a small preliminary sweep on the target function, using $K$ equal to the actual number of modes: a few short runs (3-5 of 100 iterations) to set $\omega_1$ large enough for early exploration, followed by several full-length test runs (5-10 runs of $T$ iterations), to choose $\alpha$ small enough to avoid collapse while still reaching a sufficiently low final temperature. Between these runs, we fine-tuned $\beta$: if they showed instability, we decreased $\beta$; if updates became too small too quickly, we increased $\beta$. The resulting quadruplet $(\omega_1, \alpha, \rho_1, \beta)$ were then used unchanged for all other values of $K$. A detailed set of heuristics and practical guidelines is provided in Appendix~\ref{app:hyperparameter_procedure}.

\begin{table}
 \caption{Number of calls to $\ell$ or its derivatives for pSGA, pNewton, pCMA-ES, NVA-GM, and FS-NVA-GM. The \textit{black-box} versions of NVA-GM and NVA-GM-IS mean that only evaluations of $\ell$ are performed while \textit{gradient} (resp. \textit{Hessian}) refers to the addition of evaluations of $\nabla \ell$ (resp. $\nabla^2 \ell$). }

\begin{tabular*}{\textwidth}{ @{\extracolsep{\fill}}l c c c }
\toprule
 Number of evaluations of: & $\ell(\xib)$ & $\nabla \ell(\xib)$ &  $\nabla^2 \ell(\xib)$ \\ 
\midrule
  pSGA & 0 & $\tilde{B}KT$ & 0 \\
\midrule
  pNewton & 0 & $\tilde{B}KT$ & $\tilde{B}KT$ \\ 
\midrule
 pCMA-ES & $\eta^{-1}\tilde{B}KT$ & 0 & 0 \\ 
\midrule 
 NVA-GM (Hessian) & $\tilde{B}KT$ & 0 (or $\tilde{B}KT$) & $\tilde{B}KT$  \\
 NVA-GM (gradient) & $\tilde{B}KT$ & $\tilde{B}KT$ & 0  \\
 NVA-GM (black-box) & $\tilde{B}KT$ & 0 & 0  \\
\midrule 
 NVA-GM-IS (Hessian) & $\tilde{B}T$ & 0 (or $\tilde{B}T$) & $\tilde{B}T$  \\
 NVA-GM-IS (gradient) & $\tilde{B}T$ & $\tilde{B}T$ & 0  \\
 NVA-GM-IS (black-box) & $\tilde{B}T$ & 0 & 0  \\
\midrule 
 FS-NVA-GM & $\eta^{-1}\tilde{B}KT$ & 0 & 0  \\
\midrule 
 FS-NVA-GM-IS & $\eta^{-1}\tilde{B}T$ & 0 & 0  \\
\bottomrule
\end{tabular*}
\label{tab:complexity_algorithm}
\end{table}

\paragraph{Metrics.} We use two metrics to assess the performances of the algorithms, the global peak ratio (GPR) and the all-peak ratio (APR), defined as follows: 
\begin{equation}
    \text{GPR} = \frac{1}{ I \times H }\sum_{h = 1}^{H} \text{GF}_h, \quad\text{and}\quad
    \text{APR} = \frac{1}{ J \times H }\sum_{h = 1}^{H} \text{AF}_h, 
    \label{eq:gpr_metric}
\end{equation}
where $H$ algorithm runs with random initializations have been performed and for each run $h$, $\text{GF}_h$ (resp. $\text{AF}_h$) is the number of global (resp. global and local) modes found, out of the $I$ (resp. $J$) admitted by the target function. We consider that a mode $\xib^*$ has been found if there is at least one final component mean $\mub_{k,t}$ in the rectangle $\xib^* \pm \epsilon (1,\dots,1)^T$, where $\epsilon = 10^{-2}$. These two metrics evaluate two desired properties of a multimodal optimization algorithm: identifying the global modes and identifying distinct (global and local) modes. Empirically, NVA-GM-IS and FS-NVA-GM-IS tend to require more iterations to converge and thus are excluded from the comparison.

\paragraph{Symmetric Gaussian mixture.}

We consider $\ell$ to be the log-density of the Gaussian mixture,  having three components, with means located at $\gammab_k = (\sin(\pi/2 + 2k\pi/3), \cos(\pi/2 + 2k\pi/3))$, $k \in \{1, 2, 3\}$ and covariance matrices set to $\sigma^2 \Ib$ where $\sigma^2 = 0.54$. $\ell$ has three global modes at $0.511 \gammab_k$, $k \in \{1, 2, 3\}$ and one local mode at $(0, 0)$.

For $K \in \{2, 3, 4, 5 \}$, we have uniformly drawn $K$ locations in $[-2, 2]^2$. These locations are used as initial component means for NVA-GM and FS-NVA-GM and the initial positions for pSGA and pCMA-ES. In NVA-GM and FS-NVA-GM, we use $T = 5000$, $\pi_{1,0} = \dots = \pi_{K,0} = 1/K$, $\ESSb_{1,0} = \dots = \ESSb_{K,0} = \Ib$, $\omega_t = \omega_1/t$ with $\omega_1 = 1$, and $\rho_t = 10^{-1} (\omega_1/\omega_t)^{0.8}$. For NVA-GM, we use the gradient estimators $\widehat{\gammab}^{(\mub_k, 1)}_{\omega,\Lambdab,B}$ and $\widehat{\gammab}^{(\ESSb_k^{-1}, 2)}_{\omega,\Lambdab,B}$, where $B = 4$, as given in equations (\ref{eq:est_grad_mu_1}) and (\ref{eq:est_grad_s_2}) of Appendix~\ref{app:estimation_gradients}. For FS-NVA-GM, we use the CMA-ES utility values given by~\eqref{eq:cma_es_utility} in Appendix~\ref{sub:utility_values}, with $B = 16$ and $\eta = 1/4$, i.e.~$B_0 = 4$. For pSGA, we use $T^{(\text{SGA})} = 10000$ iterations and a learning rate sequence $\rho_t^{(\text{SGA})} = t^{-1.1/2}$. For pCMA-ES, we use the same population size and number of selected individuals as FS-NVA-GM, i.e.~$B=16$ and $B_0=4$. The other hyperparameters are selected as advised in~\cite{hansen2006cma}. To detect if one mode is found, we check if the coordinates of at least one component mean (or pSGA particle) at the end of the algorithm are equal to the ones of the mode, with a tolerance of $0.1$ in all dimensions. We replicated the experiment 100 times, recovering the GPR and APR metrics.

These estimates are given in Figure~\ref{fig:mode_sym}. We remark that with these parameters, the component means in both NVA-GM and FS-NVA-GM with CMA-ES utility weights almost always find different modes for $K \in \{ 2,3,4 \}$. FS-NVA-GM almost always finds the three global modes, for $K \ge 3$. NVA-GM has slightly lower performance when $K=3$, meaning that it sometimes finds two global modes and the local modes instead of all three global modes, but performs similarly as FS-NVA-GM when $K \ge 4$. Expectedly, although pCMA-ES finds global modes more often than pSGA, it does not find more modes than pSGA. For all values of $K$ considered here, both NVA-GM and FS-NVA-GM heavily outperform pSGA and pCMA-ES with respect to both metrics. 

\begin{figure}[!tb]
\centering
\includegraphics[width=0.9\linewidth]{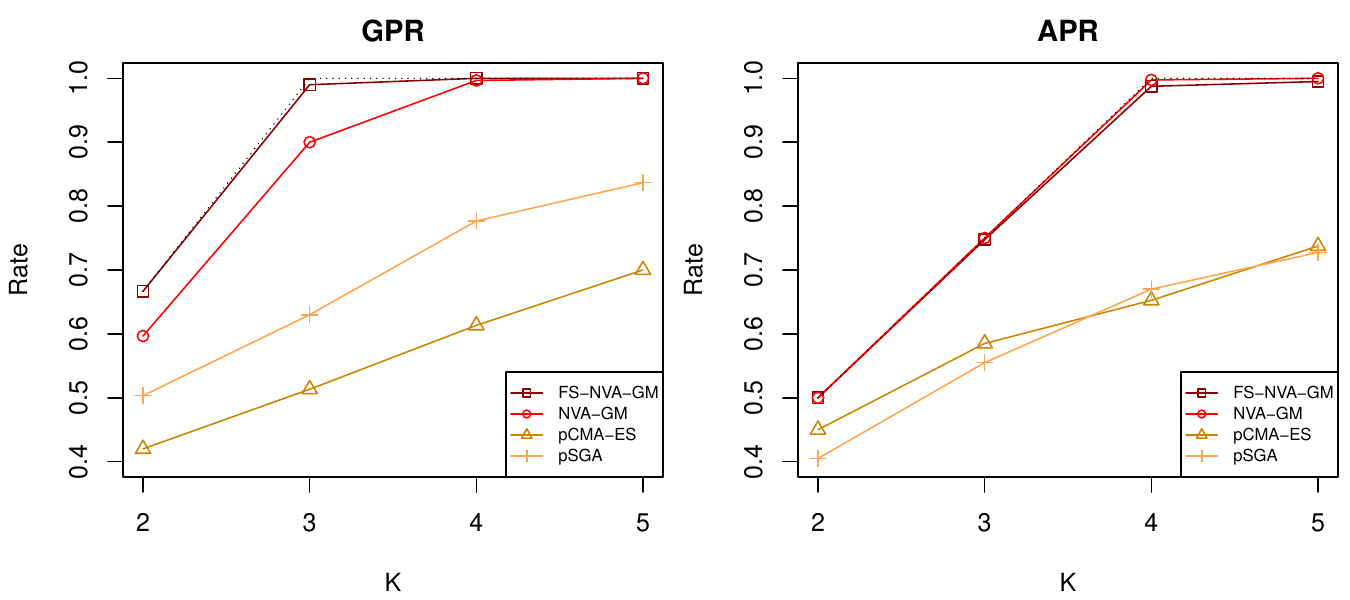}
\caption{Symmetric mixture: The first graph shows that FS-NVA-GM with CMA-ES utility values almost always locates distinct global modes of the symmetric mixture log-density. For $K \ge 3$, it almost always finds all three of them. The second graph indicates that the identified modes are distinct as long as there are at most as many components as (global and local) modes. NVA-GM performs similarly, except when $K \le 3$, where it sometimes misses a global mode, returning the local mode instead. Parallel CMA-ES and parallel SGA struggle more to find distinct modes. The dashed line represents the maximum performance achievable for given values of $K$.}
\label{fig:mode_sym}
\end{figure}

\paragraph{Styblinski--Tang's function.}

Styblinski--Tang's function is defined by $\ell(\xib) = - \sum_{i=1}^d (\xi_i^4 - 16\xi_i^2 + 5\xi_i)/2$. $\ell$ has $2^d - 1$ modes, including one global mode. The locations of the modes are $\{ (x_{i_1}, \dots, x_{i_d}) : (i_1, \dots, i_d) \in \{1, 2\}^d \}$, where $x_1 \approx -2.90$ and $x_2 \approx 2.75$, and the global mode is located at $(x_1, \dots, x_1)$.

We pick $d = 4$ so there are $16$ modes. For $K \in \{ 2, 4, 6, 8, 10, 12, 14, 16, 18, 20 \}$, we have uniformly drawn $K$ locations in $[-4, 4]^d$. These locations are used as initial component means for NVA-GM and FS-NVA-GM and the initial particle positions for pSGA and pCMA-ES. In NVA-GM and FS-NVA-GM, we use $T = 200$, $\pi_{1,0} = \dots = \pi_{K,0} = 1/K$, $\ESSb_{1,0} = \dots = \ESSb_{K,0} = \Ib$, $\omega_t = \omega_1/t^2$ with $\omega_1 = 40000$, and $\rho_t = 10^{-4} \sqrt{\omega_1/\omega_t}$.
All other quantities are set as in the previous mixture example.

The GPR and APR metrics are given in Figure~\ref{fig:mode_st}. The unique mode of Styblinski--Tang's function is almost always found by NVA-GM and FS-NVA-GM for all $K$, and by pCMA-ES for $K \ge 6$. pSGA cannot reliably find the mode even for $K = 20$. Both NVA-GM and FS-NVA-GM show good capability to capture distinct modes. FS-NVA-GM performs slightly worse than NVA-GM, as for $K = 16$, it finds in average $16 \times 0.84 \approx 13.4$ modes out of $16$, against $16 \times 0.92 \approx 14.7$ out of $16$ for NVA-GM. As $K$ is larger, the gap between NVA-GM and FS-NVA-GM on one side and pSGA and pCMA-ES on the other side is larger, with pCMA-ES performing particularly worse than pSGA. This is not surprising, since the CMA-ES has been designed for global optimization, it is more likely to find the global mode rather than a local mode when compared to SGA. 

\begin{figure}[!tb]
\centering
\includegraphics[width=0.9\linewidth]{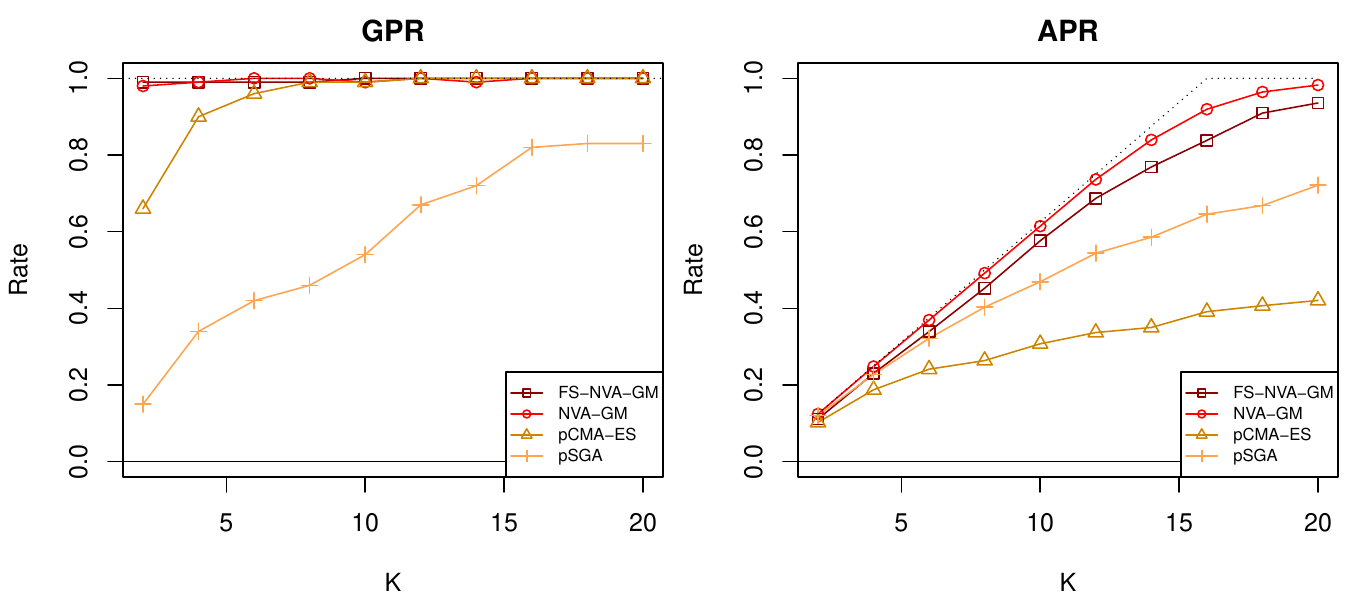}
\caption{Styblinski--Tang: The first graph shows that both NVA-GM and FS-NVA-GM with CMA-ES utility values almost always locate the unique global mode of the 4-dimensional Styblinski--Tang's function for all values of $K \ge 2$. The second graph indicates that the identified modes are distinct most of the time, except when $K$ is close to the total number of modes (16). From this point of view, NVA-GM performs slightly better than FS-NVA-GM due to the latter's greater tendency to favor convergence to the global mode over the 15 other local modes. Parallel CMA-ES and parallel SGA struggle more to find the global mode, as well as diverse local modes, although parallel CMA-ES matches NVA-GM and FS-NVA-GM's performance when $K \ge 8$ in finding the global mode. The dashed line represents the maximum performance achievable for given values of $K$.}
\label{fig:mode_st}
\end{figure}

\paragraph{Other benchmark functions.} Additionally, we have tested FS-NVA-GM on several benchmark functions from the CEC2013 suite for black-box multimodal global optimization \citep{li2013benchmark}. A detailed description and simulation results are provided in Appendix~\ref{app:cec_benchmark}. For simpler functions, FS-NVA-GM achieves good performance. However, for more complex or high-dimensional functions, FS-NVA-GM is limited by its computational cost. In such cases, FS-NVA-GM would also require a significantly larger function evaluation budget to match the performance of current state-of-the-art black-box algorithms \citep{ahrari2017multimodal, maree2018real, denobel2024avoiding}. It is important to note that FS-NVA-GM is only a straightforward implementation of the NVA framework using Gaussian mixtures with fitness shaping and that the NVA framework itself is not primarily designed for black-box problems. Achieving state-of-the-art performance by black-box optimization criteria would certainly require incorporating additional techniques.

\subsection{Limit of mixture weights}

Here, we showcase the properties of the limit mixture weights. Both NVA-GM and FS-NVA-GM have the same weight updates, hence the same properties. We apply NVA-GM (Algorithm~\ref{alg:snga_annealing}) to three objective functions with different mode flatness. With these experiments, we validate the expression of the coefficients specified in Theorem~\ref{th:annealing_limit}. Appendix~\ref{app:simu_degenerate} gives an example in a degenerate case, which is not described by Theorem~\ref{th:annealing_limit}, but discussed in Appendix~\ref{sub:interpretation_weights}.

\paragraph{Symmetric Gaussian mixture.} Here, $\ell$ is the log-density of the symmetric Gaussian mixture described  in Section~\ref{sub:simu_mode}. We apply NVA-GM with $K = 5$, using the same hyperparameters as above, except $T = 10000$ to let the weights converge further. We recover the positions of the component means and the value of the weights at each iteration.

Figure~\ref{fig:weight_sym} illustrates the trajectories of the component means and the evolution of the mixture weights during a run where each mode is successfully identified by a distinct mean. We see that component means 1, 3, and 4 converge to the three global modes, whereas component mean 2 converges to the local mode. According to Theorem~\ref{th:annealing_limit}, the weights of the search mixture should be shared between components 1, 3, and 4. The target being symmetric, the three global modes have the same Hessian determinant, so all three weights should converge to 1/3, which is the case in this run.

\begin{figure}[!tb]
\centering
\includegraphics[width=0.9\linewidth]{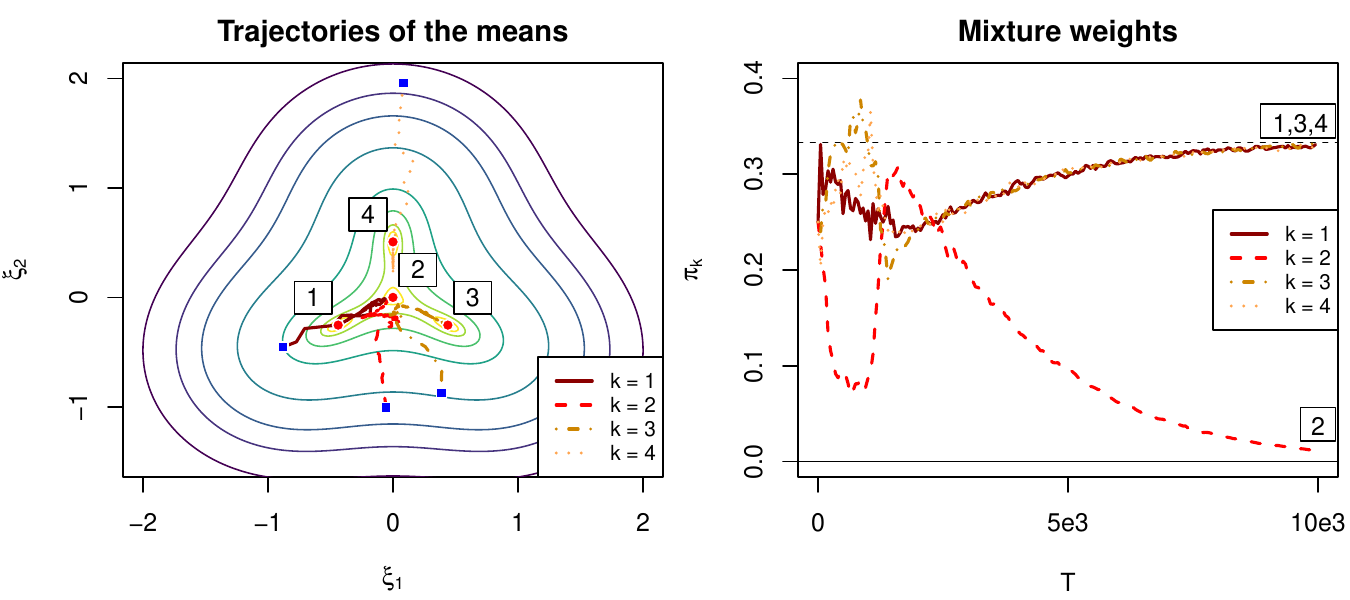}
\caption{Left: The trajectories of the means in a run of NVA-GM on the log-density of the symmetric mixture with $K = 4$ show that the four components track different modes. The contour plot represents non-equally spaced levels of $\ell$. Right: As expected, the weights of components tracking global modes (1, 3, 4) converge to $1/3$, whereas the weight of the component tracking the local mode (2) vanishes. }
\label{fig:weight_sym}
\end{figure}

\paragraph{Asymmetric Gaussian mixture.} We consider the log-density of a Gaussian mixture with three components with means $\gammab_1 = (-1, 0)$, $\gammab_2 = (0, 0)$ and $\gammab_3 = (1, 0)$, covariance matrices $\Sigmab_1=\text{diag}(0.3,0.6)$, $\Sigmab_2=\text{diag}(0.005,0.9)$, $\Sigmab_3=\text{diag}(0.03,0.3)$, 
and weights $\pi_2 = 0.1$, $\pi_1 = (1-\pi_2) \det(\Sigmab_3)^{-1/2}/(\det(\Sigmab_2)^{-1/2} + \det(\Sigmab_3)^{-1/2}) \approx 0.527$ and $\pi_3 = 1 - \pi_1 - \pi_2 \approx 0.373$. This mixture is designed so that its components are well-separated. It has two global modes at $\gammab_1$ and $\gammab_3$ and one local mode in-between at $\gammab_2$. However, the curvatures at the modes differ.

We apply NVA-GM with $K = 3$, using $T = 1000$, $\pi_{1,0} = \pi_{2,0} = \pi_{3,0} = 1/3$, $\ESSb_{1,0} = \ESSb_{2,0} = \ESSb_{3,0} = \Ib$, $\omega_t = \omega_1/t$ with $\omega_1 = 100$, and $\rho_t = 10^{-3} (\omega_1/\omega_t)^{0.8}$. For the gradients, we use the estimators $\widehat{\gammab}^{(\mub_k, 1)}_{\omega,\Lambdab,B}$ and $\widehat{\gammab}^{(\ESSb_k^{-1}, 2)}_{\omega,\Lambdab,B}$, where $B = 4$. We recover the positions of the component means and the value of the weights at each iteration.

Figure~\ref{fig:weight_asym} illustrates the trajectories of the component means and the evolution of the mixture weights during a run where each mode is successfully identified by a distinct mean. We notice that component mean 2 converges to the local mode near $\gamma_2$, hence its weight vanishes. The mixture weight is shared between components 1 and 3, converging to the two global modes, near $\gamma_1$ and $\gamma_3$ respectively. According to Theorem~\ref{th:annealing_limit}, the weight of component 1 should converge to $\tilde{c}_{1} \approx 0.586$. Similarly, the weight of component 3 should converge to $\tilde{c}_{1} \approx 0.414$. These claims are supported by the behavior of the weights across this run.

\begin{figure}[!tb]
\centering
\includegraphics[width=0.9\linewidth]{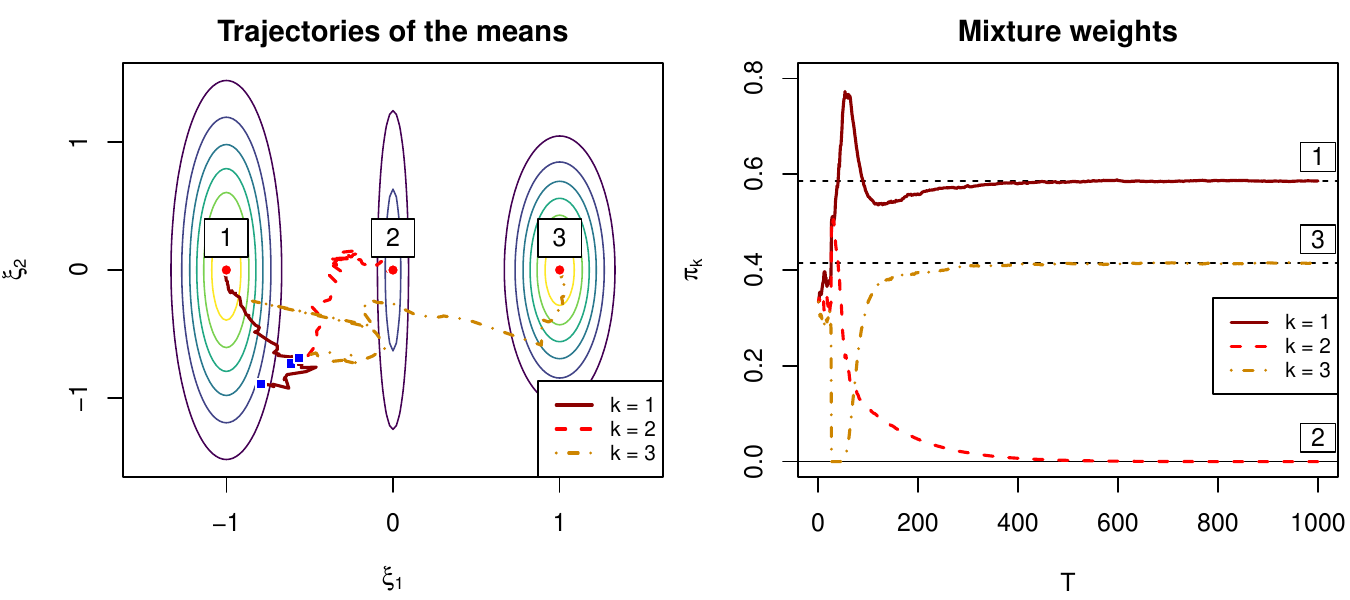}
\caption{Left: The trajectories of the means in a run of NVA-GM on the log-density of the asymmetric mixture with $K = 3$ show that the three components track different modes. The contour plot represents non-equally spaced levels of $\ell$. Right: As expected, the weights of components tracking global modes (1, 3) converge to their respective limits $\tilde{c}_1 \approx 0.586$ and $\tilde{c}_3 \approx 0.414$, whereas the weight of the component tracking the local mode (2) vanishes.  }
\label{fig:weight_asym}
\end{figure}

\section{Remote sensing inverse problem in planetary science}
\label{sec:illustration}

In general, the modes of a Gaussian mixture cannot be found analytically \citep{carreira2000mode, ray2005topography}. To illustrate the practical interest of finding the modes of a Gaussian mixture in applications, we consider a remote sensing task, which consists of characterizing, with a few meaningful parameters,  sites at the surface of a planet. 
The composition of the surface materials is generally established, 
with remote sensing techniques, from images produced by hyperspectral cameras, from different angles during a site flyover.
An example for the planet Mars is described by
\cite{murchieCompactReconnaissanceImaging2009,fernandoMartianSurfaceMicrotexture2016}.
These images carry information on the surface parameters which are extracted by inverting a model of radiative transfer.
The Hapke model is such a photometric model, that relates physically meaningful parameters to the reflectivity of a granular material for a given geometry of illumination and viewing. Formally, in our illustration, it links a set of parameters $\Psib \in \Rset^4$ to a {\it theoretical} surface Bidirectional Reflectance Distribution Factor (BRDF) denoted by $\yv = F_{\text{Hapke}}(\Psib) \in \Rset^{10}$ corresponding to 10 geometries. 
This number of geometries is typical of real observations for which the number of possible measurements during a planet flyover is limited. 
The parameters are $\Psib = (\psi_1, \psi_2, \psi_3, \psi_4)$, representing respectively the single scattering albedo, macroscopic roughness, asymmetry parameter and backscattering fraction. More details on these quantities and their photometric meanings may be found in \cite{schmidtRealisticUncertaintiesHapke2015,labar2017}. Although available, the expression of $F_{\text{Hapke}}$ is very complex and tedious to handle analytically, with many approximations required  \citep{labar2017}. In practice, it is therefore mainly used via a numerical code, allowing simulations from the model.

Previous studies \citep{Kugler2020,schmidtRealisticUncertaintiesHapke2015,forbes2022summary} have shown evidence for the existence of multiple solutions or for the possibility of obtaining very similar observations from different sets of parameters, which makes this setting appropriate for testing the ability of our procedures to recover multiple solutions. 
To do so, a training set of $10^5$  pairs $(\Psib, \yv)$ of parameters and observations are available and correspond to the application of the Hapke simulator to each $\thetab$ in the training set. 
We use the same experimental setting and preprocessing described in \cite{Kugler2020,forbes2022summary}. The approach of \cite{Deleforge2015} is used to learn a parametric approximation of the posterior distributions.  More specifically, we use the so-called GLLiM model to produce for each possible observation $\yv$ an approximate posterior distribution on $\Psib$, $p(\Psib | \yv)$, which is a Gaussian mixture model with $K=40$ components in dimension 4.

In this illustration, we focus on the characterization of one specific observation $\yv_{\text{o}}$ coming from a mineral called Nontronite which has been chosen as likely to exhibit multiple $\Psib$ solutions (see \cite{Kugler2020} for a description).  The hope is then to recover these likely solutions as the modes of the  40-component Gaussian mixture posterior distribution, which approximates the desired posterior $p(\Psib | \yv_{\text{o}})$, as provided by GLLiM.
We thus apply the NVA-GM algorithm to find these mixture modes. 
As the number of modes is a priori unknown \citep{carreira2003isotropic, carreira2003number, aprausheva2006bounds, amendola2020maximum}, we ran NVA-GM 10 times with $K=10$ and found consistently 4 candidate modes at most, the 4th one being much smaller than the others and more difficult to find. The 4 solutions found are $\Psib_1=(0.584, 0.296, 0.201, 0.143)$, $\Psib_2=(0.676, 0.317, 0.627, 0.030)$, $\Psib_3=(0.583, 0.309, 0.017, 0.614)$ and $\Psib_4=(0.729, 0.232, 0.992, 0.159)$.  Their respective log-density values are 
6.16, 5.17, 4.35, and 1.91, reflecting the mode order. 

Without ground truth, it is difficult to fully validate these estimations. However, a simple inspection consists of checking the reconstructed signals obtained by applying the Hapke model to the 4 sets of estimated parameters and comparing them with the inverted signal $\yv_{\text{o}}$. Figure \ref{fig:Hapkeobs} shows the inverted signal compared to the 4 reconstructed signals. 
The proximity of the reconstructions confirms the existence of multiple solutions and thus the relevance of a multimodal posterior. One solution can be selected by choosing the parameters that provide the best reconstruction assessed for instance by the smallest mean square error (MSE). For the 4 solutions, we get respectively MSE$_1 = 0.00338$,  MSE$_2 = 0.00343$,  MSE$_3 = 0.0115$, and  MSE$_4 = 0.0173$.  Figure \ref{fig:Hapkeobs}  is visually consistent with the MSE values, with the $\Psib_1$ and $\Psib_2$ solutions being close to the observed signal and within a $\pm 0.05$ band corresponding to the standard deviation of the noise added to the training set. The two other solutions are less satisfying in terms of shape although not too far from $\yv_{\text{o}}$ either.  Figure \ref{fig:Hapkeobs} also shows in the red dashed line the signal in the training set that is the most correlated to $\yv_{\text{o}}$. This signal is much further away from $\yv_{\text{o}}$ showing the gain provided by considering posterior modes instead.

\begin{figure}[!tb]
\centering
\includegraphics[trim={0.1cm 1.8cm 0.5cm 0.5cm},clip,width=0.55\linewidth]{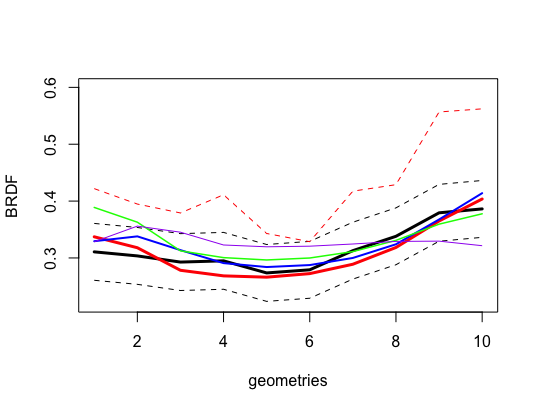}\\
\vspace{-3mm}
{\small Geometry index}\\
\vspace{2mm}
\caption{Remote sensing illustration: BRDF signal reconstructions. The Nontronite observed signal $\yv_{\text{o}}$ (black) is compared to the reconstructed signals from the 4 modes found by NVA-GM, $\Psib_1$ (red), $\Psib_2$ (blue), $\Psib_3$ (green) and $\Psib_4$ (purple). The black dashed lines show a band of $\pm 0.05$ around $\yv_{\text{o}}$. The red dashed line shows for comparison the signal in the training set with the highest correlation to $\yv_{\text{o}}$.}
\label{fig:Hapkeobs}
\end{figure}

\section{Conclusion}
\label{sec:conclusion}

In this paper, we proposed a novel, principled framework for multimodal optimization, contrasting with state-of-the-art approaches using niching, particle swarms, or restart strategies. This framework uses principles from variational inference to reformulate the initial optimization task as a variational problem, an optimization task focused on the parameters of a search distribution. Geometry-informed methods, such as natural gradient ascent or the improved Bayesian Learning Rule, are used to solve this variational problem. An entropy penalty is introduced to the variational objective, promoting convergence of the means of the mixture components to different optima. 

We derived a baseline algorithm for this framework, using a variational family of Gaussian mixtures, the NVA-GM algorithm. We proposed several extensions, including Hessian-free and black-box (derivative-free) versions. We also incorporated fitness shaping into the black-box version, resulting in the FS-NVA-GM algorithm. In our experience, the Hessian-based NVA-GM and FS-NVA-GM versions perform the best, although these algorithms face challenges when the number of modes $K$ is large, as the number of objective function evaluations scales linearly with $K$. Importance sampling can be used to mitigate this dependency on $K$, although it requires a larger mini-batch size, and finding the optimal proposal distribution at each iteration is an open problem. These versions of NVA-GM and FS-NVA-GM also suffer in high-dimensional settings ($d$ is large) due to the large covariance matrices of the Gaussian mixture, inducing high computational cost. However, constraining the variational family, for example, to Gaussian mixtures with diagonal or isotropic Gaussian covariance matrices, can address this issue. 

In summary, the NVA-M framework introduces a novel approach to multimodal optimization with variational optimization as a central principle, providing a foundation for developing more sophisticated problem-specific algorithms. For example, as variational optimization gains traction in deep learning \citep{osawa2019practical, shen2024variational}, it would be interesting to adapt a multimodal optimization algorithm suitable to handle similarly challenging problems. Future work then includes developing solutions to reduce the computational cost and extend the framework's scalability. Mode-finding performance can also be improved by incorporating new techniques or heuristics inspired by existing methods. Finally, further research may focus on the development of adaptive learning rates and annealing schedules, which are currently tuned by trial-and-error due to their problem-dependent nature.

\section*{Acknowledgements}

The authors would like to thank Jacopo Iollo, Yuesong Shen and Asuka Takatsu for fruitful discussions at different stages of this project. 

\bibliographystyle{apalike}
\bibliography{biblio}

@ARTICLE{Liang2006,
  author={Liang, J. J. and Qin, A. K. and Suganthan, P. N. and Baskar, S.},
  journal={IEEE Transactions on Evolutionary Computation}, 
  title={Comprehensive learning particle swarm optimizer for global optimization of multimodal functions}, 
  year={2006},
  volume={10},
  number={3},
  pages={281-295}
}

@article{Deleforge2015,
  title = {High-Dimensional Regression with {G}aussian Mixtures and Partially-Latent Response Variables},
  volume = {25},
    number = {5},
  journal = {Statistics and Computing},
   author = {Deleforge, A. and Forbes, F. and Horaud, R.},
  month = sep,
  year = {2015},
  pages = {893-911}
  }

@article{Kugler2020,
  TITLE = {Fast {B}ayesian Inversion for high dimensional inverse problems},
  AUTHOR = {Kugler, B. and Forbes, F. and Dout{\'e}, S.},
  journal = {Statistics and Computing},
  volume={32},
  number={2},
  pages={31},
  year={2022}
}

@article{forbes2022summary,
  title={{Summary statistics and discrepancy measures for approximate Bayesian computation via surrogate posteriors}},
  author={Forbes, F. and Nguyen, H. D. and Nguyen, T. and Arbel, J.},
  journal={Statistics and Computing},
  volume={32},
  number={5},
  pages={85},
  year={2022}
}

@article{schmidtRealisticUncertaintiesHapke2015,
  title = {Realistic Uncertainties on {{Hapke}} Model Parameters from Photometric Measurements},
  volume = {260},
  journal = {Icarus},
  author = {Schmidt, F. and Fernando, J.},
  year = {2015},
  pages = {73-93}
}

@article{fernandoMartianSurfaceMicrotexture2016,
  title = {Martian Surface Microtexture from Orbital {{CRISM}} Multi-Angular Observations: {{A}} New Perspective for the Characterization of the Geological Processes},
  shorttitle = {Martian Surface Microtexture from Orbital {{CRISM}} Multi-Angular Observations},
  author = {Fernando, J. and Schmidt, F. and Dout{\'e}, S.},
  year = {2016},
  volume = {128},
  pages = {30--51},
  journal = {Planetary and Space Science}
}

@article{murchieCompactReconnaissanceImaging2009,
  title = {Compact {{Reconnaissance Imaging Spectrometer}} for {{Mars}} Investigation and Data Set from the {{Mars Reconnaissance Orbiter}}'s Primary Science Phase},
  author = {Murchie, S. L. and Seelos, F. P. and Hash, C. D. and Humm, D. C. and Malaret, E. and McGovern, J. A. and Choo, T. H. and Seelos, K. D. and Buczkowski, D. L. and Morgan, M. F. and Barnouin-Jha, O. S. and Nair, H. and Taylor, H. W. and Patterson, G. W. and Harvel, C. A. and Mustard, J. F. and Arvidson, R. E. and McGuire, P. and Smith, M. D. and Wolff, M. J. and Titus, T. N. and Bibring, J. and Poulet, F.},
  year = {2009},
  volume = {114},
  pages = {E00D07},
  journal = {Journal of Geophysical Research: Planets},
  language = {en},
  number = {E2}
}

@phdthesis{labar2017,
title = "Caract\'erisation et mod\'elisation de la rugosit\'e multi-\'echelle des surfaces naturelles par t\'el\'ed\'etection dans le domaine solaire",
author = "Labarre, S.",
year = "2017",
school = {Universit\'e Sorbonne Paris Cit\'e}
}

@article{wierstra2014natural,
  title = {Natural Evolution Strategies},
  author = {Wierstra, D. and Schaul, T. and Glasmachers, T. and Sun, Y. and Peters, J. and Schmidhuber, J.},
  journal = {Journal of Machine Learning Research},
  volume = {15},
  pages = {949-980},
  year = {2014}
}

@article{knoblauch2022optimization,
author = {Knoblauch, J. and Jewson, J. and Damoulas, T.},
title = {An Optimization-Centric View on {B}ayes' Rule: Reviewing and Generalizing Variational Inference},
year = {2022},
issue_date = {January 2022},
publisher = {JMLR.org},
volume = {23},
number = {1},
issn = {1532-4435},
journal = {Journal of Machine Learning Research},
month = {jan},
articleno = {132},
numpages = {109}
}

@article{carreira2000mode,
  title = {Mode-Finding for Mixtures of {{Gaussian}} Distributions},
  author = {Carreira-Perpi{\~n}{\'a}n, M. {\'A}.},
  year = {2000},
  volume = {22},
  pages = {1318--1323},
  issn = {1939-3539},
   journal = {IEEE Transactions on Pattern Analysis and Machine Intelligence},
  number = {11}
}

@article{ray2005topography,
  title = {The Topography of Multivariate Normal Mixtures},
  author = {Ray, S. and Lindsay, B. G.},
  year = {2005},
  volume = {33},
  pages = {2042--2065},
  journal = {The Annals of Statistics},
  number = {5}
}

@article{lin2019stein,
  title={Stein's lemma for the reparameterization trick with exponential family mixtures},
  author={Lin, W. and Khan, M. E. and Schmidt, M.},
  journal={arXiv preprint arXiv:1910.13398},
  year={2019}
}

@article{lin2019fast,
  title = {Fast and Simple Natural-Gradient Variational Inference with Mixture of Exponential-family Approximations},
  author = {Lin, W. and Khan, M. E. and Schmidt, M.},
  journal = {International Conference on Machine Learning},
  year = {2019}
}

@article{khan2023bayesian,  
  author = {Khan, M. E. and Rue, H.},
  title = {The {B}ayesian Learning Rule},
  journal = {Journal of Machine Learning Research},
  number={281},
  volume={24},
  pages={1--46},
  year = {2023}
  }

@article{ollivier2017information,
  title={Information-geometric optimization algorithms: A unifying picture via invariance principles},
  author={Ollivier, Y. and Arnold, L. and Auger, A. and Hansen, N.},
  journal={Journal of Machine Learning Research},
  volume={18},
  number={18},
  pages={1--65},
  year={2017}
}

@article{bissiri2016general,
  title={A general framework for updating belief distributions},
  author={Bissiri, P. G. and Holmes, C. C. and Walker, S. G.},
  journal={Journal of the Royal Statistical Society: Series B},
  volume={78},
  number={5},
  pages={1103--1130},
  year={2016},
  publisher={Wiley Online Library}
}

@article{zhang2023progo,
  title={Pro{GO}: Probabilistic Global Optimizer},
  author={Zhang, X. and Ghosh, S.},
  journal={arXiv preprint arXiv:2310.04457},
  year={2023}
}

@article{zellner1988optimal,
  title={Optimal information processing and {B}ayes's theorem},
  author={Zellner, A.},
  journal={The American Statistician},
  volume={42},
  number={4},
  pages={278--280},
  year={1988},
  publisher={Taylor \& Francis}
}

@article{ibrahim2000power,
  title={Power prior distributions for regression models},
  author={Ibrahim, J. G. and Chen, M.-H.},
  journal={Statistical Science},
  volume={15},
  number={1},
  pages={46--60},
  year={2000},
  publisher={JSTOR}
}

@article{jiang2008gibbs,
  title={Gibbs posterior for variable selection in high-dimensional classification and data mining},
  author={Jiang, W. and Tanner, M. A.},
  journal={Annals of Statistics},
  volume={36},
  number={5},
  pages={2207--2231},
  year={2008},
  publisher={Institute of Mathematical Statistics}
}

@article{geman1984stochastic,
  title={Stochastic relaxation, {G}ibbs distributions, and the {B}ayesian restoration of images},
  author={Geman, S. and Geman, D.},
  journal={IEEE Transactions on Pattern Analysis and Machine Intelligence},
  volume={6},
  number={6},
  pages={721--741},
  year={1984},
  publisher={IEEE}
}

@article{furuya2024theoretical,
  title={Theoretical Error Analysis of Entropy Approximation for {G}aussian Mixture},
  author={Furuya, T. and Kusumoto, H. and Taniguchi, K. and Kanno, N. and Suetake, K.},
  journal={arXiv preprint arXiv:2202.13059},
  year={2024}
}

@article{amari1998natural,
  title={Natural gradient works efficiently in learning},
  author={Amari, S.-I.},
  journal={Neural Computation},
  volume={10},
  number={2},
  pages={251--276},
  year={1998},
  publisher={MIT Press}
}

@article{sato2001online,
  title={Online model selection based on the variational {B}ayes},
  author={Sato, M.-A.},
  journal={Neural Computation},
  volume={13},
  number={7},
  pages={1649--1681},
  year={2001},
  publisher={MIT Press One Rogers Street, Cambridge, MA 02142-1209, USA journals-info~…}
}

@article{honkela2007natural,
  title={Natural conjugate gradient in variational inference},
  author={Honkela, A. and Tornio, M. and Raiko, T. and Karhunen, J.},
  journal={International Conference on Neural Information Processing},
  year={2007}
}

@article{hoffman2013stochastic,
  title={Stochastic variational inference},
  author={Hoffman, M. D. and Blei, D. M. and Wang, C. and Paisley, J.},
  journal={Journal of Machine Learning Research},
  volume={14},
  number={1},
  pages={1303--1347},
  year={2013}
}

@article{bonnet1964transformations,
  title={Transformations des signaux al{\'e}atoires {\`a} travers les syst{\`e}mes non lin{\'e}aires sans m{\'e}moire},
  author={Bonnet, G.},
  journal={Annales des T{\'e}l{\'e}communications},
  volume={19},
  pages={203--220},
  year={1964}
}

@article{price1958useful,
  title={A useful theorem for nonlinear devices having {G}aussian inputs},
  author={Price, R.},
  journal={IRE Transactions on Information Theory},
  volume={4},
  number={2},
  pages={69--72},
  year={1958},
  publisher={IEEE}
}

@article{hansen2001completely,
  title={Completely derandomized self-adaptation in evolution strategies},
  author={Hansen, N. and Ostermeier, A.},
  journal={Evolutionary Computation},
  volume={9},
  number={2},
  pages={159--195},
  year={2001},
  publisher={MIT Press}
}

@article{andersen2002bayesian,
  title={Bayesian object recognition for the analysis of complex forest scenes in airborne laser scanner data},
  author={Andersen, H.-E. and Reutebuch, S. E. and Schreuder, G. F.},
  journal={International Archives of Photogrammetry Remote Sensing and Spatial Information Sciences},
  volume={34},
  number={3/A},
  pages={35--41},
  year={2002},
  publisher={Citeseer}
}

@article{girolami2008bayesian,
  title={Bayesian inference for differential equations},
  author={Girolami, M.},
  journal={Theoretical Computer Science},
  volume={408},
  number={1},
  pages={4--16},
  year={2008},
  publisher={Elsevier}
}

@article{ohagan1995fractional,
  title={Fractional {B}ayes factors for model comparison},
  author={O'Hagan, A.},
  journal={Journal of the Royal Statistical Society: Series B},
  volume={57},
  number={1},
  pages={99--118},
  year={1995},
  publisher={Wiley Online Library}
}

@article{gilks1995fractional,
  title={Fractional {B}ayes factors for model comparison: Discussion},
  author={Gilks, W. R.},
  journal={Journal of the Royal Statistical Society: Series B},
  volume={57},
  number={1},
  pages={118--120},
  year={1995},
  publisher={Wiley Online Library}
}

@article{friel2008marginal,
  title={Marginal likelihood estimation via power posteriors},
  author={Friel, N. and Pettitt, A. N.},
  journal={Journal of the Royal Statistical Society: Series B},
  volume={70},
  number={3},
  pages={589--607},
  year={2008},
  publisher={Oxford University Press}
}

@article{rigon2023generalized,
  title={A generalized {B}ayes framework for probabilistic clustering},
  author={Rigon, T. and Herring, A. H. and Dunson, D. B.},
  journal={Biometrika},
  volume={110},
  number={3},
  pages={559--578},
  year={2023},
  publisher={Oxford University Press}
}

@article{fong2020marginal,
  title={On the marginal likelihood and cross-validation},
  author={Fong, E. and Holmes, C. C.},
  journal={Biometrika},
  volume={107},
  number={2},
  pages={489--496},
  year={2020},
  publisher={Oxford University Press}
}

@article{agnoletto2025bayesian,
  title={Bayesian inference for generalized linear models via quasi-posteriors},
  author={Agnoletto, D. and Rigon, T. and Dunson, D. B.},
  journal={Biometrika},
  volume={112},
  number={2},
  pages={asaf022},
  year={2025},
  publisher={Oxford University Press}
}

@article{tierney1989fully,
  title={Fully exponential {L}aplace approximations to expectations and variances of nonpositive functions},
  author={Tierney, L. and Kass, R. E. and Kadane, J. B.},
  journal={Journal of the American Statistical Association},
  volume={84},
  number={407},
  pages={710--716},
  year={1989},
  publisher={Taylor \& Francis}
}

@article{staines2012variational,
  title={Variational optimization},
  author={Staines, J. and Barber, D.},
  journal={arXiv preprint arXiv:1212.4507},
  year={2012}
}

@article{stein1981estimation,
  title={Estimation of the mean of a multivariate normal distribution},
  author={Stein, C. M.},
  journal={The Annals of Statistics},
  pages={1135--1151},
  volume={9},
  number={6},
  year={1981},
  publisher={JSTOR}
}

@inproceedings{sun2009efficient,
  title={Efficient natural evolution strategies},
  author={Sun, Y. and Wierstra, D. and Schaul, T. and Schmidhuber, J.},
  booktitle={Conference on Genetic and Evolutionary Computation},
  year={2009}
}

@article{beyer2014convergence,
  title={Convergence analysis of evolutionary algorithms that are based on the paradigm of information geometry},
  author={Beyer, H.-G.},
  journal={Evolutionary Computation},
  volume={22},
  number={4},
  pages={679--709},
  year={2014},
  publisher={MIT Press}
}

@book{beyer2001theory,
  title={Theory of Evolution Strategies},
  author={Beyer, H.-G.},
  year={2001},
  publisher={Springer}
}

@inproceedings{jebalia2010log,
  title={Log-Linear Convergence of the Scale-Invariant ($\mu$/$\mu_w$,$\lambda$)-{ES} and Optimal $\mu$ for Intermediate Recombination for Large Population Sizes},
  author={Jebalia, M. and Auger, A.},
  booktitle={Parallel Problem Solving from Nature},
  year={2010},
  organization={Springer}
}

@article{arnold2006weighted,
  title={Weighted multirecombination evolution strategies},
  author={Arnold, D. V.},
  journal={Theoretical Computer Science},
  volume={361},
  number={1},
  pages={18--37},
  year={2006},
  publisher={Elsevier}
}

@article{larson2018asynchronously,
  title={Asynchronously parallel optimization solver for finding multiple minima},
  author={Larson, J. and Wild, S. M.},
  journal={Mathematical Programming Computation},
  volume={10},
  pages={303--332},
  year={2018},
  publisher={Springer}
}

@article{custodio2015glods,
  title={{GLODS}: global and local optimization using direct search},
  author={Cust{\'o}dio, A. L. and Madeira, J. A.},
  journal={Journal of Global Optimization},
  volume={62},
  pages={1--28},
  year={2015},
  publisher={Springer}
}

@article{rinnooy1987stochastic,
  title={Stochastic global optimization methods, part {II}: Multi level methods},
  author={Rinnooy Kan, A. H. G. and Timmer, G. T.},
  journal={Mathematical Programming},
  volume={39},
  pages={57--78},
  year={1987},
  publisher={Springer}
}

@article{regis2013quasi,
  title={A quasi-multistart framework for global optimization of expensive functions using response surface models},
  author={Regis, R. G. and Shoemaker, C. A.},
  journal={Journal of Global Optimization},
  volume={56},
  pages={1719--1753},
  year={2013},
  publisher={Springer}
}

@article{lee2011variable,
  title={Variable mesh adaptive direct search algorithm applied for optimal design of electric machines based on {FEA}},
  author={Lee, D. and Kim, J.-W. and Lee, C.-G. and Jung, S.-Y.},
  journal={IEEE Transactions on Magnetics},
  volume={47},
  number={10},
  pages={3232--3235},
  year={2011},
  publisher={IEEE}
}

@article{qu2012distance,
  title={A distance-based locally informed particle swarm model for multimodal optimization},
  author={Qu, B.-Y. and Suganthan, P. N. and Das, S.},
  journal={IEEE Transactions on Evolutionary Computation},
  volume={17},
  number={3},
  pages={387--402},
  year={2012},
  publisher={IEEE}
}

@inproceedings{li2007multimodal,
  title={A multimodal particle swarm optimizer based on fitness {E}uclidean-distance ratio},
  author={Li, X.},
  booktitle={Conference on Genetic and Evolutionary Computation},
  year={2007}
}

@inproceedings{kennedy1995particle,
  title={Particle swarm optimization},
  author={Kennedy, J. and Eberhart, R.},
  booktitle={International Conference on Neural Networks},
  year={1995}
}

@article{brits2007locating,
  title={Locating multiple optima using particle swarm optimization},
  author={Brits, R. and Engelbrecht, A. P. and van den Bergh, F.},
  journal={Applied Mathematics and Computation},
  volume={189},
  number={2},
  pages={1859--1883},
  year={2007},
  publisher={Elsevier}
}

@article{yamanaka2021simple,
  title={Simple gravitational particle swarm algorithm for multimodal optimization problems},
  author={Yamanaka, Y. and Yoshida, K.},
  journal={PLOS ONE},
  volume={16},
  number={3},
  pages={1-23},
  year={2021},
  publisher={Public Library of Science San Francisco, CA USA}
}

@article{li2016seeking,
  title={Seeking multiple solutions: An updated survey on niching methods and their applications},
  author={Li, X. and Epitropakis, M. G. and Deb, K. and Engelbrecht, A.},
  journal={IEEE Transactions on Evolutionary Computation},
  volume={21},
  number={4},
  pages={518--538},
  year={2016},
  publisher={IEEE}
}

@article{kirkpatrick1983optimization,
  title={Optimization by simulated annealing},
  author={Kirkpatrick, S. and Gelatt Jr, C. D. and Vecchi, M. P.},
  journal={Science},
  volume={220},
  number={4598},
  pages={671--680},
  year={1983}
}

@book{back1997handbook,
  author = {B\"{a}ck, T. and Fogel, D. B. and Michalewicz, Z.},
  title = {Handbook of Evolutionary Computation},
  year = {1997},
  publisher = {IOP Publishing Ltd.},
  edition = {1st}
}

@article{huang2018improving,
  title={Improving explorability in variational inference with annealed variational objectives},
  author={Huang, C.-W. and Tan, S. and Lacoste, A. and Courville, A. C.},
  journal={Advances in Neural Information Processing Systems},
  year={2018}
}

@inproceedings{dangelo2021annealed,
  title={Annealed {S}tein variational gradient descent},
  author={D'Angelo, F. and Fortuin, V.},
  booktitle={Advances in Approximate Bayesian Inference},
  year={2021}
}

@incollection{hansen2006cma,
  title={The {CMA} Evolution Strategy: A Comparing Review},
  editor={Lozano, J. A. and Larra{\~{n}}aga, P. and Inza, I. and Bengoetxea, E.},
  author={Hansen, N.},
  booktitle={Towards a New Evolutionary Computation: Advances in the Estimation of Distribution Algorithms},
  volume={192},
  pages={75--102},
  year={2006},
  publisher={Springer}
}

@article{donsker1976asymptotic,
  author = {Donsker, M. D. and Varadhan, S. R. S.},
  title = {Asymptotic evaluation of certain {M}arkov process expectations for large time—{III}},
  journal = {Communications on Pure and Applied Mathematics},
  volume = {29},
  number = {4},
  pages = {389-461},
  doi = {https://doi.org/10.1002/cpa.3160290405},
  url = {https://onlinelibrary.wiley.com/doi/abs/10.1002/cpa.3160290405},
  eprint = {https://onlinelibrary.wiley.com/doi/pdf/10.1002/cpa.3160290405},
  year = {1976}
}

@book{kullback1959information,
  title={Information Theory and Statistics},
  author={Kullback, S.},
  year={1959},
  publisher={John Wiley \& Sons}
}

@article{alquier2024user,
  title={User-friendly introduction to {PAC}-{B}ayes bounds},
  author={Alquier, P.},
  journal={Foundations and Trends{\textregistered} in Machine Learning},
  volume={17},
  number={2},
  pages={174--303},
  year={2024},
  publisher={Now Publishers, Inc.}
}

@article{wu2024understanding,
  title={Understanding Stochastic Natural Gradient Variational Inference},
  author={Wu, K. and Gardner, J. R.},
  journal={International Conference on Machine Learning},
  year={2024}
}

@inproceedings{lin2020handling,
  title={Handling the positive-definite constraint in the {B}ayesian Learning Rule},
  author={Lin, W. and Schmidt, M. and Khan, M. E.},
  booktitle={International Conference on Machine Learning},
  year={2020}
}

@inproceedings{rezende2014stochastic,
  title={Stochastic backpropagation and approximate inference in deep generative models},
  author={Rezende, D. J. and Mohamed, S. and Wierstra, D.},
  booktitle={International Conference on Machine Learning},
  year={2014}
}

@article{opper2009variational,
  title={The variational {G}aussian approximation revisited},
  author={Opper, M. and Archambeau, C.},
  journal={Neural Computation},
  volume={21},
  number={3},
  pages={786--792},
  year={2009}
}

@article{stuart2010inverse,
  title={Inverse problems: a {B}ayesian perspective},
  author={Stuart, A. M.},
  journal={Acta Numerica},
  volume={19},
  pages={451--559},
  year={2010},
  publisher={Cambridge University Press}
}

@book{israelachvili1992intermolecular,
  title={Intermolecular and surface forces},
  author={Israelachvili, J. N.},
  year={1992},
  publisher={Academic Press}
}

@article{long2022multimodal,
  title={Multimodal information gain in {B}ayesian design of experiments},
  author={Long, Q.},
  journal={Computational Statistics},
  volume={37},
  number={2},
  pages={865--885},
  year={2022},
  publisher={Springer}
}

@article{hwang1980laplace,
  title={Laplace's method revisited: weak convergence of probability measures},
  author={Hwang, C.-R.},
  journal={The Annals of Probability},
  volume={8},
  number={6},
  pages={1177--1182},
  year={1980},
  publisher={JSTOR}
}

@article{luo2018minima,
  title={Minima distribution for global optimization},
  author={Luo, X.},
  journal={arXiv preprint arXiv:1812.03457},
  year={2018}
}

@article{serre2025stein,
  title={Stein {B}oltzmann Sampling: A Variational Approach for Global Optimization},
  author={Serr{\'e}, G. and Kalogeratos, A. and Vayatis, N.},
  journal={International Conference on Artificial Intelligence and Statistics},
  year={2025}
}

@book{amari2000methods,
  title={Methods of information geometry},
  author={Amari, S.-I. and Nagaoka, H.},
  volume={191},
  year={2000},
  publisher={American Mathematical Soc.}
}

@article{nielsen2020elementary,
  title={An elementary introduction to information geometry},
  author={Nielsen, F.},
  journal={Entropy},
  volume={22},
  number={10},
  pages={1100},
  year={2020},
  publisher={MDPI}
}

@article{bonnabel2013stochastic,
  title={Stochastic gradient descent on {R}iemannian manifolds},
  author={Bonnabel, S.},
  journal={IEEE Transactions on Automatic Control},
  volume={58},
  number={9},
  pages={2217--2229},
  year={2013},
  publisher={IEEE}
}

@article{tieleman2012lecture,
  title={Lecture 6.5 - {RMS}prop: {D}ivide the gradient by a running average of its recent magnitude},
  author={Tieleman, T. and Hinton, G.},
  journal={COURSERA: Neural networks for machine learning 4},
  year={2012}
}

@inproceedings{kingma2014adam,
  title={Adam: A method for stochastic optimization},
  author={Kingma, D. P. and Ba, J.},
  booktitle={International Conference on Learning Representations},
  year={2015}
}

@article{robbins1951stochastic,
  title={A stochastic approximation method},
  author={Robbins, H. and Monro, S.},
  journal={The Annals of Mathematical Statistics},
  pages={400--407},
  year={1951},
  publisher={JSTOR}
}

@book{bender1999advanced,
  title={Advanced Mathematical Methods for Scientists and Engineers I: Asymptotic Methods and Perturbation Theory},
  author={Bender, C. M. and Orszag, S. A.},
  volume={1},
  year={1999},
  publisher={Springer}
}

@article{hsu1948theorem,
  title={A theorem on the asymptotic behavior of a multiple integral},
  author={Hsu, L. C.},
  journal={Duke Mathematical Journal},
  volume={15},
  number={3},
  pages={623},
  year={1948},
  publisher={Duke University Press}
}

@book{wong2001asymptotic,
  title={Asymptotic approximations of integrals},
  author={Wong, R.},
  year={2001},
  publisher={SIAM}
}

@article{amendola2020maximum,
  title={Maximum number of modes of {Gaussian} mixtures},
  author={Am{\'e}ndola, C. and Engstr{\"o}m, A. and Haase, C.},
  journal={Information and Inference: A Journal of the IMA},
  volume={9},
  number={3},
  pages={587--600},
  year={2020},
  publisher={Oxford University Press}
}

@article{aprausheva2006bounds,
  title={Bounds for the number of modes of the simplest {Gaussian} mixture},
  author={Aprausheva, N. N. and Mollaverdi, N. and Sorokin, S. V.},
  journal={Pattern Recognition and Image Analysis},
  volume={16},
  pages={677--681},
  year={2006},
  publisher={Springer}
}

@inproceedings{carreira2003number,
  title={On the number of modes of a {Gaussian} mixture},
  author={Carreira-Perpi{\~n}{\'a}n, M. {\'A}. and Williams, C. K. I.},
  booktitle={International Conference on Scale-Space Theories in Computer Vision},
  pages={625--640},
  year={2003},
  organization={Springer}
}

@techreport{carreira2003isotropic,
  title={An isotropic {G}aussian mixture can have more modes than components},
  author={Carreira-Perpi{\~n}{\'a}n, M. {\'A}. and Williams, C. K. I.},
  institution={School of Informatics, University of Edinburgh},
  address={UK},
  number={EDI-INF-RR-0185},
  year={2003}
}

@article{li2013benchmark,
  title={Benchmark functions for CEC’2013 special session and competition on niching methods for multimodal function optimization},
  author={Li, X. and Engelbrecht, A. and Epitropakis, M. G.},
  journal={RMIT University, Evolutionary Computation and Machine Learning Group, Australia, Tech. Rep},
  year={2013}
}

@article{osawa2019practical,
  title={Practical deep learning with {B}ayesian principles},
  author={Osawa, K. and Swaroop, S. and Khan, M. E. and Jain, A. and Eschenhagen, R. and Turner, R. E. and Yokota, R.},
  journal={Advances in Neural Information Processing Systems},
  year={2019}
}

@inproceedings{shen2024variational,
  title={Variational learning is effective for large deep networks},
  author={Shen, Y. and Daheim, N. and Cong, B. and Nickl, P. and Marconi, G. M. and Bazan, C. and Yokota, R. and Gurevych, I. and Cremers, D. and Khan, M. E. and M\"ollenhoff, T.},
  booktitle={International Conference on Machine Learning},
  year={2024}
}

@inproceedings{maree2018real,
  title={Real-valued evolutionary multi-modal optimization driven by hill-valley clustering},
  author={Maree, S. C. and Alderliesten, T. and Thierens, D. and Bosman, P. A. N.},
  booktitle={Conference on Genetic and Evolutionary Computation},
  year={2018}
}

@article{ahrari2017multimodal,
  title={Multimodal optimization by covariance matrix self-adaptation evolution strategy with repelling subpopulations},
  author={Ahrari, A. and Deb, K. and Preuss, M.},
  journal={Evolutionary Computation},
  volume={25},
  number={3},
  pages={439--471},
  year={2017},
  publisher={MIT Press}
}

@inproceedings{denobel2024avoiding,
  title={Avoiding Redundant Restarts in Multimodal Global Optimization},
  author={de Nobel, J. and Vermetten, D. and Kononova, A. V. and Shir, O. M. and B{\"a}ck, T.},
  booktitle={Parallel Problem Solving from Nature},
  year={2024}
}

@inproceedings{tao2019variational,
  title={Variational annealing of {GAN}s: A {L}angevin perspective},
  author={Tao, C. and Dai, S. and Chen, L. and Bai, K. and Chen, J. and Liu, C. and Zhang, R. and Bobashev, G. and Duke, L. C.},
  booktitle={International Conference on Machine Learning},
  year={2019},
}

@article{sanokowski2023variational,
  title={Variational annealing on graphs for combinatorial optimization},
  author={Sanokowski, S. and Berghammer, W. and Hochreiter, S. and Lehner, S.},
  journal={Advances in Neural Information Processing Systems},
  year={2023}
}

@article{duchi2011adaptive,
  title={Adaptive subgradient methods for online learning and stochastic optimization},
  author={Duchi, J. and Hazan, E. and Singer, Y.},
  journal={Journal of Machine Learning Research},
  volume={12},
  number={7},
  year={2011}
}

@inproceedings{mandt2016variational,
  title={Variational tempering},
  author={Mandt, S. and McInerney, J. and Abrol, F. and Ranganath, R. and Blei, D.},
  booktitle={Artificial Intelligence and Statistics},
  year={2016},
}

@article{katahira2008deterministic,
  title={Deterministic annealing variant of variational {B}ayes method},
  author={Katahira, K. and Watanabe, K. and Okada, M.},
  journal={Journal of Physics: Conference Series},
  volume={95},
  number={1},
  pages={012015},
  year={2008}
}

@article{yoshida2010bayesian,
  title={Bayesian learning in sparse graphical factor models via variational mean-field annealing},
  author={Yoshida, R. and West, M.},
  journal={The Journal of Machine Learning Research},
  volume={11},
  pages={1771--1798},
  year={2010}
}

@article{gershman2012nonparametric,
  title={Nonparametric variational inference},
  author={Gershman, S. and Hoffman, M. and Blei, D.},
  journal={International Conference on Machine Learning},
  year={2012}
}

@article{daudel2021mixture,
  title={Mixture weights optimisation for alpha-divergence variational inference},
  author={Daudel, K. and Douc, R.},
  journal={Advances in Neural Information Processing Systems},
  year={2021}
}

@article{arenz2023unified,
  title={{A unified perspective on natural gradient variational inference with Gaussian mixture models}},
  author={Arenz, O. and Dahlinger, P. and Ye, Z. and Volpp, M. and Neumann, G.},
  journal={Transactions on Machine Learning Research},
  year={2023}
}

@article{daudel2023monotonic,
  title={Monotonic alpha-divergence minimisation for variational inference},
  author={Daudel, K. and Douc, R. and Roueff, F.},
  journal={Journal of Machine Learning Research},
  volume={24},
  number={62},
  pages={1--76},
  year={2023}
}

@article{petit2025variational,
  title={Variational Inference with Mixtures of Isotropic {G}aussians},
  author={Petit-Talamon, M. and Lambert, M. and Korba, A.},
  journal={arXiv preprint arXiv:2506.13613},
  year={2025}
}

\newpage

\appendix

\startcontents[appendices]
\printcontents[appendices]{l}{1}{\section*{Appendices}\setcounter{tocdepth}{2}}

\begin{appendix}

\section{Strict concavity of the entropy}
\label{app:kl_convexity}

In this section, we prove the strict concavity of the entropy of a continuous distribution $q$, defined by $\mathcal{H}(q) = - \int q(\xib) \log q(\xib) d\xib$. 

\begin{proposition}
    The mapping $q \mapsto -\mathcal{H}(q)$ is strictly convex.
    \label{prop:kl_convexity}
\end{proposition}

\begin{proof}
    Let $q_0$ and $q_1$ be two distinct probability distributions with the same support. For $t \in (0,1)$, define 
    \begin{equation*}
        q_t = (1-t) q_0 + t q_1.
    \end{equation*}
    Our aim is to show that for all $t \in (0,1)$,
    \begin{equation*}
        \Delta_t := -(1-t) \mathcal{H}(q_0) - t \mathcal{H}(q_1) + \mathcal{H}(q_t) > 0.
    \end{equation*}
    Let $g : q \mapsto q \log q$. Since $q_0$, $q_1$ and $q_t$ are absolutely continuous with respect to $p$, then we have
    \begin{equation*}
        \Delta_t = \int [(1-t) g(q_0) + t g(q_1) - g(q_t)](\xib) d \xib.
    \end{equation*}
    To conclude that $\Delta_t > 0$, we write the first order Taylor expansions of $g(q_0)$ and $g(q_1)$.

    For $i \in \{0, 1\}$, define $\varphi_i : u \mapsto (1 - u) q_t + u q_i$, then 
    \begin{equation*}
    \begin{split}
        g(q_i) &= g(\varphi_i(1)) \\
        &= g(\varphi_i(0)) + (\varphi_i(1) - \varphi_i(0)) g'(\varphi(0)) - \int_{0}^1 (\varphi_i(u) - \varphi_i(0)) g''(\varphi_i(u)) d\varphi_i(u) \\
        &= g(q_t) + (q_i - q_t) g'(q_t) + (q_i - q_t)^2 \int_{0}^1 \frac{1}{(1-u) q_t + u q_i} (1 - u) du.
    \end{split}
    \end{equation*}
    Therefore,
    \begin{equation*}
        (1-t) g(q_0) + t g(q_1) - g(q_t) = t (1-t) (q_1 - q_0)^2 \int_{0}^1 \left( \frac{t}{(1-u) q_t + u q_0} + \frac{1-t}{(1-u) q_t + u q_1} \right) (1 - u) du.
    \end{equation*}
    Since $q_0$ and $q_1$ are distinct, it can be deduced that $\Delta_t > 0$.
\end{proof}

\section{Link between the Gibbs measure and the variational problem}
\label{app:gibbs_variational_proof}

In this section, we prove Proposition~\ref{prop:solution_fixed_omega}.
\begin{proof}[Proof of Proposition~\ref{prop:solution_fixed_omega}]
    This proof is similar to the one of \cite{alquier2024user}. Let $q$ be a probability distribution. Consider the Kullback--Leibler divergence between $q$ and $g_\omega$:
    \begin{equation*}
    \begin{split}
        0 &\ge - \KL(q \mid\mid g_\omega) \\
        &= - \int q(\xib) \log\frac{q(\xib)}{g_\omega(\xib)} d \xib \\
        &= \int q(\xib) \frac{\ell(\xib)}{\omega} d \xib - \int q(\xib) \log q(\xib) d \xib - \log Z(\omega) \\
        &= \frac{1}{\omega} \left( \mathbb{E}_q[\ell(\xib)] + \omega \mathcal{H}(q) \right)  - \log Z(\omega).
    \end{split}
    \end{equation*}
    Since $\KL(q \mid\mid g_\omega) = 0$ if and only if $q = g_\omega$, it follows that the quantity $\mathbb{E}_q[\ell(\xib)] + \omega \mathcal{H}(q)$ is maximized if and only if $q = g_\omega$.
\end{proof}

\section{Laplace's method}
\label{app:proof_annealing_limit}

In this section, we will prove Theorem~\ref{th:annealing_limit}. Before that, we state a multivariate version of the well-known Laplace's theorem (e.g. Section 6.4 of \cite{bender1999advanced}), which seems to be due to \cite{hsu1948theorem} (see also IX.5 of \cite{wong2001asymptotic}).

\paragraph{Notation.} 
\begin{itemize}
    \item In the metric space $(\mathbb{R}^d, \lVert \cdot \rVert)$, for some $\epsilon > 0$, $\mathcal{B}_\epsilon(\xib)$ denotes the ball of $\mathbb{R}^d$ centered on $\xib$ with radius $\epsilon$.
    \item Let $\epsilon$ be such that $0 < \epsilon < \min_{i,j}~\lVert \xib^*_j - \xib^*_i \rVert$. Such an $\epsilon$ exists because there are a finite number of global modes, so they are all isolated. In this case, 
    \begin{itemize}
        \item $K_\epsilon := \bigcup_{i=1}^I B_\epsilon(\xib^*_i)$,
        \item $M_\epsilon := \frac{1}{2} \max_{\xib \in \mathbb{R}^d \setminus K_\epsilon}~\ell(\xib)$,
        \item $K^{(M_\epsilon)}_\epsilon := K_\epsilon \bigcap \left\{ \xib \in \mathbb{R}^d : \ell(\xib) \ge M_\epsilon \right\}$.
    \end{itemize}
    \item Let $f$ be a continuous and bounded function. For $\omega > 0$, we define 
    \begin{itemize}
        \item for all $i \in [I]$, $C^f_i(\omega) := \int_{B_\epsilon(\xib^*_i)} f(\xib) e^{\ell(\xib)/\omega} d \xib$,
        \item $C^f_0(\omega) := \int_{\mathbb{R}^d \setminus K_\epsilon} f(\xib) e^{\ell(\xib)/\omega} d \xib = Z(\omega) - \sum_{i=1}^I C^f_i(\omega)$.
    \end{itemize}
    \item When $f \equiv 1$, we write $C^1_i(\omega) := C^{f}_i(\omega)$ and $C^1_0(\omega) := C^{f}_0(\omega)$.
\end{itemize}

\begin{theorem}[Laplace's theorem]
    Let $\xib^* \in \mathbb{R}^d$, $\delta > 0$ and $\mathcal{B}_{\xib^*}(\delta)
    $ be the ball centered around $\xib^*$ and with radius $\delta$. Let $\phi : \mathcal{B}_{\xib^*}(\delta) \rightarrow \mathbb{R}$ be a twice continuously differentiable function such that 
    \begin{itemize}
        \item $\phi(\xib^*) = 0$,
        \item $\phi(\xib) < 0$ for all $\xib \in \mathcal{B}_{\xib^*}(\delta)\backslash \{\xib^*\}$,
        \item $\det(- \nabla^2 \phi(\xib^*)) > 0$.
    \end{itemize}
    Then, for all bounded function $f$, for all $0 < \epsilon \le \delta$, we have
    \begin{equation*}
        \int_{\mathcal{B}_{\xib^*}(\epsilon)} f(\xib) \exp(\phi(\xib)/\omega) d\xib \underset{\omega \rightarrow 0}{\sim} (2\pi \omega)^{d/2} f(\xib^*) \det(- \nabla^2 \phi(\xib^*))^{-1/2}.
    \end{equation*}
    \label{th:laplace}
\end{theorem}

To prove Theorem~\ref{th:annealing_limit}, we recall that it is assumed that $\ell$ is twice continuously differentiable and all the global modes of $\ell$ are non-degenerate, i.e. the Hessian matrix of $\ell$ is negative definite at these modes. Without loss of generality, we can also consider that $\ell$ is non-positive. Therefore, the following assumptions are made throughout this section.
\begin{assumption}
The function $\ell$ satisfies the following conditions:
    \begin{itemize}
        \item the set $\{ \xib^*_1, \dots, \xib^*_I \} = \argmax_\xib~\ell(\xib)$ is finite and $\max_\xib~\ell(\xib) = 0$,
        \item $\int_{\mathbb{R}^d} \exp(\ell(\xib)/\omega) d\xib < \infty$ for all $\omega > 0$,
        \item $\ell$ is twice differentiable,
        \item for all the global modes $\{ \xib^*_1, \dots, \xib^*_I \}$ of $\ell$, we have
    \begin{equation*}
        \det(- \nabla^2 \ell(\xib^*_i)) > 0.
    \end{equation*}
    \end{itemize}
    \label{ass:determininant_positive}
\end{assumption}

We then need to prove the following lemmas.

\begin{lemma}
    For a continuous and bounded function $f$ on $\mathbb{R}^d$, we have
    \begin{equation*}
        \frac{C^f_0(\omega)}{Z(\omega)}  \xrightarrow[\omega \rightarrow 0]{} 0.
    \end{equation*}
    \label{lem:convergence_cf0_z_omega}
\end{lemma}

\begin{proof}
    For all $\xib \in K^{(M_\epsilon)}_\epsilon$, we have $M_\epsilon \le \ell(\xib) \le 0$. This means that 
    \begin{equation*}
        \begin{split}
            Z(\omega) &\ge \int_{K^{(M_\epsilon)}_\epsilon} e^{\ell(\xib)/\omega} d \xib \\
            &\ge e^{M_\epsilon/\omega} \mathbb{P}(K^{(M_\epsilon)}_\epsilon).
        \end{split}
    \end{equation*}
    Therefore, for a continuous and bounded function $f$, we have
    \begin{equation*}
        \begin{split}
            \left\lvert \frac{C^f_0(\omega)}{Z(\omega)} \right\rvert &\le \frac{1}{\mathbb{P}(K^{(M_\epsilon)}_\epsilon) } \int_{\mathbb{R}^d \setminus K_\epsilon} \left\lvert f(\xib) \right\rvert e^{\frac{1}{\omega} (\ell(\xib) - M_\epsilon)} d \xib \\
            &\le \frac{\max_{\xib \in \mathbb{R}^d \setminus K_\epsilon}~\left\lvert f(\xib) \right\rvert}{\mathbb{P}(K^{(M_\epsilon)}_\epsilon)} \int_{\mathbb{R}^d \setminus K_\epsilon} e^{\frac{1}{\omega} (\ell(\xib) - M_\epsilon)} d \xib.
        \end{split}
    \end{equation*}

    Any sequence $(\omega_t)_{t \ge 1}$ such that $\omega_t \xrightarrow[t \rightarrow \infty]{} 0$ is bounded, so there is a value $\omega_M$ such that $\omega_t \le \omega_M$ for all $t \ge 1$. Therefore, on $\mathbb{R}^d \setminus K_\epsilon$, we have $e^{\frac{1}{\omega_t} (\ell(\xib) - M_\epsilon)} \le e^{\ell(\xib)/\omega_M}$ for all $t \ge 1$, since $\ell(\xib) - M_\epsilon < 0$. Thus, we can apply the dominated convergence theorem and 
    \begin{equation*}
        \begin{split}
            \lim_{\omega \rightarrow 0} \left\lvert \frac{C^f_0(\omega)}{Z(\omega)} \right\rvert &\le \lim_{\omega \rightarrow 0} \frac{\max_{\xib \in \mathbb{R}^d \setminus K_\epsilon}~\left\lvert f(\xib) \right\rvert}{\mathbb{P}(K^{(M_\epsilon)}_\epsilon)}  \int_{\mathbb{R}^d \setminus K_\epsilon} e^{\frac{1}{\omega} (\ell(\xib) - M_\epsilon)} d \xib \\
            &= \frac{\max_{\xib \in \mathbb{R}^d \setminus K_\epsilon}~\left\lvert f(\xib) \right\rvert}{\mathbb{P}(K^{(M_\epsilon)}_\epsilon)} \int_{\mathbb{R}^d \setminus K_\epsilon} \lim_{\omega \rightarrow 0} e^{\frac{1}{\omega} (\ell(\xib) - M_\epsilon)} d \xib \\
            &= 0. 
        \end{split}
    \end{equation*}
\end{proof}

\begin{lemma}
    For a continuous and bounded function $f$ on $\mathbb{R}^d$, we have for all $i \in [I]$
    \begin{equation*}
        (2 \pi)^{-d/2} \omega^{-d/2} C^f_i(\omega) \xrightarrow[\omega \rightarrow 0]{}  f(\xib^*_i) \det(L_i)^{-1},
    \end{equation*}
    where $L_i$ is the matrix such that $L_i^2 = - \nabla^2 \ell(\xib^*_i)$.
    \label{lem:limit_cfi}
\end{lemma}

\begin{proof}
    This is a direct consequence of Laplace's theorem (Thm.~\ref{th:laplace}), with $\phi := \ell$.
\end{proof}

\begin{lemma}
    For a continuous and bounded function $f$ on $\mathbb{R}^d$, we have for all $i \in [I]$
    \begin{equation*}
        \frac{C^f_i(\omega)}{Z(\omega)} \xrightarrow[\omega \rightarrow 0]{} \tilde{c}_i f(\xib^*_i),
    \end{equation*}
    where $\tilde{c}_i = \frac{\det(L_i)^{-1}}{\sum_{i'=1}^I \det(L_{i'})^{-1}} f(\xib^*_i)$ and $L_i$ is the matrix such that $L_i^2 = - \nabla^2 \ell(\xib^*_i)$.
    \label{lem:ci_convergence}
\end{lemma}

\begin{proof}
    Remark that 
    \begin{equation*}
        \begin{split}
            Z(\omega) &= \int_{\mathbb{R}^d} e^{\ell(\xib)/\omega} d \xib \\
            &= \int_{\mathbb{R}^d \setminus K_\epsilon} e^{\ell(\xib)/\omega} d \xib + \sum_{i'=1}^I \int_{B_\epsilon(\xib^*_{i'})} e^{\ell(\xib)/\omega} d \xib \\
            &= C^1_0(\omega) + \sum_{i'=1}^I C^1_{i'}(\omega).
        \end{split}
    \end{equation*}
    
    Therefore, applying Lemma~\ref{lem:convergence_cf0_z_omega} and then Lemma~\ref{lem:limit_cfi}, this ensures that
    \begin{equation*}
        \begin{split}
            \lim_{\omega \rightarrow 0} \frac{C^f_i(\omega)}{Z(\omega)} &= \lim_{\omega \rightarrow 0} \frac{C^f_i(\omega)}{C^1_0(\omega) + \sum_{i'=1}^I C^1_{i'}(\omega)} \\
            &= \lim_{\omega \rightarrow 0} \frac{C^f_i(\omega)}{\sum_{i'=1}^I C^1_{i'}(\omega)} \\
            &= \frac{\det(L_i)^{-1}}{\sum_{i'=1}^I \det(L_{i'})^{-1}} f(\xib^*_i).
        \end{split}
    \end{equation*}
\end{proof}

Finally, thanks to Lemmas~\ref{lem:convergence_cf0_z_omega} and~\ref{lem:ci_convergence}, we can concisely prove Theorem~\ref{th:annealing_limit}.

\begin{proof}[Proof of Theorem~\ref{th:annealing_limit}]
    Using Lemmas~\ref{lem:convergence_cf0_z_omega} and~\ref{lem:ci_convergence}, we have for any continuous and bounded function $f$,
    \begin{equation*}
        \begin{split}
            \lim_{\omega \rightarrow 0} \int_{\mathbb{R}^d} f(\xib) g_\omega(d \xib) &= \lim_{\omega \rightarrow 0} \frac{\int_{\mathbb{R}^d \setminus K_\epsilon} f(\xib) e^{\ell(\xib)/\omega} d \xib + \sum_{i'=1}^I \int_{B_\epsilon(\xib^*_{i'})} f(\xib) e^{\ell(\xib)/\omega} d \xib}{Z(\omega)} \\
            &= \lim_{\omega \rightarrow 0} \left\{ \frac{C^f_0(\omega)}{Z(\omega)} + \sum_{i=1}^I \frac{C^f_{i}(\omega)}{Z(\omega)} \right\} \\
            &= 0 + \sum_{i=1}^I \tilde{c}_i f(\xib^*_i) \\
            &= \int_{\mathbb{R}^d} f(\xib) \left( \sum_{i=1}^I \tilde{c}_i \delta_{\xib^*_i}(d \xib) \right).
        \end{split}
    \end{equation*}
    This proves that $g_\omega$ converges weakly to $\sum_{i=1}^I \tilde{c}_i \delta_{\xib^*_i}$.
\end{proof}

\section{Interpretation of the weights of the annealed Gibbs measure}
\label{sub:interpretation_weights}

We claim that the weights $(\tilde{c}_i)_{i \in [I]}$ of the annealed Dirac mixture $g_0$ of Theorem~\ref{th:annealing_limit} provide valuable insight on the curvature of the function $\ell$ at its modes. Specifically, these weights offer a means to assess the ``flatness'' of the modes. To explore this, we introduce a measure of flatness and discuss the implications of the non-degeneracy assumption.

\paragraph{A measure of flatness based on curvature.} The Hessian matrix $\nabla^2 \ell(\xib)$ characterizes the local convexity (or concavity) of $\ell$ at a location $\xib$. To quantify the flatness of a mode, we consider the eigenvalues of $\nabla^2 \ell(\xib)$, which indicate the curvature of $\ell$ along the directions specified by the eigenvectors. At a maximum $\xib^*_i$, the eigenvalues $\{ \lambda_{i,1}, \dots, \lambda_{i,d} \}$ of $\nabla^2 \ell(\xib^*_i)$ are non-positive, and the sharpness of the mode is reflected by the magnitude of $\lvert\lambda_{i,1}\rvert, \dots, \lvert\lambda_{i,d}\rvert$. To characterize this sharpness, one can compute the geometric mean $\lvert\prod_{j=1}^d \lambda_{i,j} \rvert^{1/d}$. Conversely, the flatness of the mode can be measured by 
\begin{equation*}
    \zeta(\xib^*_i) := \left\lvert \prod_{j=1}^d  \lambda_{i,j} \right\rvert^{-1/d}.
\end{equation*} 

To illustrate the utility of $\zeta$ as a measure of flatness, consider a multivariate normal distribution $\ell(\cdot) := \mathcal{N}(~\cdot~; \mub, \Sigmab)$, with mode located at $\mub$. The Hessian at the mode is given by $\nabla^2_{\xib}\mathcal{N}(\mub; \mub, \Sigmab) = - \det(\Sigmab)^{-1/2} \Sigmab^{-1}$, leading to
\begin{equation*}
    \zeta(\mub) = \det(-\nabla^2_{\xib}\mathcal{N}(\mub; \mub, \Sigmab))^{-1/d} = (\det \Sigmab)^{\frac{d+2}{2d}}.
\end{equation*}
Thus, $\zeta(\mub)$ increases with the flatness of the Gaussian distribution, as larger values of $\det(\Sigmab)$ correspond to flatter distributions. 

\paragraph{Link with the annealed mixture weights.} For a general function $\ell$ with modes $\{ \xib^*_1, \dots, \xib^*_I \}$, Theorem~\ref{th:annealing_limit} establishes that the weight $\tilde{c}_i$ assigned to the mode $\xib^*_i$ by the annealed Dirac mixture is proportional to $\zeta(\xib^*_i)$. Indeed, 
\begin{equation*}
    \tilde{c}_i \propto \det(- \nabla^2 \ell(\xib^*_i))^{-1/2} = \left(\prod_{j=1}^d (- \lambda_j )\right)^{-1/2} = \zeta(\xib^*_i)^{d/2}.
\end{equation*}
Therefore, under the non-degeneracy assumption, the annealed Dirac mixture assigns greater weight to flatter maxima, as measured by $\zeta$. 

\paragraph{Degenerate modes.} The Hessian matrix $\nabla^2 \ell$ is negative semi-definite at maxima, meaning that the curvature of $\ell$ is non-positive in all directions. A zero eigenvalue indicates complete flatness along the corresponding eigenvector direction. Therefore, the number of zero eigenvalues determines the dimension of the subspace on which the mode is completely flat. This can be interpreted as the ``degree of flatness'' of the mode, with modes having more zero eigenvalues being considered flatter. The concept of degree of flatness helps clarify the interpretation of $\zeta$. In particular, $\zeta$ measures flatness only for non-degenerate modes (where the ``degree of flatness'' is 0). For degenerate modes, where $\zeta(\xib^*_i) = \infty$, $\zeta$ views them as ``infinitely'' flatter than non-degenerate ones. 

\paragraph{Beyond the non-degeneracy assumption.} While Theorem~\ref{th:annealing_limit} excludes degenerate cases, we can provide an intuitive explanation of the annealing process when degenerate modes are present. Since the weights of the Dirac mixture are related to the flatness index $\zeta$, non-degenerate modes will have comparable curvatures. In contrast, when degenerate modes are present, the weights of the Dirac mixture are predominantly assigned to the modes with the highest degree of flatness, as these are ``infinitely'' flatter than modes with a lower degree of flatness. The exact weights depend on the higher-order derivatives of $\ell$. Deriving these explicitly is beyond the scope of this paper, which focuses on the non-degenerate case that applies to most scenarios.

\section{Link with Bayesian inference}
\label{sec:bayesian_link}

The main concepts of NVA-M are the variational formulation of the initial optimization problem giving rise to Gibbs measures, their approximation in the space of mixtures of exponential family distributions using natural gradients, and annealing of the temperature parameter. In this section, we show that they admit a Bayesian interpretation. Then, we conclude that our optimization framework can be seen as a generalized Bayesian problem where the prior is flat. 

\subsection{Variational formulation and generalized posterior}

Bayesian approaches to statistical inference consist of updating prior beliefs about an event $A$ in light of an observable event $B$. Given a prior probability $\mathbb{P}(A)$, Bayesian inference seeks to determine the posterior probability $\mathbb{P}(A \mid B)$, which is governed by Bayes' rule:
\begin{equation*}
    \mathbb{P}(A \mid B) = \frac{\mathbb{P}(B \mid A) \mathbb{P}(A)}{\mathbb{P}(B)}.
\end{equation*}

In many cases, prior probabilities are represented by a distribution $p(\xib)$ over a parameter of interest $\xib \in E$. For a set of observables $X_1, \dots, X_n$, the posterior distribution of $\xib$ is
\begin{equation}
    q_p^*(\xib) = \frac{L(\xib; X_1, \dots, X_n) p(\xib)}{\int L(\xib; X_1, \dots, X_n) p(\xib) d\xib},
    \label{eq:regular_bayesian_update}
\end{equation}
where $L(\xib; X_1, \dots, X_n)$ denotes the likelihood, defined as the joint distribution of observations $X_1, \dots, X_n$ given the parameter $\xib$. It is well-established \citep{zellner1988optimal} that the posterior $q_p^*$ also satisfies
\begin{equation*}
    q_p^* = \argmin_{q \in \mathcal{P}(E)}~- \mathbb{E}_q[\log L(\xib; X_1, \dots, X_n)] + \KL(q \mid\mid p),
\end{equation*}
where $\mathcal{P}(E)$ represents the space of probability distributions on $E$, and $\KL(q \mid\mid p)$ is the Kullback--Leibler divergence between $q$ and $p$, defined as: 
\begin{equation*}
    \KL(q \mid\mid p) = \int q(\xib) \log\left( \frac{q(\xib)}{p(\xib)} \right) d \xib.
\end{equation*}

This formulation casts posterior inference as an optimization problem, where the objective is to balance the trade-off between two competing terms. The first expectation term represents the expected negative log-likelihood, measuring the discrepancy between the distribution $q$ and the observed data. The second term $\KL(q \mid\mid p)$ quantifies the information loss incurred by approximating the prior $p$ with the distribution $q$, therefore favoring distributions close to the prior. 

Remarkably, this optimization-based view of posterior inference can be generalized by replacing the negative log-likelihood with a loss function $\mathfrak{L}$, often referred to as a quasi-likelihood. Reformulating the minimization problem as a maximization problem yields: 
\begin{equation}
    q_p^{*} = \argmax_{q \in \mathcal{P}(E)}~\mathbb{E}_q[\ell(\xib)] - \KL(q \mid\mid p),
    \label{eq:generalized_bayesian}
\end{equation}
where $\ell(\xib) = - \mathfrak{L}(\xib; X_1, \dots, X_n)$. Recent studies have increasingly adopted this generalized approach to updating prior beliefs, as it offers greater flexibility compared to the traditional Bayesian framework \citep{bissiri2016general, fong2020marginal, rigon2023generalized, agnoletto2025bayesian}. In this context, the resulting $q_p^{*}$ is referred to as a \textit{generalized posterior}. 

\subsection{Gibbs posteriors and annealing}
\label{subsec:gibbs_posteriors}

Many Bayesian methods use a \textit{Gibbs posterior} defined by
\begin{equation}
    q_p^{*, \omega}(\xib) = \frac{\exp(\ell(\xib)/\omega) p(\xib)}{\int \exp(\ell(\xib)/\omega) p(\xib) d\xib},
    \label{eq:gibbs_posterior}
\end{equation}
where $\omega > 0$. Gibbs posteriors are also commonly referred to as the tempered posterior \citep{andersen2002bayesian, girolami2008bayesian}, fractional posterior \citep{ohagan1995fractional, gilks1995fractional} or power posterior \citep{friel2008marginal} and have been utilized in a wide range of statistical methodologies. For instance, they have been used to incorporate historical data in Bayesian data analyses~\citep{ibrahim2000power}, offering a robust, but flexible, mechanism for incorporating prior knowledge. Furthermore, in scenarios where the model is misspecified, i.e.~the true data-generating process lies outside of the support of the prior, it has been shown that Gibbs posteriors can outperform the traditional likelihood-based approach in terms of risk minimization~\citep{jiang2008gibbs}. 

The role of $\omega$ is clear when comparing~\eqref{eq:gibbs_posterior} with the generalized Bayesian update, derived from~\eqref{eq:regular_bayesian_update}. Specifically, $\omega$ modulates the influence of the information derived from the data relative to the prior distribution, effectively controlling how fast the model ``learns'' from the data. This relationship becomes even more apparent in the variational formulation of the Gibbs posterior, given by:
\begin{equation}
    q_p^{*, \omega} = \argmax_{q \in \mathcal{P}(E)}~\mathbb{E}_q[\ell(\xib)] - \omega \KL(q \mid\mid p).
    \label{eq:gibbs_variational_formulation}
\end{equation}

Here, if $\omega = 1$,~\eqref{eq:gibbs_posterior} reduces to the generalized Bayesian update~\eqref{eq:generalized_bayesian}. If $\omega > 1$, the loss of information with respect to the prior is given greater importance than the data. Conversely, when $0 < \omega < 1$, the data is given a stronger influence than in the usual Bayesian framework. 

The process of powering up the quasi-likelihood in~\eqref{eq:gibbs_posterior} draws a parallel with thermodynamic systems in statistical physics, where the probability of the state is given by a Gibbs measure. For this reason, the parameter $\omega$ is also commonly referred to as the temperature in the Bayesian literature, and the process of adjusting the temperature is also called annealing \citep{geman1984stochastic}. This analogy underscores the foundational role of $\omega$ in modulating the influence of prior information versus new data, akin to controlling the temperature in a physical system.

\subsection{Posterior approximation by variational inference}
\label{subsec:posterior_approximation}

In classical Bayesian inference, the prior distribution is often selected from a family of distributions that is conjugate to the likelihood function. In this case, the posterior distribution remains in the same family as the prior, enabling analytical computation of the posterior parameters. However, in modern Bayesian problems, this assumption of conjugacy is not always feasible. In many scenarios, the posterior distribution is intractable. To address this challenge, one common approach is to approximate the posterior by a distribution using a member of a parametric family of distributions $\mathcal{Q}$. For example, Laplace's approximation can be used to approximate the posterior with a Gaussian distribution, providing tractable solutions \citep{tierney1989fully}. 

An alternative approach uses the variational formulation presented by~\eqref{eq:gibbs_variational_formulation}. Consider a Gibbs posterior $q_p^{*, \omega}$ solving~\eqref{eq:gibbs_variational_formulation} for some value of $\omega > 0$. Since the optimization problem defined by~\eqref{eq:gibbs_variational_formulation} is infinite-dimensional, deriving an algorithm to find $q_p^{*, \omega}$ is not straightforward. To overcome this, one may approximate $q_p^{*, \omega}$ with a distribution $\tilde{q}_p^{*, \omega}$ from the parametric family of probability distributions $\mathcal{Q}= (q_\Lambdab)_{\Lambdab \in \Theta}$, solving the constrained optimization problem
\begin{equation*}
    \tilde{q}_p^{*, \omega} = \argmin_{q \in \mathcal{Q}}~\KL(q \mid\mid q_p^{*, \omega}).
\end{equation*}
It is well-established that the solution $\tilde{q}_p^{*, \omega}$ is also the solution of the variational problem defined by~\eqref{eq:gibbs_variational_formulation}, but constrained to the parametric family $\mathcal{Q}$:
\begin{equation*}
     \tilde{q}_p^{*, \omega} = \argmax_{q \in \mathcal{Q}}~\mathbb{E}_q[\ell(\xib)] - \omega \KL(q \mid\mid p).
\end{equation*}

Usually, the parametric family $\mathcal{Q}$ is chosen to simplify the optimization problem by reducing it from an infinite-dimensional problem to a finite-dimensional one involving the parameters $\Lambdab$ of the family. Consequently, an algorithm can be derived to find the optimal parameters
\begin{equation}
     \Lambdab^{*, \omega} = \argmax_{\Lambdab}~\mathbb{E}_{q_\Lambdab}[\ell(\xib)] - \omega \KL(q_\Lambdab \mid\mid p),
     \label{eq:variational_annealed_bayesian}
\end{equation}
where the resulting approximate posterior distribution is given by $\tilde{q}_p^{*, \omega} = q_{\Lambdab^{*, \omega}}$. This methodology defines variational inference and the chosen parametric family $\mathcal{Q}$ is referred to as the variational family.

\subsection{Link with the initial optimization problem}

Our optimization problem~\eqref{eq:simple_problem} is not a Bayesian problem in nature. However, our approach to solving it is inspired by the Bayesian principles discussed above. Now, we provide an intuitive exploration of the behavior of a Gibbs posterior as the temperature parameter $\omega$ approaches zero. 

From the expression of the Gibbs posterior $q_p^{*, \omega}$ given by~\eqref{eq:gibbs_posterior}, it can be observed that $q_p^{*, \omega}$ increasingly concentrates around the global modes of $\ell$ when the temperature $\omega$ decreases. Therefore, when $\omega$ tends to $0$, the Gibbs posterior becomes a very sharp distribution, resembling a mixture of Dirac measures centered on the global modes of $\ell$, provided all these modes lie within the support of the prior distribution $p$. Therefore, the modes of $\ell$ can be identified by locating the regions where $q_p^{*, \omega}$ remains non-negligible. 

Our objective is to exploit this phenomenon by approximating $q_p^{*, \omega}$ with a Gaussian mixture model. Provided the number of mixture components is at least equal to the number of global modes of $\ell$, then each component should fit under one of the modes. Thus, the modes of $\ell$ can be effectively recovered as the means of the mixture components. To help the convergence of the variational parameters during the optimization process, we propose the use of an annealing schedule for the temperature $\omega$. In particular, at each iteration $t$, we sequentially fit $q_p^{*, \omega_t}$, where the temperature $\omega_t$ gradually decreases to $0$.

Given the absence of a compelling reason to select a specific prior, we opt for a uniform prior by default. Under a uniform prior, the Kullback--Leibler divergence term simplifies to $\KL(q_\Lambdab \mid\mid p) = - \mathcal{H}(q_\Lambdab)$, where $\mathcal{H}(q_\Lambdab)$ denotes the entropy of $q_\Lambdab$. It is also noteworthy that with a uniform prior, the modes of $q_p^{*, \omega}$ remain invariant with respect to the value of $\omega > 0$, which is a property that may not hold for other prior distributions. 

\section{Natural gradient optimization and information geometry}
\label{subsec:natural_gradients}

We recall some properties of the natural gradient and its relation with information geometry, motivating its use. More details can be found in~\cite{khan2023bayesian}.

\paragraph{Limitation of the traditional gradient approach.} Let us first examine the (vanilla) gradient update rule to solve problem~\eqref{eq:problem_lambda_omega}:
\begin{equation*}
    \Lambdab_{t+1} = \Lambdab_t + \rho_t \nabla_\Lambdab \mathcal{L}_\omega(\Lambdab)|_{\Lambdab_t},
\end{equation*}
which is the update obtained when solving:
\begin{equation*}
    \Lambdab_{t+1} = \argmin_{\Lambdab}~\left< \nabla_\Lambdab \mathcal{L}_\omega(\Lambdab)|_{\Lambdab_t}, \Lambdab \right> + \frac{1}{2 \rho_t} \lVert \Lambdab - \Lambdab_t \rVert_2^2.
\end{equation*}
This makes clear that the update is driven by an Euclidean penalty in the parameter space. However, since the goal is to update $\Lambdab$ such that $q_{\Lambdab}$ converges to $g_\omega$, it is more appropriate to use a penalty based on a distributional metric rather than a Euclidean metric in the parameter space. The Euclidean distance, which depends on the parameterization, is not an ideal metric for evaluating the proximity between distributions.

\paragraph{Natural gradients and link to information geometry.} As discussed in \cite{khan2023bayesian}, using a $\KL$ divergence penalty leads to the update rule
\begin{equation*}
    \Lambdab_{t+1} = \argmin_{\Lambdab}~\left< \nabla_\Lambdab \mathcal{L}_\omega(\Lambdab)|_{\Lambdab_t}, \Lambdab \right> + \frac{1}{2 \rho_t} \KL(q_\Lambdab \mid\mid q_{\Lambdab_t}),
\end{equation*}
which results in~\eqref{eq:natural_gradient_update}. Typically, the natural gradient amounts to scaling the vanilla gradient by the inverse of the Fisher Information Matrix (FIM). However, for mixtures of exponential family distributions, the FIM associated with the density function $q_\Lambdab(\xib)$ may be singular. Therefore, it is more convenient to define the natural gradient using the joint distribution $q_\Lambdab(\xib, Z)$ of the MCEF~\eqref{eq:mcef_gaussian_mixture}. The natural gradient is then given by~\eqref{eq:natural_gradient_def}.

The natural gradient is narrowly linked to information geometry \citep{bonnabel2013stochastic}. The core idea behind information geometry is to use the shape of the manifold induced by the set of mixtures in the space of all probability distributions, which is independent of their parameterization \citep{amari2000methods, ollivier2017information, nielsen2020elementary}. The natural gradient points in the direction of steepest ascent on the Riemannian manifold with respect to the metric induced by the FIM, offering a faster convergence in variational inference \citep{sato2001online, honkela2007natural, hoffman2013stochastic, wu2024understanding}. This motivates the choice of natural gradient ascent for solving our optimization problem.

\paragraph{Inexact Riemannian gradient ascent.} Natural gradient ascent can be viewed as a form of inexact Riemannian gradient ascent, as it only uses the first-order approximation of the geodesic defined by the FIM, i.e. the direction of steepest ascent \citep{bonnabel2013stochastic}. Indeed, at each iteration $t$, the update $\Lambdab_{t+1}$ is achieved by taking a step $\rho_t$ forward from $\Lambdab_t$, following a straight line rather than the true curved geodesic. However, this approximation is only valid within a neighborhood with a small radius. Typically, if the step size $\rho_t$ is too large, the positive-definite constraint of the precision matrices can be violated. \cite{lin2020handling} suggested an improved learning rule (\textit{improved Bayesian Learning Rule}, iBLR) using a second-order approximation of the geodesic, which can be easily obtained in the case of exponential family distributions and their mixtures through a retraction map. The iBLR update ensures that the precision matrices remain positive-definite. In this paper, we have also derived the iBLR update rule to solve our optimization problem.

\section{Identities for gradient derivation}
\label{app:proof_optimization}

This section proves some identities used in Sections~\ref{sec:optimization} and~\ref{sec:mego} and Appendix~\ref{app:computation_natural_gradients_gaussian}. They mostly come from~\cite{lin2019fast}.

\begin{proposition}
    Let $q_\theta(\xib)$ be a probability density function parameterized by $\theta \in \Theta$. We have
\begin{equation*}
        \nabla_{\theta} \mathbb{E}_{q_\theta}[\log q_\theta(\xib)]|_{\theta=\theta_0} = \int \log q_{\theta_0}(\xib) \nabla_{\theta} q_\theta(\xib)|_{\theta=\theta_0} d \xib = \nabla_{\theta} \mathbb{E}_{q_\theta}[\log q_{\theta_0}(\xib)]|_{\theta=\theta_0}.
\end{equation*}
\label{prop:score_theorem}
\end{proposition}

\begin{proof}
    The result is obtained by developing
    \begin{equation*}
        \begin{split}
            \nabla_{\theta} \mathbb{E}_{q_\theta}[\log q_\theta(\xib)]|_{\theta=\theta_0} &= \int \nabla_{\theta} \left(  q_\theta(\xib) \log q_\theta(\xib) \right)|_{\theta=\theta_0} d \xib \\
            &= \int \left( \nabla_{\theta} q_\theta(\xib)|_{\theta=\theta_0} \right) \log q_{\theta_0}(\xib) d \xib + \mathbb{E}_{q_{\theta_0}}[\nabla_\theta \log q_\theta(\xib)|_{\theta=\theta_0}] \\
            &= \int \log q_{\theta_0}(\xib) \nabla_{\theta} q_\theta(\xib)|_{\theta=\theta_0} d \xib,
        \end{split}
    \end{equation*}
    where we have used the  Fisher identity that is $\mathbb{E}_{q_{\theta_0}}[\nabla_\theta \log q_{\theta}(\xib)|_{\theta=\theta_0}] = \int \nabla_\theta q_{\theta}(\xib)|_{\theta=\theta_0} d\xib = 0$.
\end{proof}

\begin{proposition}
    Let $\xib \in \mathbb{R}^d$. We have
    \begin{equation*}
        \begin{split}
            \nabla_\mub \mathcal{N}(\xib; \mub, \Sigmab) &= \Sigmab^{-1} (\xib - \mub) \mathcal{N}(\xib; \mub, \Sigmab), \\
            \nabla_\Sigmab \mathcal{N}(\xib; \mub, \Sigmab) &= \frac{1}{2} \left\{\Sigmab^{-1} (\xib - \mub) (\xib - \mub)^T \Sigmab^{-1} - \Sigmab^{-1} \right\} \mathcal{N}(\xib; \mub, \Sigmab).\\
        \end{split}
    \end{equation*}
    \label{prop:grad_hess_normal}
\end{proposition}

\begin{proof}
    Using that 
    \begin{align*}
        \nabla_\mub \mathcal{N}(\xib; \mub, \Sigmab) &= \nabla_\mub \log \mathcal{N}(\xib; \mub, \Sigmab) \; \mathcal{N}(\xib; \mub, \Sigmab),  \\
        \nabla_\Sigmab \mathcal{N}(\xib; \mub, \Sigmab) &= \nabla_\Sigmab \log \mathcal{N}(\xib; \mub, \Sigmab) \; \mathcal{N}(\xib; \mub, \Sigmab), 
        \end{align*}
    and that 
    \begin{equation*}
        \log \mathcal{N}(\xib; \mub, \Sigmab) = -\frac{d}{2} \log 2\pi - \frac{1}{2} \log  \det(\Sigmab) -\frac{1}{2} \lVert \xib - \mub \rVert_\Sigmab^2 , 
    \end{equation*}
    the first equality is obtained using that $\nabla_\mub \lVert \xib - \mub \rVert_\Sigmab^2 = 2 \Sigmab^{-1} (\xib - \mub)$, and 
     the second equality, using  that $\nabla_\Sigmab \log \det(\Sigmab) = \Sigmab^{-1}$ and that $\nabla_\Sigmab \lVert \xib - \mub \rVert_\Sigmab^2 = - \Sigmab^{-1} (\xib - \mub) (\xib - \mub)^T \Sigmab^{-1}$.
\end{proof}

\begin{proposition}
    Let $q_\Lambdab(\xib) = \sum_{k = 1}^K q_{\lambdab_k}(\xib)$ where for all $k \in [K]$, $q_{\lambdab_k}(\xib) := \mathcal{N}(\xib; \mub_k, \ESSb_{k}^{-1})$. For all $k \in [K]$, let $r_k(\xib) := \pi_k q_{\lambdab_k}(\xib)/q_{\Lambdab}(\xib)$ and, in addition, for all $\ell \in [K]$, $A_{\ell k}(\xib) := \ESSb_\ell (\mub_\ell - \xib) (\mub_k - \xib)^T \ESSb_k$. Then, we have
\begin{align*}
 \nabla_\xib \log q_\Lambdab(\xib) &=  \sum_{k=1}^K r_k(\xib) \ESSb_k (\mub_k - \xib),   \\
 \nabla^2_\xib \log q_\Lambdab(\xib)  &=   \sum_{k=1}^K r_k(\xib) \left[\Ab_{kk}(\xib) -  \ESSb_k -   \sum_{\ell=1}^K r_\ell(\xib) \Ab_{\ell k}(\xib) \right] \\
  &= -\left(\nabla_\xib \log q_\Lambdab(\xib)\right)\left(\nabla_\xib \log q_\Lambdab(\xib)\right)^T + \sum_{k=1}^K r_k(\xib) \left[\Ab_{kk}(\xib) -  \ESSb_k \right].
\end{align*}
    \label{prop:grad_hess_logmixture}
\end{proposition}

\begin{proof}
    Let us first write that, by Proposition~\ref{prop:grad_hess_normal},
    \begin{equation}
        \nabla_\xib q_{\lambdab_k}(\xib) = \nabla_\xib \mathcal{N}(\xib; \mub_k, \ESSb_k^{-1}) = \nabla_\xib \mathcal{N}(\mub_k; \xib, \ESSb_k^{-1}) = \ESSb_k (\mub_k - \xib) q_{\lambdab_k}(\xib).
        \label{eq:lem_responsibility}
    \end{equation}
    
    The first identity comes from~\eqref{eq:lem_responsibility}, as we have 
    \begin{equation*}
            \nabla_\xib \log q_\Lambdab(\xib) = \sum_{k = 1}^K \pi_k \frac{\nabla_\xib q_{\lambdab_k}(\xib)}{q_\Lambdab(\xib)} = \sum_{k=1}^K \frac{q_{\lambdab_k}(\xib)}{q_\Lambdab(\xib)} \ESSb_k (\mub_k - \xib).
    \end{equation*}

    To obtain the second identity, we can differentiate the first identity and use~\eqref{eq:lem_responsibility} again, so that
    \begin{equation*}
        \begin{split}
            &\nabla^2_\xib \log q_\Lambdab(\xib) \\
            &= \sum_{k=1}^K \pi_k \left\{\frac{\nabla_\xib (\ESSb_k (\mub_k - \xib) q_{\lambdab_k}(\xib))}{q_\Lambdab(\xib)} - \frac{\nabla_\xib q_\Lambdab(\xib) q_{\lambdab_k}(\xib) (\mub_k - \xib)^T \ESSb_k}{ q_\Lambdab(\xib)^2} \right\} \\
            &= \sum_{k=1}^K \pi_k \frac{q_{\lambdab_k}(\xib)}{q_\Lambdab(\xib)} \left\{  \frac{\nabla_\xib q_{\lambdab_k}(\xib) (\mub_k - \xib)^T \ESSb_k}{q_{\lambdab_k}(\xib)} - \ESSb_k - \sum_{\ell=1}^K \pi_\ell \frac{\nabla_\xib q_{\lambdab_\ell}(\xib) (\mub_k - \xib)^T \ESSb_k}{ q_\Lambdab(\xib)} \right\} \\
            &= \sum_{k=1}^K \pi_k \frac{q_{\lambdab_k}(\xib)}{q_\Lambdab(\xib)} \left\{ \ESSb_k (\mub_k - \xib) (\mub_k - \xib)^T \ESSb_k - \ESSb_k - \sum_{\ell = 1}^K \pi_\ell \frac{q_{\lambdab_\ell}(\xib)}{q_\Lambdab(\xib)} \ESSb_\ell (\mub_\ell - \xib) (\mub_k - \xib)^T \ESSb_k \right\}.
        \end{split}
    \end{equation*}
    This also implies the alternative form of the second identity, by remarking that 
    \begin{equation*}
        \left(\nabla_\xib \log q_\Lambdab(\xib)\right) \left(\nabla_\xib \log q_\Lambdab(\xib)\right)^T = \sum_{k=1}^K \sum_{\ell = 1}^K \pi_k \pi_\ell \frac{q_{\lambdab_k}(\xib)}{q_\Lambdab(\xib)} \frac{q_{\lambdab_\ell}(\xib)}{q_\Lambdab(\xib)} \ESSb_\ell (\mub_\ell - \xib) (\mub_k - \xib)^T \ESSb_k.
    \end{equation*}
\end{proof}

\section{Bonnet's and Price's theorems}
\label{app:bonnet_price}

In this section, we recall Bonnet's and Price's theorems~\citep{bonnet1964transformations, price1958useful, opper2009variational, rezende2014stochastic, lin2019fast}, formulated as consequences of Stein's lemma~\citep{stein1981estimation}. Trading generality for clarity, the following statements require stronger, but simpler, assumptions than in~\cite{lin2019stein}.

\begin{theorem}[Bonnet's theorem]
    Let $f : \mathbb{R}^d \rightarrow \mathbb{R}$ be a continuously differentiable function and $q(\xib) := \mathcal{N}(\xib; \mub, \Sigmab)$ a multivariate Gaussian distribution. We have
    \begin{equation*}
        \nabla_\mub \mathbb{E}_q[f(\xib)] = \mathbb{E}_q[\nabla_\xib f(\xib)].
    \end{equation*}
\end{theorem}

\begin{theorem}[Price's theorem]
    Let $f : \mathbb{R}^d \rightarrow \mathbb{R}$ be a twice continuously differentiable function and $q(\xib) := \mathcal{N}(\xib; \mub, \Sigmab)$ a multivariate Gaussian distribution. We have
    \begin{equation*}
        \nabla_\Sigmab \mathbb{E}_q[f(\xib)] = \mathbb{E}_q[\nabla_\xib^2 f(\xib)].
    \end{equation*}
\end{theorem}

\section{Computation of natural gradients}
\label{app:computation_natural_gradients_gaussian}

In this section, we derive the natural gradient update rules and the gradient estimators for the variational mixtures, including when they are Gaussian mixtures (Sect.~\ref{sec:optimization} and~\ref{sec:mego}).

\subsection{Natural gradient update rules for the mixture parameters}
\label{app:natural_gradient_update_rule_gaussian}

In the case of Gaussian mixtures, we have the natural parameters $\lambdab_k = (\lambdab_k^{(1)}, \lambdab_k^{(2)}) := (\ESSb_k \mub_k, -\ESSb_k/2)$ and the expectation parameters $\Mb_k := (\mb_k^{(1)}, \mb_k^{(2)}) = (\pi_k \mub_k, \pi_k (\ESSb_k^{-1} + \mub_k \mub_k^T))$ (see~\cite{lin2019fast} for details). Using the chain rule, it comes
\begin{align*}
\nabla_{\mb_k^{(1)}} \mathcal{L}_\omega(\Lambdab)  &=  \frac{1}{\pi_k}  \left( \nabla_{\mub_k} \mathcal{L}_\omega(\Lambdab) - 2 (\nabla_{\ESSb_k^{-1}} \mathcal{L}_\omega(\Lambdab)) \mub_k  \right),  \\
\nabla_{\mb_k^{(2)}}\mathcal{L}_\omega(\Lambdab) &=  \frac{1}{\pi_k}  \nabla_{\ESSb_k^{-1}} \mathcal{L}_\omega(\Lambdab).
\end{align*}
Using~\eqref{eq:duality_identity} and substituting these gradients into the natural gradient update rule~\eqref{eq:natural_gradient_update}, we derive update equations for the parameters $\mub_k$, $\ESSb_k$ and $\pi_k$:
\begin{align*}
    \ESSb_{k, t+1} &= \ESSb_{k, t} - \frac{2 \rho_t}{\pi_{k,t}} \nabla_{\ESSb_k^{-1}} \mathcal{L}_\omega(\Lambdab)|_{\Lambdab_t},  \\
    \mub_{k, t+1} &= \mub_{k, t} + \frac{\rho_t}{\pi_{k,t}} \ESSb_{k, t+1}^{-1} \nabla_{\mub_k} \mathcal{L}_\omega(\Lambdab)|_{\Lambdab_t},  \\
    v_{k,t+1} &= v_{k,t} + \rho_t \nabla_{\pi_k} \mathcal{L}_\omega(\Lambdab)|_{\Lambdab_t},
\end{align*}
where $v_{k,t} := \log(\pi_{k,t}/\pi_{K,t})$. 

\subsection{Expression for $\nabla_{\pi_k} \mathcal{L}_\omega(\Lambdab)$} 
\label{app:expression_grad_pi}

Proposition~\ref{prop:score_theorem} (Appendix~\ref{app:proof_optimization}) provides the identity
\begin{equation*}
\nabla_{\pi_k}\mathcal{L}_\omega(\Lambdab) = \int \left( \nabla_{\pi_k} q_\Lambdab(\xib) \right)  f_\omega(\xib; \Lambdab) d\xib.
\end{equation*}
The gradient with respect to $\pi_k$ depends only on $\nabla_{\pi_k} q_\Lambdab(\xib)$. Since $\pi_K = 1- \sum_{k=1}^{K-1} \pi_k$, we have
\begin{equation*}
\nabla_{\pi_k} q_\Lambdab(\xib)= q_{\lambdab_k}(\xib) - q_{\lambdab_K}(\xib) ,
\end{equation*}
and thus,
\begin{equation}
\begin{split}
\nabla_{\pi_k}\mathcal{L}_\omega(\Lambdab) &= \mathbb{E}_{q_{\lambdab_k}}[f_\omega(\xib;\Lambdab)] -  \mathbb{E}_{q_{\lambdab_K}}[f_\omega(\xib;\Lambdab)]\\
&= \mathbb{E}_{\mathcal{N}(\mub_k, \ESSb_k^{-1})}[f_\omega(\xib;\Lambdab)] -  \mathbb{E}_{\mathcal{N}(\mub_K, \ESSb_K^{-1})}[f_\omega(\xib;\Lambdab)].
\end{split}
\label{eq:grad_pi}
\end{equation}

\subsection{Expressions for $\nabla_{\ESSb_k^{-1}} \mathcal{L}_\omega(\Lambdab)$ and $\nabla_{\mub_k} \mathcal{L}_\omega(\Lambdab)$} 
\label{app:expression_grad_mu_s}

To obtain the gradients with respect to the means and the covariance matrices, we use Proposition~\ref{prop:score_theorem} again:
\begin{equation*}
    \begin{split}
        \nabla_{\ESSb_k^{-1}} \mathcal{L}_\omega(\Lambdab) &= \int \left( \nabla_{\ESSb_k^{-1}} q_\Lambdab(\xib) \right) f_\omega(\xib; \Lambdab) d \xib, \\
        \nabla_{\mub_k} \mathcal{L}_\omega(\Lambdab) &= \int \left( \nabla_{\mub_k} q_\Lambdab(\xib) \right) f_\omega(\xib; \Lambdab) d \xib.
    \end{split}
\end{equation*}
These expressions lead to a few equivalent identities, from which different estimators can be suggested.

\paragraph{Black-box method.} Proposition~\ref{prop:grad_hess_normal} provides expressions for the gradients derived from direct differentiation of $q_\Lambdab(\xib)$ under the integral with respect to $\ESSb_k^{-1}$ and $\mub_k$. The following expressions are obtained:
\begin{align}
        \nabla_{\ESSb_k^{-1}} \mathcal{L}_\omega(\Lambdab) &= \frac{\pi_k}{2} \mathbb{E}_{\mathcal{N}(\mub_k, \ESSb_k^{-1})}[(\ESSb_k (\xib - \mub_k) (\xib - \mub_k)^T \ESSb_k - \ESSb_k) f_\omega(\xib; \Lambdab)], \label{eq:grad_s_0} \\
        \nabla_{\mub_k} \mathcal{L}_\omega(\Lambdab) &= \pi_k \mathbb{E}_{\mathcal{N}(\mub_k, \ESSb_k^{-1})}[\ESSb_k (\xib - \mub_k) f_\omega(\xib; \Lambdab)] . \label{eq:grad_mu_0}
\end{align}
These expressions are consistent with those given in Section 2 of~\cite{wierstra2014natural} for the natural gradient with respect to a Gaussian distribution. 

\paragraph{Bonnet's and Price's theorems.}
Instead of differentiating $q_\Lambdab(\xib)$, Bonnet's and Price's theorems \citep{price1958useful, bonnet1964transformations, lin2019stein} can be used to find alternative expressions for these gradients. We recall these results in Appendix~\ref{app:bonnet_price}. This amounts to a reparameterization trick, using the properties of the Gaussian distribution to transform the gradients with respect to the parameters into gradients with respect to $\xib$, and then performing integration by parts: 
\begin{align}
        \nabla_{\ESSb_k^{-1}} \mathcal{L}_\omega(\Lambdab) &= \frac{\pi_k}{2} \mathbb{E}_{\mathcal{N}(\mub_k, \ESSb_k^{-1})}[\ESSb_k (\xib - \mub_k) \nabla_{\xib} f_\omega(\xib; \Lambdab)], \label{eq:grad_s_1} \\
        \nabla_{\mub_k} \mathcal{L}_\omega(\Lambdab) &= \pi_k \mathbb{E}_{\mathcal{N}(\mub_k, \ESSb_k^{-1})}[\nabla_{\xib} f_\omega(\xib; \Lambdab)]. \label{eq:grad_mu_1}
\end{align}
An alternative expression for $\nabla_{\ESSb_k^{-1}} \mathcal{L}_\omega(\Lambdab)$ can be obtained by integrating by parts once more:
\begin{equation}
        \nabla_{\ESSb_k^{-1}} \mathcal{L}_\omega(\Lambdab) = \frac{\pi_k}{2} \mathbb{E}_{\mathcal{N}(\mub_k, \ESSb_k^{-1})}[\nabla^2_{\xib} f_\omega(\xib; \Lambdab)]. \label{eq:grad_s_2}
\end{equation}
A more detailed derivation for these expressions is given in~\cite{lin2019fast}, Appendix B.2. 

\subsection{Estimation of the gradients}
\label{app:estimation_gradients}

The gradients expressed in the previous section can be estimated using Monte Carlo approximations from simulated samples. For some $t \ge 1$, $B \ge 1$, for all $k \in [K]$, let $\xib^{(k)}_1, \dots, \xib^{(k)}_B \overset{\text{i.i.d.}}{\sim} \mathcal{N}(\mub_{k,t}, \ESSb_{k,t}^{-1})$. 

\paragraph{Estimating $\nabla_{\pi_k} \mathcal{L}_\omega(\Lambdab)|_{\Lambdab_t}$.} We propose the following estimator for $\nabla_{\pi_k} \mathcal{L}_\omega(\Lambdab)|_{\Lambdab_t}$ defined by
\begin{equation}
    \widehat{\nabla_{\pi_k}\mathcal{L}}_\omega(\Lambdab)|_{\Lambdab_t} = \widehat{\gamma}^{(\pi_k)}_{\omega,\Lambdab_t,B} := \frac{1}{B} \sum_{b = 1}^B \left(f_\omega(\xib^{(k)}_b; \Lambdab_t) - f_\omega(\xib^{(K)}_b; \Lambdab_t)\right),
    \tag{$G_{\pi}$}
    \label{eq:est_grad_pi_0}
\end{equation}
which is unbiased and consistent from~\eqref{eq:grad_pi}.

\paragraph{Estimating $\nabla_{\mub_k} \mathcal{L}_\omega(\Lambdab)|_{\Lambdab_t}$.} We propose two estimators for $\nabla_{\mub_k} \mathcal{L}_\omega(\Lambdab)|_{\Lambdab_t}$ of the form $\widehat{\nabla_{\mub_k} \mathcal{L}}_\omega(\Lambdab)|_{\Lambdab_t} = \pi_{k,t} \widehat{\gammab}^{(\mub_k)}_{\omega,\Lambdab_t,B} := \pi_{k,t} \widehat{\gammab}^{(\mub_k, i)}_{\omega,\Lambdab_t,B}$, $i \in \{0, 1\}$,
where 
\begin{equation}
    \widehat{\gammab}^{(\mub_k, 0)}_{\omega,\Lambdab_t,B} := \frac{1}{B} \ESSb_{k,t} \sum_{b = 1}^B (\xib^{(k)}_b - \mub_{k,t}) f_\omega(\xib^{(k)}_b; \Lambdab_t),
    \tag{$G_{\mub}$-bb}
    \label{eq:est_grad_mu_0}
\end{equation}
and 
\begin{equation}
    \widehat{\gammab}^{(\mub_k, 1)}_{\omega,\Lambdab_t,B} := \frac{1}{B} \sum_{b = 1}^B \nabla_\xib f_\omega(\xib^{(k)}_b; \Lambdab_t),
    \tag{$G_{\mub}$-grad}
    \label{eq:est_grad_mu_1}
\end{equation}
lead to unbiased and consistent estimators from~\eqref{eq:grad_mu_0} and~\eqref{eq:grad_mu_1} respectively.

\paragraph{Estimating $\nabla_{\ESSb_k^{-1}} \mathcal{L}_\omega(\Lambdab)|_{\Lambdab_t}$.} We propose three estimators for $\nabla_{\ESSb_k^{-1}} \mathcal{L}_\omega(\Lambdab)|_{\Lambdab_t}$ of the form $\widehat{\nabla_{\ESSb_k^{-1}} \mathcal{L}}_\omega(\Lambdab)|_{\Lambdab_t} = \pi_{k,t} \widehat{\gammab}^{(\ESSb_k^{-1})}_{\omega,\Lambdab_t,B} /2 := \pi_{k,t} \widehat{\gammab}^{(\ESSb_k^{-1}, i)}_{\omega,\Lambdab_t,B} /2$, $i \in \{0, 1, 2\}$,
where
\begin{equation}
    \widehat{\gammab}^{(\ESSb_k^{-1}, 0)}_{\omega,\Lambdab_t,B} := \frac{1}{B} \ESSb_{k,t} \sum_{b = 1}^B \left( (\xib^{(k)}_b - \mub_{k,t})(\xib^{(k)}_b - \mub_{k,t})^T \ESSb_{k,t} - \Ib \right) f_\omega(\xib^{(k)}_b; \Lambdab_t),
    \tag{$G_{\Sigmab}$-bb}
    \label{eq:est_grad_s_0}
\end{equation}
\begin{equation}
    \widehat{\gammab}^{(\ESSb_k^{-1}, 1)}_{\omega,\Lambdab_t,B} := \frac{1}{B} \ESSb_{k,t} \sum_{b = 1}^B (\xib^{(k)}_b - \mub_{k,t}) \nabla_\xib f_\omega(\xib^{(k)}_b; \Lambdab_t),
    \tag{$G_{\Sigmab}$-grad}
    \label{eq:est_grad_s_1}
\end{equation}
and 
\begin{equation}
    \widehat{\gammab}^{(\ESSb_k^{-1}, 2)}_{\omega,\Lambdab_t,B} := \frac{1}{B} \sum_{b = 1}^B \nabla^2_\xib f_\omega(\xib^{(k)}_b; \Lambdab_t),
    \tag{$G_{\Sigmab}$-hess}
    \label{eq:est_grad_s_2}
\end{equation}
lead to unbiased and consistent estimators from~\eqref{eq:grad_s_0},~\eqref{eq:grad_s_1} and~\eqref{eq:grad_s_2} respectively.

\paragraph{Computability of $\nabla_\xib f_\omega$ and $\nabla^2_\xib f_\omega$.} From Proposition~\ref{prop:grad_hess_logmixture}, the quantities $\log q_\Lambdab(\xib)$, $\nabla_\xib \log q_\Lambdab(\xib)$ and $\nabla^2_\xib \log q_\Lambdab(\xib)$ can be exactly computed. Therefore, the computability of the derivative of $f_\omega$ only depends on the computability of the corresponding derivatives of $\ell$.

\section{Stochastic natural gradient ascent algorithm}
\label{app:snga}

This section gives the stochastic natural gradient ascent algorithm in the general mixture case (Algorithm~\ref{alg:snga}).

\LinesNumbered
\begin{algorithm}[]
\SetAlgoRefName{A} 
	\caption{Stochastic natural gradient ascent for a mixture of exponential family variational approximation (SNGA-M)}  \label{alg:snga}
 \textsc{Given} a function $\ell$ and a value for $\omega$. \\
		 \textsc{Set} $T$, $B$, $K$, $\Lambdab_0$, $(\rho_t)_{t \in [T]}$. \\
   \textsc{Compute} $(v_{k,0}) = (\log(\pi_{k,0}/\pi_{K,0}))_{k \in [K-1]}.$ \\
		 \For{$t=0\!:\!(T-1)$}{
				   \For {$k=1\!:\!K$}{
				    \textsc{Sample} $\xib^{(k)}_b \overset{\text{i.i.d.}}{\sim} q_{\lambdab_{k,t}},  \quad \text{for  $b=1\!:\!B$}. $ \\
      \textsc{Compute} $\widehat{\nabla_{\Mb_k} \mathcal{L}}_{\omega}(\Lambdab_t)$. \\
 \textsc{Update} $\lambdab_{k,t+1} =  \lambdab_{k,t}  + \rho_t  \widehat{\nabla_{\Mb_k} \mathcal{L}}_{\omega}(\Lambdab_t).$
 } 
\For {$k=1\!:\!(K\!-\!1)$}{
\textsc{Compute} $\widehat{\gamma}^{(\pi_k)}_{\omega,\Lambdab_t,B}.$ \\ 
\textsc{Update} $v_{k, t+1} =  v_{k, t}  + \rho_t  \widehat{\gamma}^{(\pi_k)}_{\omega,\Lambdab_t,B}.$ \\
}
}
\textsc{Compute} $(\pi_{k,T})_{k \in [K]}$ from $(v_{k,T})_{k \in [K-1]}$. \\
		\Return $\Lambdab_T.$ 
\end{algorithm}

\section{Behavior of a single Gaussian when annealing}
\label{app:asymptotic_behavior_single_gaussian}

In this section, we provide the proof of Proposition~\ref{prop:asymptotic_behavior_single_gaussian}, stated in Section~\ref{sub:asymptotic_behavior}.

\begin{proof}[Proof of Proposition~\ref{prop:asymptotic_behavior_single_gaussian}]
    For some $t$, the parameters of $q_{\thetab^{*, \omega_t}_1} = \mathcal{N}(\mub^{*, \omega_t}_1, \Sigmab^{*, \omega_t}_1)$ minimize $\KL(q_{\thetab^{*, \omega_t}_1} \mid\mid g_{\omega_t})$. They can be found by solving the equation system
    \begin{equation*}
    \left\{
    \begin{aligned}
        & \nabla_{\mub_1} \KL(q_{\thetab_1} \mid\mid g_\omega) = 0 \\
        & \nabla_{\Sigmab_1} \KL(q_{\thetab_1} \mid\mid g_\omega) = 0
    \end{aligned}
    \right. .
    \end{equation*}
    We also recall that
    \begin{equation*}
        \KL(q_{\thetab_1} \mid\mid g_\omega) = - \mathcal{H}(q_{\thetab_1}) - \mathbb{E}_{q_{\thetab_1}}\left[\frac{\ell(\xib)}{\omega} - \log Z(\omega)\right],
    \end{equation*}
    where $\mathcal{H}(q_{\thetab_1}) = d(1+\log(2\pi))/2 + \log(\det(\Sigmab))/2$.

    First, we prove that solutions $(q_{\thetab^{*,\omega_t}_1})_{t \rightarrow \infty}$ converge to a Dirac measure centered on $\xib^*$ as $t \rightarrow \infty$. Using Price's theorem (Appendix~\ref{app:bonnet_price}), we have
    \begin{equation*}
            \nabla_{\Sigmab_1} \KL(q_{\thetab_1} \mid\mid g_\omega) = - \frac{1}{2} \Sigmab_1^{-1} - \nabla_{\Sigmab_1}\mathbb{E}_{q_{\thetab_1}}\left[\frac{\ell(\xib)}{\omega}\right] = - \frac{1}{2} \Sigmab_1^{-1} - \frac{1}{2} \mathbb{E}_{q_{\thetab_1}}\left[\frac{\nabla^2_\xib \ell(\xib)}{\omega}\right].
    \end{equation*}
    Because $\ell$ is strictly concave, the eigenvalues of $-\nabla^2_\xib \ell(\xib)$ are strictly positive for all $\xib$. So for all $t \ge 1$, $\Sigmab^{*,\omega_t}_1$ satisfies
    \begin{equation}
            \Sigmab^{*,\omega_t}_1 = \omega_t \mathbb{E}_{\mathcal{N}(\mub^{*, \omega_t}_1, \Sigmab^{*, \omega_t}_1)}\left[- \nabla^2_\xib \ell(\xib)\right]^{-1}.
            \label{eq:convergence_sigma_one_gaussian}
    \end{equation}
    This implies that $\Sigmab^{*,\omega_t}_1 \xrightarrow[t \rightarrow \infty]{} 0$.
    
    Furthermore, using Bonnet's theorem (Appendix~\ref{app:bonnet_price}), we have
    \begin{equation*}
            \nabla_{\mub_1} \KL(q_{\thetab_1} \mid\mid g_\omega) = \nabla_{\mub} \mathbb{E}_{q_{\thetab_1}}\left[\frac{\ell(\xib)}{\omega}\right] = \mathbb{E}_{q_{\thetab_1}}\left[\frac{\nabla_\xib \ell(\xib)}{\omega}\right].
    \end{equation*}
    Therefore, $\mub^{*,\omega_t}_1$ must satisfy
    \begin{equation*}
            \mathbb{E}_{\mathcal{N}(\mub^{*,\omega_t}_1, \Sigmab^{*,\omega_t}_1)}\left[\nabla_\xib \ell(\xib)\right] = 0.
    \end{equation*}
    Since $\xib^*$ is the only critical point of $\ell$ and $\Sigmab^{*,\omega_t}_1 \xrightarrow[t \rightarrow \infty]{} 0$, then we have $\mub^{*, \omega_t}_1 \xrightarrow[t \rightarrow \infty]{} \xib^*$. We also deduce that the sequence $(q_{\thetab^{*,\omega_t}_1})_{t \ge 1}$ converges weakly to the Dirac measure centered on $\xib^*$ as $t \rightarrow \infty$.

    Finally, we use this weak convergence result and equation~\eqref{eq:convergence_sigma_one_gaussian} to conclude that
    \begin{equation*}
        \omega_t^{-1} \Sigmab^{*, \omega_t}_1 \xrightarrow[t \rightarrow \infty]{}  (-\nabla^2_\xib \ell(\xib^*))^{-1}.
    \end{equation*}
\end{proof}

\section{Behavior of Gaussian covariance matrices when annealing}
\label{app:entropy_approximation}

In this section, we will see how $\Sigmab^{*, \omega}_k \underset{\omega \rightarrow 0}{\sim} \omega (- \nabla^2_\xib \ell(\xib^*_k))^{-1}$ for all $k \in [K]$. We use the notations of Section~\ref{sub:asymptotic_behavior}. First, we prove a preliminary result.

\subsection{Component separation and entropy approximation lemma}

The following lemma shows that if the components of the mixture are separated, then the entropy of the mixture can be approximated by a sum of entropies associated with the components. To prove the validity of the approximation, we will pick a Gaussian mixture where components have fixed means and vanishing covariance matrices. We assume that the eigenvalues of a covariance matrix are decreasing to $0$ uniformly within itself but also with respect to all the other covariance matrices.

\begin{assumption}
    We assume that there exists 
    \begin{itemize}
        \item a positive function $\phi : h \mapsto \phi(h)$ with $\phi(h) \xrightarrow[h \rightarrow \infty]{} + \infty$,
        \item for $k \in [K]$, a matrix $\Eb_k$,
    \end{itemize}
    such that 
    \begin{equation*}
        \phi(h)\Sigmab_k(h) \xrightarrow[h \rightarrow \infty]{} \Eb_k.
    \end{equation*}

    Let $\alpha_{k,1}(h) \le \dots \le \alpha_{k,d}(h)$ denote the eigenvalues of $\Sigmab_k(h)^{1/2}$  in increasing order and analogously, $\beta_{k,1} \le \dots \le \beta_{k,d}$ the ordered eigenvalues of $\Eb_k^{1/2}$. A consequence is that for all $1 \le i \le d$, we have
    \begin{equation*}
        \phi(h) \alpha_{k,i}(h) \xrightarrow[h \rightarrow \infty]{} \beta_{k,i}.
    \end{equation*}
    \label{ass:component_separation}
\end{assumption}

\begin{lemma}
    Suppose that the mixture parameters $\Lambdab$ are constant except for the covariance matrices $\Sigmab_k(h)$, for $k \in [K]$, converging accordingly to Assumption~\ref{ass:component_separation}. Then
    \begin{equation*}
        \mathcal{H}(q_\Lambdab) \xrightarrow[h \rightarrow \infty]{} \tilde{\mathcal{H}}(q_\Lambdab) := - \sum_{k=1}^K \pi_k \log \pi_k + \sum_{k=1}^K \pi_k \mathcal{H}(q_{\lambdab_k}).
    \end{equation*}
    \label{lem:entropy_decomposition}
\end{lemma}

\begin{proof}
    Let $u_{k, k'} := \lVert \mub_k - \mub_{k'} \rVert_{\Sigmab_k}/(1 + \lVert \Sigmab_k^{-1/2} \Sigmab_{k'}^{1/2} \rVert_{\text{op}})$ where $\lVert \Sigmab \rVert_{\text{op}}$ is the operator norm of $\Sigmab$, given by its largest eigenvalue.

    Theorem~4.2 of~\cite{furuya2024theoretical} states that for $s \in (0,1)$, we have 
    \begin{equation*}
        \lvert \mathcal{H}(q_\Lambdab) - \tilde{\mathcal{H}}(q_\Lambdab) \rvert \le \min \left\{ \frac{K}{2}, \frac{2}{(1-s)^{d/4}} \sum_{k=1}^K \sum_{k' \neq k} \sqrt{\pi_k \pi_{k'}} \exp \left( \frac{-s u_{k,k'}}{4}\right) \right\}.
    \end{equation*}But we have $\lVert \Sigmab_k^{-1/2} \Sigmab_{k'}^{1/2} \rVert_{\text{op}} \le \alpha_{k',d}(h) \alpha_{k,1}(h)$ and $\lVert \mub_k - \mub_{k'} \rVert_{\Sigmab_k} \ge \lVert \mub_k - \mub_{k'} \rVert_{2} \alpha_{k,d}(h)^{-1}$, so that
    \begin{equation*}
        u_{k, k'} \ge \frac{\lVert \mub_k - \mub_{k'} \rVert_{2} }{\alpha_{k,1}(h) + \alpha_{k',d}(h) } \sim_{h \rightarrow \infty } \phi(h) \frac{\lVert \mub_k - \mub_{k'} \rVert_{2} }{\beta_{k,1} + \beta_{k',d} } \xrightarrow[h \rightarrow \infty]{} + \infty.
    \end{equation*}
    Therefore, we can conclude that $\lvert \mathcal{H}(q_\Lambdab) - \tilde{\mathcal{H}}(q_\Lambdab) \rvert \xrightarrow[h \rightarrow \infty]{} 0$.
\end{proof}

\subsection{Asymptotic behavior of Gaussian covariance matrices }

Now, we can explain how $\Sigmab^{*, \omega}_k \underset{\omega \rightarrow 0}{\sim} \omega (-\nabla^2_\xib \ell(\xib^*_k))^{-1}$ for all $k \in [K]$. If $\thetab^{*,\omega} \xrightarrow[]{} \thetab^{*, 0}$ as $\omega \rightarrow 0$, then for a sufficiently small $\omega$, $\thetab^{*,\omega}$ is a mixture of very sharp Gaussian components, as their covariance matrices vanish when $\omega \rightarrow 0$. Following Lemma~\ref{lem:entropy_decomposition}, this implies that below a certain threshold of $\omega$, the entropy of $q_{\thetab^{*,\omega}}$ can be approximately decomposed as 
\begin{equation*}
    \mathcal{H}(q_{\thetab^{*,\omega}}) \approx - \sum_{k=1}^K \pi^{*,\omega}_k \log \pi^{*,\omega}_k + \sum_{k=1}^K \pi^{*,\omega}_k \mathcal{H}(\thetab^{*, \omega}_k).
\end{equation*}
Therefore,
\begin{equation*}
    \mathbb{E}_{q_{\thetab^{*,\omega}}}[\ell(\xib)] + \omega \mathcal{H}(q_{\thetab^{*, \omega}}) \approx \sum_{k=1}^K \pi^{*,\omega}_k \left(\mathbb{E}_{q_{\thetab^{*,\omega}_k}}[\ell(\xib)] - \omega \mathcal{H}(q_{\thetab^{*, \omega}_k}) \right).
\end{equation*}
This indicates that, for sufficiently small $\omega$, problem~\eqref{eq:variational_theta} can be nearly decomposed into $K$ separated problems
\begin{equation*}
    \thetab^{*,\omega}_k \in \argmax_{\thetab_k}~\mathbb{E}_{q_{\thetab_k}}[\ell(\xib)] + \omega \mathcal{H}(q_{\thetab_k}).
\end{equation*} 

Given that $\thetab^{*, \omega} \xrightarrow[]{} \thetab^{*, 0}$ as $\omega \rightarrow 0$, this suggests that in the mixture approximation problem~\eqref{eq:variational_theta}, if the components are well-separated and close to the modes of $\ell$, then the behavior of each component will resemble that in the single-mode case as described by Proposition~\ref{prop:asymptotic_behavior_single_gaussian}, hence, for all $k \in [K]$,
\begin{equation*}
    \Sigmab^{*, \omega}_k \underset{\omega \rightarrow 0}{\sim} \omega (-\nabla^2_\xib \ell(\xib^*_k))^{-1}.
\end{equation*}
While this conclusion is based on theoretical considerations, its practical utility will be demonstrated for the choice of algorithm hyperparameters, more specifically the learning rate (see Appendix~\ref{sub:learning_rate}).

\section{Interpretations of the NVA-GM algorithm}

NVA-GM implements a natural gradient method to solve a variational problem, with an annealed objective function. The core idea presented was that in fitting a Gaussian mixture, each Gaussian component is likely to adjust to a location where the objective function is high to favor a high expectation in \eqref{eq:problem_lambda_omega}. If there are as many modes of $\ell$ as components in the Gaussian mixture ($K = I$), then we hope that each Gaussian component takes responsibility for one mode of $\ell$. Alternatively, we can look at this algorithm from other perspectives.

\subsection{Particle interpretation}
\label{sub:particle_interpretation}

\cite{khan2023bayesian} proved that the \textit{Bayesian Learning Rule}, consisting in solving the approximate Bayesian problem~\eqref{eq:variational_annealed_bayesian} with $\omega=1$, using a variational family from the exponential family and natural gradients, can lead to classic optimization algorithms such as stochastic gradient algorithms and Newton--Raphson's method. These algorithms iteratively update the position of a particle until convergence to an optimum. For example, in the case of Stochastic Gradient Ascent (SGA), the updates are of the form
\begin{equation*}
    \xib_{t + 1} = \xib_t + \rho_t \nabla \ell(\xib),
\end{equation*}
where $\ell(\xib)$ is the function to be maximized, and $\xib_t$ is expected to converge to a local or global maximum of $\ell$. Conceptually, the particle is subject to the potential $\ell$, generating a driving force $\nabla \ell$, directed towards the sinks of the potential, i.e.~the maxima of $\ell$.

When the variational family in~\eqref{eq:variational_annealed_bayesian} is restricted to Gaussian distributions with covariance fixed to identity, the optimization is only performed on the mean $\mub$ of the Gaussian. In~\cite{khan2023bayesian}, it is shown that a delta-method approximation of the natural gradient leads to the update rule
\begin{equation*}
    \mub_{t+1} = \mub_t + \rho_t \nabla \ell(\mub_t).
\end{equation*}
Thus, the trajectory of the mean $(\mub_t)_{t \ge 1}$ of the Gaussian distribution is equivalent to the trajectory of the particle $(\xib_t)_{t \ge 1}$ in the SGA algorithm.

This analogy can be extended to Gaussian mixtures by considering the component means as individual particles, each moving according to forces generated by potentials. From this perspective, a Gaussian mixture with $K$ components can be viewed as a system of $K$ particles with locations $(\mub_{k,t})_{k \in [K]}$ at time $t$, with the goal that they converge to the sinks of the potential $\ell$. In our situation, the function to optimize at each iteration $t$ is $\mathcal{L}_{\omega_t}(\Lambdab) = \mathbb{E}_{q_\Lambdab}[f_{\omega_t}(\xib; \Lambdab)]$, which can be interpreted as an average of a general potential $f_{\omega_t}(\xib; \Lambdab)$ (where $\xib \sim q_\Lambdab$). This general potential can be decomposed into two components:
\begin{equation}
    f_{\omega_t}(\xib; \Lambdab) = \underbrace{\ell(\xib)}_{\text{driving potential}} + \quad \omega_t \times \underbrace{(-\log q_\Lambdab(\xib))}_{\text{repulsive potential}}.
    \label{eq:potential_decomposition}
\end{equation}

To clarify the role of these two potentials, recall that the gradient of $\mathcal{L}_{\omega_t}(\Lambdab)$ with respect to $\mub_k$ is given by equation~\eqref{eq:grad_mu_0}, which shows that the gradient is the weighted average of vectors $(\xib - \mub_k)$ where $\xib \sim \mathcal{N}(\mub_k, \ESSb_k)$ is sampled symmetrically around $\mub_k$, with weights determined by the potential $f_{\omega_t}(\xib; \Lambdab)$. The decomposition of the potential makes it clear that $\ell(\xib)$ contributes positively to the weight, whereas $\log q_\Lambdab(\xib)$ contributes negatively. Consequently, the driving potential directs the gradient towards regions where $\ell$ is high, whereas the repulsive potential pushes it away from regions where $q_\Lambdab$ is high, i.e.~away from the other component means of the mixture. Therefore, the covariance $\ESSb_k^{-1}$ of a component determines the strength of this repulsive force. Whereas the individual covariance affects the repulsion force of the particles individually, $\omega_t$ can be seen as a parameter setting the strength of the repulsive potential of all particles. This is analogous to many setups in statistical physics, where temperature influences the interparticle forces \citep{israelachvili1992intermolecular}.

If the aim of the algorithm was only to let the particles converge to the sinks of $\ell$, the driving potential alone would achieve this. However, the inclusion of the repulsive potential, induced by the entropy term, forces the particles to spread out, which is crucial when we want them to converge to different sinks. A discussion on how one should balance the potentials over time through the specification of an annealing schedule is given in Section~\ref{sub:annealing_schedule}.

\subsection{Evolutionary interpretation}
\label{sub:evolutionary_interpretation}

Evolutionary algorithms represent a class of black-box optimization algorithms. To optimize a function, called fitness function, with respect to certain parameters, these algorithms generate a population of individuals in the search space. Each individual is assigned a fitness value, typically the function's value at the parameters represented by that individual. At each iteration, a new generation of individuals is produced through a process that simulates mechanisms of natural evolution, specifically mutation, and selection. The mutation step involves randomly perturbing the current population to explore the parameter space, while the selection step uses the fitness values of the current population to generate new individuals with higher fitness, therefore focusing on already explored regions of the parameter space.

A subset of evolutionary algorithms, those within the Information-Geometric Optimization framework (IGO, \cite{ollivier2017information}), maintains a connection with gradient-based optimization methods. Notably, the Natural Evolution Strategies (NES, \cite{wierstra2014natural}) use a parameterized search distribution to generate populations. The parameters of this search distribution are updated based on an estimate of the search natural gradient, i.e.~the natural gradient of the fitness function with respect to the search distribution. At each iteration, a new population is sampled under the current search distribution (mutation step), and these samples are used to estimate the search natural gradient, directing the search towards regions of the parameter space with higher expected fitness (selection step).

Our NVA-GM algorithm exhibits analogous principles. At each iteration, it samples locations from each Gaussian component $\mathcal{N}(\mub_k, \ESSb_k^{-1})$ of the mixture. The covariance matrix $\ESSb_k^{-1}$ can be interpreted as the mutation strength for the corresponding search distribution, defining the dispersion of the samples. These samples are used to update the components from which they were drawn, by computing natural gradients. Thus, each component acts as a search distribution. Thus, the mixture induces $K$ search distributions, enabling the identification of up to $K$ distinct solutions. However, these search distributions are not updated independently, but they are interdependent due to the entropy term in our objective function, which involves the whole mixture. As seen earlier, the entropy term promotes spatial diversity between the components. Consequently, rather than acting completely independently, the different search distributions collaborate to find distinct solutions, with the temperature $\omega_t$ setting the strength of the collaboration. In the limiting case where $\omega_t = 0$, the search distributions are updated independently.

\subsection{Interpretation of the annealing schedule}
\label{sub:annealing_schedule_interpretation}

In this section, we analyze how the value of $\omega$ influences the behavior of the system of particles formed by the component means, particularly in terms of exploration versus exploitation. We identify and discuss two distinct regimes: when $\omega$ is very small and when $\omega$ is very large. We will then provide guidance on how to navigate between these two regimes to ensure that the particles behave well to solve our optimization problem. 

\paragraph{Full-drive regime.}

We argue that when $\omega$ is too small in SNGA (Algorithm~\ref{alg:snga}), several particles are likely to converge to the same mode. To understand this, consider a scenario where $\omega \approx 0$ is sufficiently small so that the objective function can be approximated by this decomposition:
\begin{equation*}
    \mathcal{L}_\omega(\Lambdab) \approx \sum_{k=1}^K \pi_k \mathbb{E}_{q_{\lambdab_k}}[\ell(\xib)].
\end{equation*}
In problem~\eqref{eq:problem_lambda_omega}, the entropy penalty, by introducing a repulsive force between the particles, prevents the optimization problem from being separable into $K$ independent optimization sub-problems for each particle. However, when the entropy term is negligible, we can separate it: for $k \in [K]$, 
\begin{equation*}
    \lambdab_k = \argmax_\lambdab~\mathbb{E}_{q_\lambdab}[\ell(\xib)].
\end{equation*}

Thus, setting a very small value for $\omega$ amounts to running $K$ parallel optimization algorithms. The result depends heavily on the initial parameters, like in most gradient algorithms, and the mixture components converge independently. In particular, the update for the component means (the positions of the particles) is identical for all $k \in [K]$, i.e.~
\begin{equation*}
    \mub_{k,t+1} = \mub_{k,t} + \rho_t \tilde{\nabla}_\mub \mathbb{E}_{\mathcal{N}(\mub, \ESSb^{-1})}[\ell(\xib)]|_{(\mub, \ESSb) = (\mub_{k,t}, \ESSb_{k,t})}.
\end{equation*}

Because all the particles follow the same dynamic, they are likely to converge to the same mode if their initial positions are close. Similar to classic gradient ascent, the landscape of $\ell$ can be partitioned into basins of attraction associated with the potential sinks (the modes), where a particle initialized within a basin is driven to the corresponding sink. To ensure that particles converge to different modes, they must be initialized in different basins of attraction. However, even if we knew the number of modes in $\ell$ and their approximate locations, the exact boundaries of the basins are unknown, especially when there is no prior information on $\ell$ available. Setting the initial positions in well-spread locations of space might mitigate the problem in low-dimensional spaces, where there is a higher chance that they fall in different basins, but this becomes almost impossible in high-dimensional settings (for example $d \ge 3$), where the shape of these basins are more complex. This shows why using Algorithm~\ref{alg:snga} with a small $\omega$ is not suitable to find distinct solutions.

\paragraph{Equilibrium regime.}

When $\omega$ is non-negligible, the entropy term in~\eqref{eq:problem_lambda_omega} plays a critical role in preserving component separation. Entropy is maximized when a distribution is uniform. For a Gaussian mixture, this implies that its entropy is higher when the mixture components are spread out, meaning the means are distant from each other and the variances are large. Recall from equation~\eqref{eq:potential_decomposition} that the potential governing the dynamics of the particles can be decomposed into two parts: a driving potential and a repulsive potential. 

The gradient pulling the component means towards the modes is fully determined by the driving potential. Because the repulsive potential between the particles counteracts this driving force, a non-negligible $\omega$ prevents them from converging to the sinks. Instead, they converge to an equilibrium state where the driving force is compensated by the repulsive force, i.e.~a set of parameters $\Lambdab$ where $f_\omega(\xib; \Lambdab)$ is constant. Therefore, with a large $\omega$, convergence of the particles is inconsistent, as the final locations of the means are not at the modes of $\ell$, but rather at an equilibrium between the two potentials. Therefore, using Algorithm~\ref{alg:snga} with a large $\omega$ does not allow us to find the modes of $\ell$. However, this repulsive potential can be used to force spatial separation between the particles, allowing them to escape one basin of attraction that is already occupied and move to another. 

\paragraph{Exploration versus exploitation.} Ultimately, we want the component means to converge to the modes of $\ell$, and only canceling the repulsive force ($\omega \approx 0$) would achieve this. However, this also makes it likely that multiple means converge to the same mode. On the other hand, the repulsive force keeps particles spatially separated and encourages them to explore other regions of the search space, but it also prevents them from reaching the modes. The idea behind the annealing schedule in NVA-GM is to initially force the particles to spread out and explore the search space by setting a strong repulsive force (large $\omega$), then gradually transition into a regime where the particles are fully subject to the driving force by decreasing $\omega$ over time.

To some extent, we can say that the annealing schedule is similar to an exploration-exploitation strategy, a classic analogy in applications of annealing that has already been used in optimization \citep{huang2018improving, dangelo2021annealed}. Ideally, during the exploration phase, one should prioritize spatial separation of the components and exploration of the landscape of $\ell$. Then, $\omega$ decreases, and the driving potential should progressively dominate the overall dynamics, until reaching a pure exploitation phase, where the particles are only guided by the landscape of $\ell$. The transition phase allows the particles to be attracted to regions with high values of $\ell$ while still preserving some spatial separation. If the transition is slow enough, the particles are more likely to discover different basins of attraction, leading to convergence to distinct modes.

\paragraph{Setting an annealing schedule.} Originally, $\omega$ is used to set the target $g_\omega$ of the optimization problem, like in Section~\ref{sec:optimization}. To implement annealing in Algorithm~\ref{alg:snga_annealing}, a piecewise-constant decreasing schedule can be defined. First, set $\omega$ to a large value $\omega_1$ and let the algorithm run until $\Lambdab^{*, \omega_1}$ is reached. Then, reduce $\omega$ to a slightly smaller value $\omega_2 \le \omega_1$, find $\Lambdab^{*, \omega_2}$ and repeat until the estimation stabilizes.
This approach essentially restarts the algorithm after convergence with a fixed $\omega$, using the previous result as the initial condition for the next run.

However, this approach has several drawbacks. First, a criterion for convergence must be established before restarting with a new $\omega$. There is no general guideline for choosing a relevant criterion. More importantly, because it is preferable to decrease $\omega$ slowly to ensure convergence to distinct modes, the algorithm needs to be restarted many times, which is computationally costly.

Given that setting $\omega$ can be interpreted as balancing exploration and exploitation behaviors of a mode-finding algorithm, then from this perspective, restarts are not necessary: there is no reason to wait for convergence to adjust the value of $\omega$. Thus, it makes sense to consider any sequence $(\omega_t)_{t \ge 1}$ such that $\omega_t \xrightarrow[]{} 0$ as $t \rightarrow 0$ as a valid annealing schedule. The challenge lies in finding the optimal annealing schedule. To prevent component means from converging to the same modes, the annealing schedule should begin with large values of $\omega$ and not decrease too quickly to $0$. In practice, we aim to use the smallest starting value possible and decrease it as quickly as possible, all that while ensuring convergence to distinct modes, therefore saving computation time. An example of annealing schedule is given in Appendix~\ref{sub:hyper_annealing}.

\section{Hyperparameter tuning for NVA-GM}
\label{app:hyperparameters}

The NVA-GM algorithm (Algorithm~\ref{alg:snga_annealing_utility}) is governed by several hyperparameters, including the maximum number of iterations $T$, the mini-batch size $B$, the number of components of the mixture $K$, the initial parameters of the mixture $\thetab_0$, the annealing schedule $(\omega_t)_{t \ge 1}$ and the learning rate $(\rho_t)_{t \ge 1}$. 

The hyperparameter $T$ can be set similarly to other classic optimization problems: increasing $T$ allows the algorithm to perform more iterations and move closer to the solution but at a computational cost $O(T)$. Additionally, one may incorporate a stopping criterion to reduce this cost, such as stopping the algorithm when $\sum_{k=1}^K \lVert \mub_{k,t} - \mub_{k,t-1} \rVert_2 \le \epsilon$ for some value of $\epsilon$, which introduces an additional hyperparameter. The focus of this section is to discuss the choice and the impact of the remaining hyperparameters, with relevant insights and examples provided.

\subsection{Number of components}

The number of components $K$ represents the maximum number of solutions (modes) that can be identified by our algorithms. If $K$ is less than the number of (global) modes ($K < I$), then a single run cannot capture all the global modes. Conversely, if $K > I$ and all the global modes are identified with $K$ components, then the remaining components will either converge to local modes or some of the global modes already found. Indeed, although the ideal solution $g_0$ is a Dirac mixture concentrated on the global modes of $\ell$, the smoothed version $g_\omega$ retains both global and local modes of $\ell$, even though the local modes are less pronounced and therefore harder to detect. 

If all the global modes are already captured by a component, then the remaining components can potentially converge to local modes due to the entropy term, which promotes separation between components. However, if no local modes are available, these remaining components will converge to already occupied (global or local) modes. However, more importantly, the weights $(\tilde{c}_i)_{i \in [I]}$ assigned to the modes, specified in $g_0$, remain unchanged: the combined weight of components converging to a global mode $\xib^*_i$ converges to $\tilde{c}_i$, while components converging to local modes will see their weights vanish.

In some applications, such as Bayesian inverse problems \citep{stuart2010inverse} or optimal experimental design \citep{long2022multimodal}, finding local modes in addition to the global modes may be of particular interest. Nevertheless, global and local modes can be immediately distinguished by examining the component weights or evaluating $\ell$. Thus, even if local modes are not the primary focus, identifying them is not problematic since they can be easily recognized and disregarded. The main downside of choosing a larger $K$ is the additional computational cost associated with processing the extra components, which is not useful when the goal is only to find the $I$ global modes. 

In general, without prior knowledge of the exact number of modes $I$, it is generally advisable to slightly overestimate $K$ to ensure that $K \ge I$, provided that this remains computationally reasonable. Although the number of parameters in the mixture scales linearly with $K$, the complexity of the algorithm is not strictly linear in $K$. The computation cost would be exactly $O(K)$ if $\omega = 0$, as this case allows for independent parameter updates across components. However, when $\omega > 0$, the evaluation of the gradient depends on the entropy term which involves the parameters of all components, therefore increasing the cost. Thus, the complexity of the algorithm becomes $O(K\varphi(K))$, where $\varphi(K)$ reflects the gradient evaluation cost, depending on the estimators used. 

Specifically, for the NVA-GM algorithm (Algorithm~\ref{alg:snga_annealing}), 
\begin{itemize}
    \item in $\widehat{\gammab}^{(\pi_k, 0)}_{\omega,\Lambdab,B}$, $\widehat{\gammab}^{(\mub_k, 0)}_{\omega,\Lambdab,B}$ and $\widehat{\gammab}^{(\ESSb_k^{-1}, 0)}_{\omega,\Lambdab,B}$, $\varphi(K)$ is the cost of computing $\log q_\Lambdab(\xib)$ in $f_\omega(\xib; \Lambdab)$, which is $O(K)$,
    \item in $\widehat{\gammab}^{(\mub_k, 1)}_{\omega,\Lambdab,B}$ and $\widehat{\gammab}^{(\ESSb_k^{-1}, 1)}_{\omega,\Lambdab,B}$, $\varphi(K)$ is the cost of computing $\nabla_\xib \log q_\Lambdab(\xib)$ in $\nabla_\xib f_\omega(\xib; \Lambdab)$, which is also $O(K)$,
    \item in $\widehat{\gammab}^{(\ESSb_k^{-1}, 2)}_{\omega,\Lambdab,B}$, $\varphi(K)$ is the cost of computing $\nabla^2_\xib \log q_\Lambdab(\xib)$ in $\nabla^2_\xib f_\omega(\xib; \Lambdab)$, which is $O(K^2)$.
\end{itemize}

\subsection{Mini-batch size}

The mini-batch size $B$ is the number of samples used to compute the gradients. Larger values of $B$ lead to more accurate gradient estimates, though at the cost of increased computation time. Indeed, the variance of the gradient estimators decreases asymptotically with $B^{-1}$, while the computational complexity scales as $O(B)$. Alternatively, from an evolutionary perspective, $B$ can be viewed as the size of the population sampled by each search distribution. 

In Algorithm~\ref{alg:snga_annealing_utility}, the samples must also be sorted before computing the gradients with utility values, which results in a complexity of $O(B \log B)$ (assuming a merge-sort algorithm). However, in practice, the effective time required for sorting is generally negligible compared to gradient computation, unless $B$ is very large. 

It may be relevant to use a dynamic mini-batch size, varying over time according to a schedule similar to those used for $\omega$ and $\rho$. For instance, a smaller mini-batch at the beginning of the run sizes can improve exploration thanks to the inherent stochasticity in the Monte Carlo gradient approximations. Although our algorithms can be easily modified to accommodate a dynamic mini-batch size, for the sake of clarity, in the paper we have assumed a constant mini-batch size.

\subsection{Initial mixture parameters}

The initial parameters $\thetab_0$ consist of the initial mixture weights $(\pi_{1,0}, \dots, \pi_{K,0})$, the initial means $(\mub_{1,0}, \dots, \mub_{K,0})$, and the initial covariance matrices $(\ESSb_{1,0}^{-1}, \dots, \ESSb_{K,0}^{-1})$. 

\paragraph{Initial mixture weights.} The weights of the search mixture do not influence the parameter updates. As shown in equations~\eqref{eq:natural_gradient_update_pi} and~\eqref{eq:grad_pi}, the weight increments are independent of $\pi_{k,t}$. Similarly, although $\pi_{k,t}$ appears in the update equations for the mean locations~\eqref{eq:natural_gradient_update_mu} and the covariance matrices~\eqref{eq:natural_gradient_update_s}, it cancels out with the gradients as shown in equations~\eqref{eq:grad_s_0} to~\eqref{eq:grad_s_2}. Therefore, the behavior of the algorithms is not fundamentally affected by the choice of the initial weights.

Any positive weight vector summing to $1$ should suffice. It takes fewer iterations for the weights to converge if their initial value is close to their limit. However, in general, we cannot predict which component will converge to which mode and we have no prior knowledge about the limit of the weights. Thus, it is reasonable, although not critical, to initialize the weight uniformly, i.e.~$\pi_{k,0} = 1/K$ for all $k \in [K]$.

\paragraph{Initial covariance matrices.} In contrast, the covariance matrices play an important role in two ways. First, they determine the dispersion of the samples used to compute the gradients. Samples give information about the landscape of the objective function. If the eigenvalues of the covariance matrices are too small, then the algorithm will sample points only near the means, limiting exploration. Second, the update equation for component means in the natural gradient approach~\eqref{eq:natural_gradient_update_mu} implies that 
\begin{equation}
    \lVert \mub_{k, t+1} - \mub_{k, t} \rVert_2 = \frac{\rho_t}{\pi_{k,t}} \lVert \ESSb_{k, t+1}^{-1} \nabla_{\mub_k} \mathcal{L}_\omega(\Lambdab)|_{\Lambdab_t} \rVert_2.
    \label{eq:mean_jumps}
\end{equation}
One can remark that the size of the jump between $\mub_{k, t}$ and $\mub_{k, t+1}$ is regulated by the eigenvalues of $\ESSb_{k, t+1}^{-1}$. Since we have explained in Section~\ref{sub:asymptotic_behavior} how they are expected to vanish, this means that the algorithm can converge prematurely not because the means have reached stationary points, but because the eigenvalues have become too small. Picking larger initial eigenvalues for the covariance matrix helps avoid this issue. 

Without prior information on $\ell$, there is no reason not to pick isotropic Gaussian components. Therefore, we can reasonably set $\ESSb_{k,0}^{-1} = \sigma_0^2 \Ib$ for all $k \in [K]$, where $\sigma_0$, which should be large enough, can be picked by trial and error.

\paragraph{Initial means.} In classic gradient algorithms, such as gradient ascent or Newton--Raphson's method, the initial position of the particle is crucial because it may determine the (global or local) mode to which it will converge. Our algorithms also rely on gradients, but the entropy penalty in the objective function mitigates poor initialization by forcing separation between the components. Therefore, as long as there are enough components and with a well-chosen annealing schedule (Appendix~\ref{sub:hyper_annealing}) and learning rate sequence (Appendix~\ref{sub:learning_rate}), the choice of the initial means does not carry as much importance. 

\subsection{Annealing schedule}
\label{sub:hyper_annealing}

The temperature $(\omega_t)_{t \in [T]}$ controls the balance between exploration and concentration. At the beginning of the run, large values of $\omega_t$ amplify the entropy term and allow the components to spread across the landscape of $\ell$. Near convergence, $\omega_t$ should be small to ensure that the updates are dominated by the likelihood and that the means concentrate around the modes. The schedule must therefore introduce a smooth transition between these two regimes, and it should decrease slowly enough so that the dynamics at each $\omega_t$ can closely follow the corresponding fixed-$\omega$ solution path.

Because the appropriate scales for $\omega_1, \omega_T$, and the decay rate depend on the geometry of $\ell$ and on other hyperparameters such as the learning rate, we recommend using a smooth, monotonically decreasing parameterization. Throughout this work, we use schedules of the form $\omega_t = \omega_1 t^{-\alpha}$, where $\omega_1 > 0$ and $\alpha > 0$ determine, respectively, the initial temperature and the decay rate. Further details on how these parameters can be selected in practice are given in Appendix~\ref{app:hyperparameter_procedure}.

\subsection{Learning rate}
\label{sub:learning_rate}

The learning rate $(\rho_t)_{t \in [T]}$ plays a special role in our algorithm because it is narrowly linked with $\omega$ through the precision matrices of the mixture components. Indeed, this can be seen in equation~\eqref{eq:mean_jumps}. The dependence of the distance between $\mub_{k,t}$ and $\mub_{k, t+1}$ on $\rho_t$ and $\mathcal{L}_\omega(\Lambdab)|_{\Lambdab_t}$ is classic for gradient algorithms. The learning rate sets the proportionality constant $\rho_t$, while the gradient is necessary for convergence to the correct stationary point. However, here, we notice that the distance also depends on $\ESSb_{k, t+1}$.

In Appendix~\ref{app:entropy_approximation}, we have conjectured that $\ESSb^{*,\omega_t}_{k} \underset{t \rightarrow \infty}{\sim} \omega_t^{-1} (-\nabla^2_\xib \ell(\xib^*_k))$ where $\xib^*_k$ is the mode of $\ell$ such that $\mub^{*,\omega_t}_{k} \xrightarrow[t \rightarrow \infty]{} \xib^*_k$. But, near convergence of Algorithm~\ref{alg:snga_annealing_gaussian}, we have $\ESSb_{k,t} \approx \ESSb^{*,\omega_t}_{k}$, which means that 
\begin{equation*}
    \lVert \mub_{k, t+1} - \mub_{k, t} \rVert_2 \approx \frac{\rho_t \omega_t}{\pi_{k,t}} \lVert (-\nabla^2_\xib \ell(\xib^*_k))^{-1} \nabla_{\mub_k} \mathcal{L}_\omega(\Lambdab)|_{\Lambdab_t} \rVert_2.
\end{equation*}
Since $\omega_t \xrightarrow[t \rightarrow \infty]{} 0$, it slows the trajectory of $\mub_{k,t}$ down. As a result, our experiments have shown that with a constant learning rate, the decay in $\omega_t$ can occasionally lead to premature stagnation, therefore preventing $\mub_{k,t}$ from reaching the correct value.

To compensate for this effect, we consider learning-rate schedules of the form $\rho_t = \omega_t^{-\beta}$, where $0 < \beta < 1$. In classical gradient ascent (GA), it is typically recommended to use a decreasing learning rate to facilitate convergence, as a constant learning rate does not guarantee convergence. In our algorithm, setting $\beta = 1$ effectively corresponds to a constant learning rate in GA, since $\rho_t$ fully compensates for the effect of $\omega_t$ in the updates of $\mub_k$. Therefore, we suggest selecting a value of $\beta$ below $1$ to partially offset the impact of $\omega_t$. 

\subsection{Practical selection procedure}
\label{app:hyperparameter_procedure}

Conceptually, annealing can be viewed as solving a sequence of fixed-$\omega$ optimization problems whose solutions form a continuous path as $\omega$ decreases. For natural-gradient ascent to follow this path reliably, the temperature should decrease slowly enough so that two conditions hold. First, the decay of $\omega_t$ must not outpace the convergence of the fixed-$\omega$ dynamics, which depends on curvature, learning rate, and initialization. Second, sufficiently large temperatures encourage exploration by amplifying the entropy term, allowing components to spread, cross valleys, and avoid local traps. Because these effects depend on problem-specific features such as the number of modes, the shape of $\ell$, and the geometry of attraction basins, we have not found any universal rule guaranteeing that a schedule is slow enough.

Extremely slow schedules are also impractical, since they require many iterations. With a fixed budget $T$, two constraints become essential: the initial temperature $\omega_1$ must be large enough to promote exploration, and the final temperature $\omega_T$ must be small enough to allow concentration near the modes. In practice, we use the parameterization $\omega_t = \omega_1 t^{-\alpha}$, characterized by $\omega_1$ and the exponent $\alpha$, and determine these parameters through a small amount of trial-and-error steps.

\paragraph{Choosing $\omega_1$.}
We begin with a constant temperature schedule (e.g. $\omega_t = \omega_1$) and monitor the component mean updates during the first 10-100 iterations. The value $\omega_1$ should be large enough for entropy repulsion to dominate early in the run, but not so large that the means diverge. This procedure yields an appropriate order of magnitude for $\omega_1$, typically after 5-10 short runs. The same runs also help define an initial learning-rate schedule $(\rho_1, \beta)$ by ensuring that the means exhibit visible, stable updates (start with $\beta = 0.9$).

\paragraph{Choosing $\alpha$.} We then perform several short test runs (5-10 runs of at most $T$ iterations) with early stopping. Starting with $\alpha=1$, if exploration persists for too long, meaning that means are still drifting at iteration $T$, we increase $\alpha$. If the runs converge too quickly or display instability across random initializations (such as mode collapse due to insufficient exploration), we decrease $\alpha$ or refine $\omega_1$. 

\paragraph{Fine-tuning $(\rho_1, \beta)$, iteration budget increase.} The magnitude of the mean updates during these runs also informs fine adjustments to the learning-rate schedule. If neither adjustments produce stable behavior, increasing the total number of iterations $T$ allows for slower annealing without altering the start or end temperature. Increasing the learning rate, while preserving stability, can also accelerate the fixed-$\omega$ dynamics and reduce the need for excessively slow schedules.

\section{Gradient computation and estimation in FS-NVA-GM}
\label{app:gradient_computation_fsnvagm}

In this section, we derive the natural gradient estimators needed in FS-NVA-GM, the fitness-shaped version of NVA-GM.

Following the derivation approach in Appendix~\ref{app:natural_gradient_update_rule_gaussian}, we obtain the following update rules for $\mub_k$ and $\ESSb_k$
\begin{align*}
    \ESSb_{k, t+1} &= \ESSb_{k, t} - \frac{2 \rho_t}{\pi_{k,t}} \nabla_{\ESSb^{-1}_k} \mathcal{L}_{W_{\lambdab_{k,t}}^{(\omega_t, \Lambdab_t)}}(\lambdab_k)|_{\lambdab_{k,t}}, \\
    \mub_{k, t+1} &= \mub_{k, t} + \frac{\rho_t}{\pi_{k,t}} \ESSb_{k, t+1}^{-1} \nabla_{\mub_k} \mathcal{L}_{W_{\lambdab_{k,t}}^{(\omega_t, \Lambdab_t)}}(\lambdab_k)|_{\lambdab_{k,t}}.
\end{align*}
For these updates, like in Appendix~\ref{app:expression_grad_mu_s}, Gaussian identities yield the following gradient expressions
\begin{align*}
        \nabla_{\ESSb^{-1}_k} \mathcal{L}_{W_{\lambdab_{k,t}}^{(\omega_t, \Lambdab_t)}}(\lambdab_k) &= \frac{\pi_{k,t}}{2} \mathbb{E}_{\mathcal{N}(\mub_k, \ESSb^{-1}_k)}\left[(\ESSb (\xib - \mub_k) (\xib - \mub_k)^T \ESSb_k - \ESSb_k) W_{\lambdab_{k,t}}^{(\omega_t, \Lambdab_t)}(f_{\omega_t}(\xib; \Lambdab_t))\right], \\
        \nabla_{\mub_k} \mathcal{L}_{W_{\lambdab_{k,t}}^{(\omega_t, \Lambdab_t)}}(\lambdab_k) &= \pi_{k,t} \mathbb{E}_{\mathcal{N}(\mub_k, \ESSb^{-1}_k)}\left[\ESSb (\xib - \mub_k) W_{\lambdab_{k,t}}^{(\omega_t, \Lambdab_t)}(f_{\omega_t}(\xib; \Lambdab_t))\right] . 
\end{align*}

Finally, for an arbitrary mini-batch size $B$, let $\ub = (u_b)_{b \in [B]}$ be the utility values associated to $w$ of~\eqref{eq:define_rewriting_multi}. The gradients can be estimated by sampling $B$ locations $\xib_1, \dots, \xib_B \overset{\text{i.i.d.}}{\sim} q_{\lambdab_{k,t}} = \mathcal{N}(\mub_{k,t}, \ESSb_{k,t}^{-1})$ and sorting them as $\xib_{(1)}, \dots, \xib_{(B)}$ such that $f_{\omega_t}(\xib_{(1)}, \Lambdab_t) \ge \dots \ge f_{\omega_t}(\xib_{(B)}, \Lambdab_t)$. According to Theorem~6 of \cite{ollivier2017information}, consistent estimators for the gradients are $\widehat{\nabla_{\mub_k} \mathcal{L}}_{W_{\lambdab_k}^{(\omega_t, \Lambdab_t)}}(\lambdab_k)|_{\lambdab_{k,t}} = \pi_{k,t} \widehat{\nub}^{(\mub_k)}_{\omega_t,\Lambdab_t,\ub}$ and $\widehat{\nabla_{\ESSb^{-1}_k} \mathcal{L}}_{W_{\lambdab_k}^{(\omega_t, \Lambdab_t)}}(\lambdab_k)|_{\lambdab_{k,t}} = \pi_{k,t} \widehat{\nub}^{(\ESSb_k^{-1})}_{\omega_t,\Lambdab_t,\ub}/2$, where:
\begin{align}
    \widehat{\nub}^{(\mub_k)}_{\omega_t,\Lambdab_t,\ub} &:= \frac{1}{B} \ESSb_{k,t} \sum_{b = 1}^B (\xib^{(k)}_{(b)} - \mub_{k,t}) u_b, \tag{$G_{\mub}$-fs} \label{eq:estimator_nu_mu} \\
    \widehat{\nub}^{(\ESSb_k^{-1})}_{\omega_t,\Lambdab_t,\ub} &:= \frac{1}{B} \ESSb_{k,t} \sum_{b = 1}^B \left( (\xib^{(k)}_{(b)} - \mub_{k,t})(\xib^{(k)}_{(b)} - \mub_{k,t})^T \ESSb_{k,t} - \Ib \right) u_b. \tag{$G_{\Sigmab}$-fs} \label{eq:estimator_nu_s}
\end{align}

\section{Choice of utility values}
\label{sub:utility_values}

In this section, we discuss the choice of utility values $\ub$, or alternatively, of a weighting scheme $w$, in the FS-NVA-M (or FS-NVA-GM) algorithm, in relation to the literature in evolutionary algorithms. 

In most implementations of the NES or the IGO framework, the choice of the utility values is typically arbitrary. While \cite{wierstra2014natural} empirically observed that this choice does not significantly impact algorithm performance, \cite{beyer2014convergence} demonstrated that it could strongly influence the behavior of the optimization process by analyzing the ordinary differential equation governing the dynamics of these algorithms. However, this impact is dependent on the specific fitness function, making it impossible to generalize. In the absence of prior information, defining an optimal $\ub$ or $w$ is not feasible. Below, we present several common choices for $w$, or alternatively $\ub$.  

A natural selection rule in evolutionary algorithms is truncation selection, which involves selecting a fixed fraction of the current population, representing the samples with the highest fitness, to compute the next generation. Typically, this corresponds to the weighting scheme $w(x) = \mathds{1}_{\{x \le \eta \}}$, where $\eta \in (0,1)$ is the fraction of the population kept. If the population size $B$ is fixed, then the number of samples selected $B_0 = \lfloor B \eta + 1/2 \rfloor$ is constant and the equivalent utility values are 
\begin{equation*}
    u_b = \begin{cases}
        B/B_0 & \text{for~} 1 \le b \le B_0, \\
        0 & \text{for~} B_0 + 1 \le b \le B.
    \end{cases} 
\end{equation*}
Consequently, the choice of $\eta$ determines the number of samples effectively used in the gradient computations. If $\eta$ is too large, the algorithms may lead to premature convergence. \cite{ollivier2017information} suggest that $\eta$ should generally be less than $1/2$, confirming the theoretical findings of \cite{beyer2001theory}, \cite{jebalia2010log} and \cite{ beyer2014convergence}, who identified optimal values for specific examples of fitness functions. From a computational perspective, a lower $\eta$ implies that fewer samples are used (in fact, $\eta$ must be at least $0.5/N$ to select to ensure that at least one sample is selected), even though these samples have higher fitness. Thus, selecting a low $\eta$ can reduce the computational cost but at the expense of increased variance in the gradient computations. 

While truncation selection assigns equal weight to each retained sample, a variation of this selection scheme assigns rank-based weights, placing greater importance on the samples with the highest fitness. Notably, the weights used in the widely used CMA-ES algorithm~\citep{hansen2001completely} can be employed as utility values \citep{wierstra2014natural}. They are defined by
\begin{equation}
    u_b = \begin{cases}
        B\frac{\log(B_0 + 1) - \log(b)}{\sum_{c=1}^{B_0}(\log(B_0 + 1) - \log(c))}  & \text{for~} 1 \le b \le B_0, \\
        0 & \text{for~} B_0 + 1 \le b \le B.
    \end{cases} 
    \label{eq:cma_es_utility}
\end{equation}

Departing from truncation selection, it is important to note that utility values can also be negative. For instance, \cite{arnold2006weighted} proved that the optimal utility values for the spherical fitness function $f(\xib) = - \lVert \xib^* - \xib \rVert^2_2$ are given by
\begin{equation*}
    u_b = \mathbb{E}[Z_{(b)}],
\end{equation*}
where $Z_{(1)} \ge \dots \ge Z_{(B)}$ are the order statistics of the vector $(Z_{1}, \dots, Z_{B})$ with $Z_b \overset{\text{i.i.d.}}{\sim} \mathcal{N}(0,1)$, for all $b \in [B]$.

Finally, \cite{ollivier2017information} discussed the fact that any weighting scheme $w$, or vector of utility values $\ub$, can be shifted by an additive constant $c$. While this does not change the expression of the gradient, it affects its estimation by introducing an additional term with a null expectation. Therefore, although $c$ does not influence the expectation of the gradient estimator, it can be optimized, for example, to minimize variance.

\section{Annealing schedule sensitivity}
\label{app:annealing_sensitivity}

We assess the sensitivity of the decay rate in the annealing schedule for the symmetric triangle mixture of Section~\ref{sub:simu_mode}. We consider schedules of the form $\omega_t = \omega_1/t^\alpha$, parameterized by $(\omega_1, \alpha, T)$. As a reference, we use the same hyperparameters as in Section~\ref{sub:simu_mode}: the iteration budget is $T^{(0)} = 5000$ and the annealing schedule is $\omega^{(0)}_t := \omega^{(0)}_1/t$ with $\omega^{(0)}_1 = 1$, which gives $\omega^{(0)}_{T^{(0)}} = 2 \cdot 10^{-4}$. 

\paragraph{Scenarios.} We study two scenarios by varying the decay exponent $\alpha$ in the interval $[0.8, 1.5]$:
\begin{enumerate}
    \item \emph{~Fixed end value.} We set $\omega_1 = \omega^{(0)}_1$ and $\omega_T = \omega^{(0)}_{T^{(0)}}$. For each value of $\alpha$, the iteration budget $T$ is chosen so that the schedule reaches the prescribed end value. This leads to large values when $\alpha$ is small. For example, $\alpha = 0.8$ gives $T \approx 42000$.
    \item \emph{~Fixed iteration budget.} We set $\omega_1 = \omega^{(0)}_1$ and $T=T^{(0)}$. The end value becomes $\omega_T = \omega^{(0)}_1/(T^{(0)})^\alpha$, which varies with $\alpha$.
\end{enumerate}

\paragraph{Metrics.} For both scenarios we report the GPR and APR metrics of Section~\ref{sub:simu_mode} in Figure~\ref{fig:mode_sym_alpha}. We also report a collapse rate. For each run, we count the number of components that have converged to a mode within tolerance $\epsilon = 10^{-2}$ and subtract the number of distinct modes recovered, and we average this quantity over the $H$ runs used for GPR and APR. Figure~\ref{fig:collapse_sym} shows the collapse rate for the second scenario, since it is equal to zero in the first.

\paragraph{Results.} In the first scenario, the metrics remain stable when $\alpha$ lies in the interval $[0.8, 1.2]$. When $\alpha \ge 1.3$, the metrics degrade because the corresponding budgets $T$ are small and the components have not converged. 

In the second scenario, convergence is incomplete when $\alpha < 1.0$ because the schedule decays too slowly with the fixed budget $T^{(0)}$. When $\alpha \in [1.0, 1.2]$, the components converge and the collapse rate is negligible (except for $K=5$, because there are only 4 modes). When $\alpha \ge 1.3$, the collapse rate increases with $\alpha$, which explains the decrease in GPR and APR.

\paragraph{Conclusion.} Scenario 1 shows that a wide range of exponents can achieve good performance when the iteration budget is adjusted accordingly. Scenario 2 shows that large exponents lead to component collapse and degraded metrics, while moderate values allow ideal convergence. Based on the trends observed in Figure~\ref{fig:mode_sym_alpha}, $\alpha = 1$ or $1.1$ provides a good balance between performance and computational cost. 

\begin{figure}[!tb]
\centering
\includegraphics[width=0.9\linewidth]{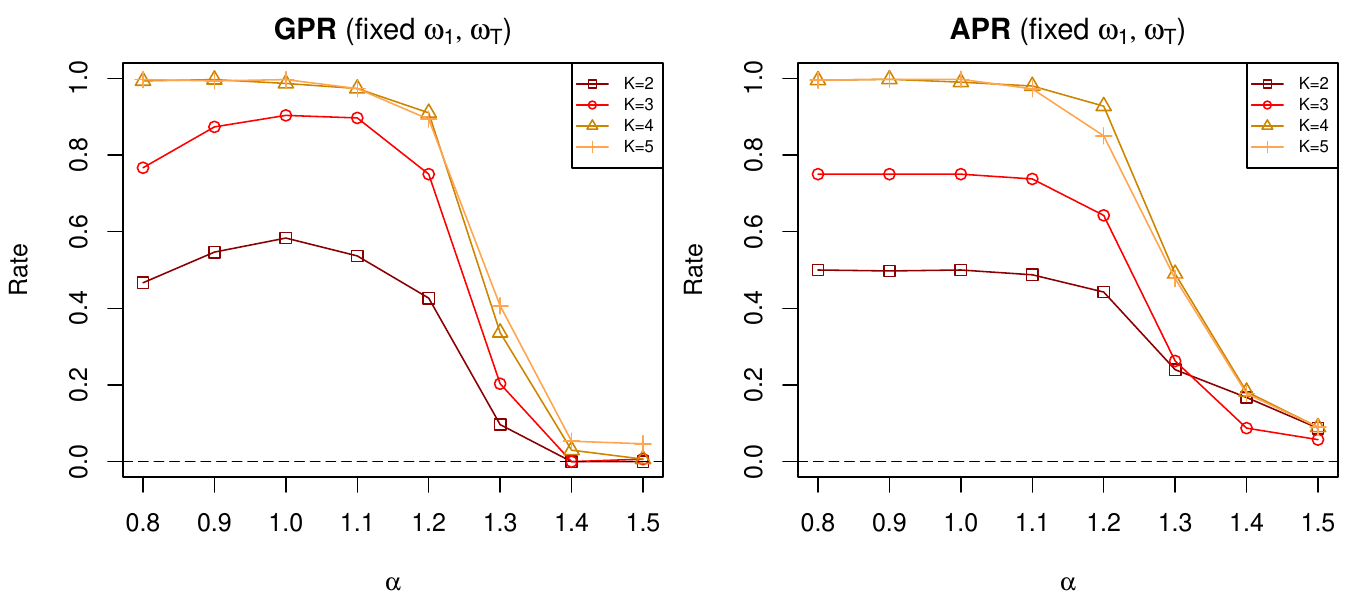}
\\
\includegraphics[width=0.9\linewidth]{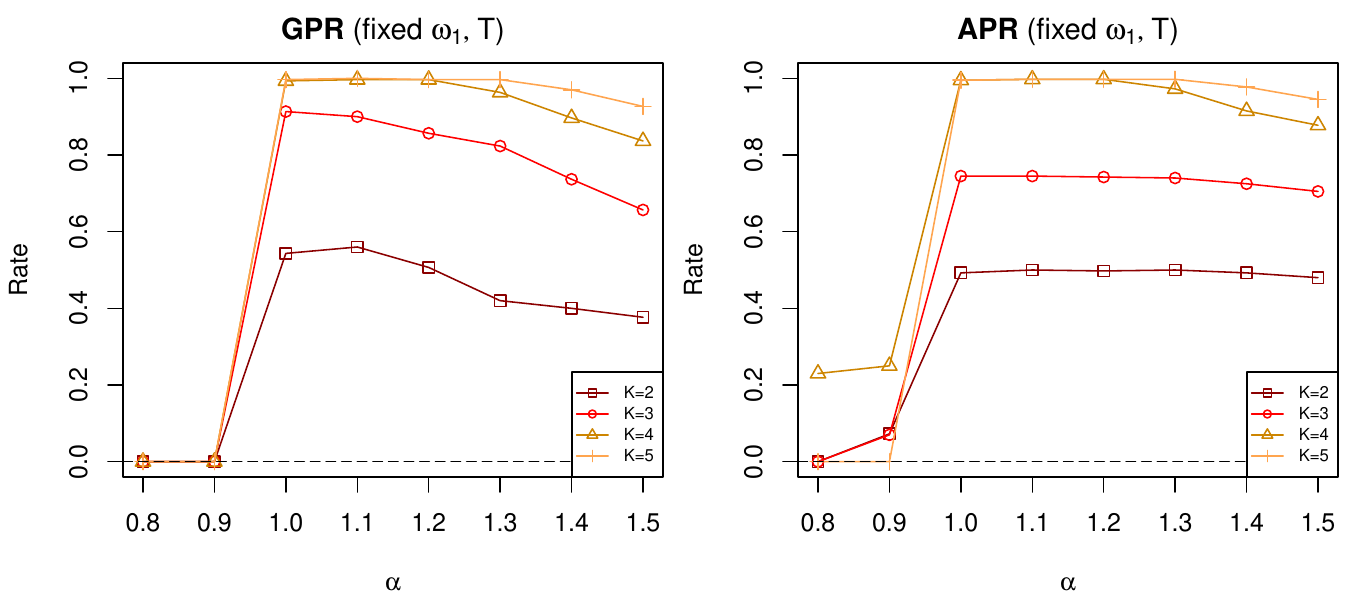}
\caption{Sensitivity analysis of the annealing schedule on the symmetric mixture example. Top row: When the iteration budget is set to keep $\omega_T = 2 \cdot 10^{-4}$ constant, the metrics sharply drop for $\alpha > 1.2$. Bottom row: On a fixed iteration budget ($T = 5000$), the graphs show a sharp drop for $\alpha < 1.0$.}
\label{fig:mode_sym_alpha}
\end{figure}

\begin{figure}[!tb]
\centering
\includegraphics[width=0.5\linewidth]{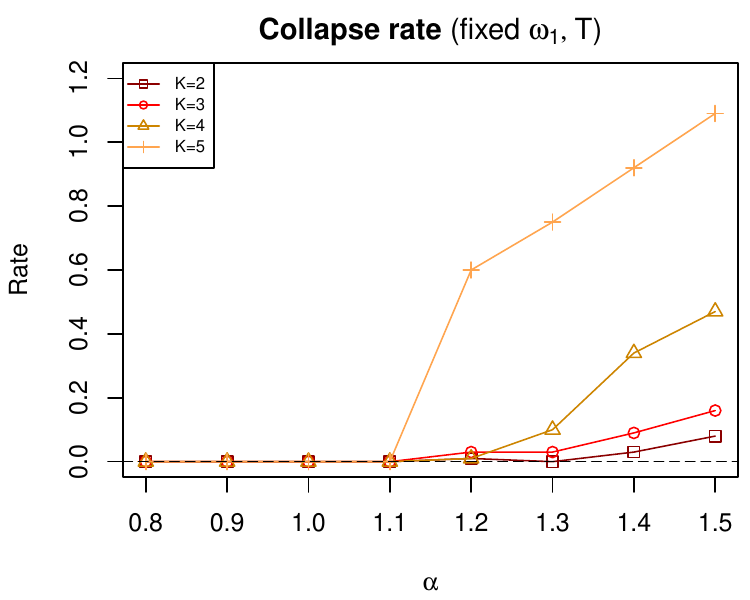}
\caption{Collapse rate on the symmetric mixture example.
When $K$ is smaller than the number of modes (4), collapse is observed for $\alpha > 1.2$.}
\label{fig:collapse_sym}
\end{figure}

\section{Simulations with degenerate modes} 
\label{app:simu_degenerate}

In this section, we present simulation results showcasing the behavior of the mixture weights in NVA-GM when $\ell$ has a degenerate mode and a non-degenerate mode. We set $d = 2$ and $\ell(\xib) = \psi(\xi_1) (\xi_2^2 + 1)$ where
\begin{equation*}
    \psi(\xi) = \begin{cases}
        - (\xi + 3)^4 - 1 & \text{for } \xi < -2, \\
        - \frac{1}{8} \xi^3 + \frac{3}{4} \xi^2 + \frac{1}{2} \xi - 5 & \text{for } -2 \le \xi \le 2, \\
        - (\xi - 3)^2 - 1 & \text{for } 2 < \xi.
    \end{cases}
\end{equation*}
Here, $\ell$ has two modes, which are both global, at $\xib^*_1 = (-3, 0)$ and $\xib^*_2 = (3, 0)$. However, we have $\det(-\nabla^2 \ell(\xib^*_1)) = 0$ and $\det(-\nabla^2 \ell(\xib^*_2)) > 0$. Therefore, the mode at $\xib^*_1$ is degenerate whereas the one at $\xib^*_2$ is not. 

We apply NVA-GM with $K = 2$, using $T = 50$, $\pi_{1,0} = \pi_{2,0} = 1/2$, $\ESSb_{1, 0}^{-1} = \ESSb_{2, 0}^{-1} = \Ib$, $\omega_t = \omega_1/t^2$ with $\omega_1 = 0.1$, and $\rho_t = 10^{-1}(\omega_1/\omega_t)^{0.8}$. For the gradients, we use the estimators $\widehat{\gammab}^{(\mub_k, 1)}_{\omega,\Lambdab,B}$ and $\widehat{\gammab}^{(\ESSb_k^{-1}, 2)}_{\omega,\Lambdab,B}$, where $B = 4$. We recover the positions of the component means and the value of the weights at each iteration.

Figure~\ref{fig:weight_degen} illustrates the trajectories of the component means and the evolution of the mixture weights during a run where each mode is successfully identified by a distinct mean. Although Theorem~\ref{th:annealing_limit} is not directly applicable to functions with degenerate modes, in the subsequent discussion in Appendix~\ref{sub:interpretation_weights}, we have provided an interpretation that degenerate modes are ``infinitely'' flatter than non-degenerate ones. Consequently, in this run of NVA-GM, it is not surprising that all the weight is captured by component 1, which is the one converging to the degenerate mode.

\begin{figure}[!tb]
\centering
\includegraphics[width=0.9\linewidth]{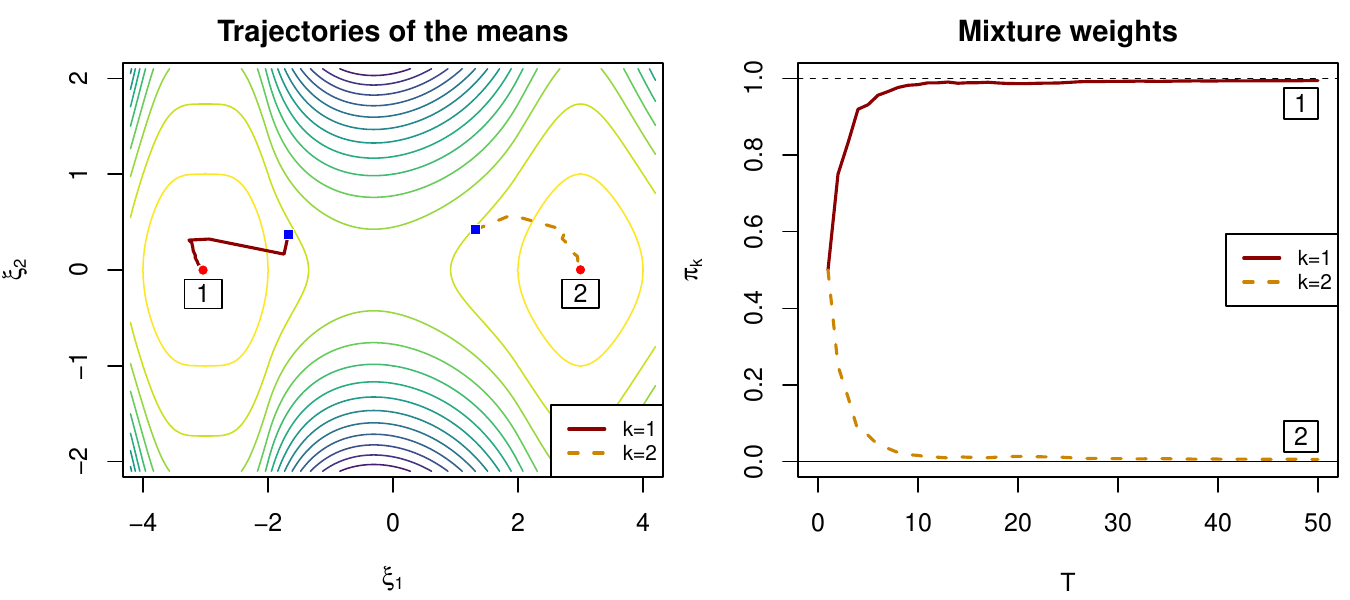}
\caption{Left: The trajectories of the means in a run of NVA-GM on the function with a degenerate mode and a non-degenerate mode with $K = 2$ show that the two components track different modes. The contour plot represents non-equally spaced levels of $\ell$. Right: As expected, the weights of the component tracking the degenerate modes (1) converge to $1$, whereas the weight of the component tracking the non-degenerate mode (2) vanishes.}
\label{fig:weight_degen}
\end{figure}

\section{Black-box optimization benchmark}
\label{app:cec_benchmark}

To our knowledge, the only established benchmarks for multimodal optimization are designed for black-box optimization methods, using only the objective function evaluation. For completeness, we thus illustrate the performance of FS-NVA-GM, the black-box setting of which does not resort to gradients or Hessians. The FS-NVA-GM algorithm is applied to several CEC2013 benchmark functions for multimodal optimization \citep{li2013benchmark}, which have been used in both CEC and GECCO niching competitions. Although our framework is designed to locate all modes, the benchmark suite targets black-box multimodal global optimization. The benchmark assesses the ability to consistently find all the global modes of a set of test functions, using only function evaluations. 

We use the first six functions of the benchmark suite. Many of these functions exhibit local optima. The function domains of the benchmark suite are constrained to rectangles of the form $\mathcal{D}_{\ab, \bb} = \prod_{i=1}^d [a_i, b_i]$, where $d$ is the input variable dimension. As this is not our current illustrative purpose, to bypass the need for constrained optimization, we extend all functions to the entire space $\mathbb{R}^d$ as follows:
\begin{equation*}
    \tilde{f}(\xib) = f(\ab + \rb(\xib)) - \lvert \qb(\xib) \rvert_1 A(f),
\end{equation*}
where 
\begin{itemize}
    \item $\rb(\xib) = (r_i(\xi_i))_{1 \le i \le d}$ and $\qb(\xib) = (q_i(\xi_i))_{1 \le i \le d}$ and $(q_i(\xi_i), r_i(\xi_i))$ are the quotient and the remainder of the (generalized) Euclidean division of $\xi_i$ by $(b_i - a_i)$, i.e. $\xi_i = q_i(\xi_i) (b_i - a_i) + r_i(\xi_i)$ with $0 \le r_i(\xi_i) < (b_i - a_i)$ and $q_i(\xi_i) \in \mathbb{Z}$, for all $1 \le i \le d$,
    \item $\lvert \qb(\xib) \rvert_1 = \sum_{i = 1}^d \lvert q_i(\xib) \rvert$,
    \item $A(f) = \max_{\xib \in \mathcal{D}}~f(\xib) - \min_{\xib \in \mathcal{D}}~f(\xib)$.
\end{itemize}
The modified function $\tilde{f}$ is a ``pyramidal" extension of $f$ on the whole space $\mathbb{R}^d$. This construction partitions $\mathbb{R}^d$ into a grid of rectangular regions, each with the same dimensions as the original domain $\mathcal{D}_{\ab, \bb}$. In each of these regions, the function $f$ is copied but shifted down by a constant, which increases with the distance of that region from the original domain $\mathcal{D}_{\ab, \bb}$. This construction preserves the global modes as $f$ in $\mathcal{D}$ but introduces an infinite number of new local modes due to the repeated, offset copies of $f$ across the extended domain.

We fix an accuracy level $\epsilon = 0.1$, a threshold of the target function below which a mode is considered found. To evaluate our algorithms, we use the global peak ratio (GPR) and global success rate (GSR) as advised in \cite{li2013benchmark}. The GPR has been defined earlier by~\eqref{eq:gpr_metric}. The GSR is given by
\begin{equation*}
    \text{GSR} = \frac{1}{ H } \sum_{h = 1}^{H} \text{AGF}_h, 
\end{equation*}
where $H$ is the number of runs, and for each run $h$, $\text{AGF}_h = 1$ if all target global modes are found, and $0$ otherwise. 

We use FS-NVA-GM with CMA-ES utility values, as given by~\eqref{eq:cma_es_utility}. We made slight modifications to FS-NVA-GM. First, to prevent numerical instability, we enforce a lower bound $\tau$ on the eigenvalues of the covariance matrices in the search mixture. This is done by adding the following step after each precision matrix update in FS-NVA-GM:
\begin{equation*}
    \ESSb_{k, t+1} \leftarrow (\ESSb_{k, t+1}^{-1} + \tau \Ib)^{-1}.
\end{equation*}
The hyperparameter $\tau$ can be seen as a damping factor and can be fixed to a small value, $\tau = 10^{-10}$ in our experiments. Second, we observed that skipping updates of the precision matrices for a few initial iterations can improve performance. We introduce an additional hyperparameter $\kappa$, which defines a burn-in period, i.e. the number of initial iterations before precision matrix updates begin. Mathematically, these first $\kappa$ iterations correspond to solving the variational problem in the family of Gaussian mixtures with fixed covariance matrices. After $\kappa$ iterations, covariances are allowed to be updated. For each run, the initial means are initialized uniformly at random in a rectangle for each function, containing all global modes, specified in \cite{li2013benchmark}, of the form $\mathcal{D}_{\ab, \bb} = \prod_{i=1}^d [a_i, b_i]$. The initial covariance matrices are set to $\sigma_0^2 \Ib$, where $\sigma_0^2 = \lVert \mathcal{D}_{\ab, \bb}/2 \rVert^2_\infty = \max_{i \in \{1, \dots, d\}}~(b_i - a_i)^2/4$ and the initial mixture weights are set to $1/K$. The annealing schedules are defined as $\omega_t = \omega_1 t^{-\alpha}$, and the learning rates are given by $\rho_t = \rho_1 (\omega_1/ \omega_t)^{\beta}$, where $\omega_1, \rho_1, \alpha, \beta$ are hyperparameters specified in Table~\ref{tab:hyper_benchmark_cec}.

\begin{table}
 \caption{Number of global modes $I$ and hyperparameters used for the first six functions of the CEC2013 benchmark ($F_1$ to $F_6$).}
    \setlength{\tabcolsep}{10pt}
\begin{tabular*}{\textwidth}{@{\extracolsep{\fill}} l c  c c c c c c c c }
\toprule
 Function & $I$ & $T$ & $K$ & $B$ & $\omega_1$ & $\alpha$ & $\rho_1$ & $\beta$ & $\kappa$ \\ 
\midrule
  $F_1$ & 2 & 500 & 2 & 16 & 100000 & 2 & $10^{-3}$ & 0.8 & 0 \\   
  $F_2$ & 5 & 2000 & 5 & 32 & 20 & 1 & $10^{-3}$ & 0.9 & 0 \\ 
  $F_3$ & 1 & 2000 & 1 & 32 & 20 & 1 & $10^{-3}$ & 0.9 & 0 \\ 
  $F_4$ & 4 & 2000 & 4 & 16 & 2000000 & 1.8 & $10^{-4}$ & 0.7 & 50 \\ 
  $F_5$ & 2 & 2000 & 2 & 16 & 10000 & 2 & $10^{-5}$ & 0.8 & 0 \\ 
  $F_6$ & 18 & 2000 & 18 & 16 & 1000000 & 1.8 & $10^{-5}$ & 0.8 & 50 \\ 
\bottomrule
\end{tabular*}
\label{tab:hyper_benchmark_cec}
\end{table}

In the CEC2013 competition guidelines, a maximum budget in the number of function evaluations is allocated for each target function, which is $50~000$ for functions $F_1$ to $F_5$, and $200~000$ for $F_6$. For FS-NVA-GM, the number of evaluations performed during a run is $f\text{-eval} = BKT$. $K$ is set to the number of global modes to find. Increasing $K$ improves the performance metrics but raises $f$-eval. We use $B$ and $T$ to control performance within a reasonable value for $f$-eval, slightly exceeding the budget specified in the CEC2013 competition. Under this setting, the benchmark results (Table~\ref{tab:results_benchmark_cec}) show that FS-NVA-GM performs well for simple functions ($F_{1}$ to $F_{5}$), locating all global modes in almost every run. For $F_6$ (Shubert 2D function), which is highly multimodal with $18$ global and many local modes, FS-NVA-GM finds on average $0.749 \times 18 \approx 13.5$ modes per run with our chosen hyperparameters.

\begin{table}
 \caption{Number of evaluations of the target function and evaluation metrics GPR and GSR (the closer to 1 the better) for $H = 50$ runs of FS-NVA-GM with the hyperparameters specified in Table~\ref{tab:hyper_benchmark_cec}.}
    \setlength{\tabcolsep}{26pt}
\begin{tabular*}{\textwidth}{ @{\extracolsep{\fill}}l c c c  }
\toprule
 Function & $f$-eval & GPR & GSR  \\ 
\midrule
  $F_1$ & 16000 & 1.000 & 1.000 \\   
  $F_2$ & 320000 & 0.996 & 0.980  \\ 
  $F_3$ & 64000 & 1.000 & 1.000  \\ 
  $F_4$ & 130000 & 0.990 & 0.960  \\ 
  $F_5$ & 64000 & 1.000 & 1.000  \\ 
  $F_6$ & 580000 & 0.749 & 0.000  \\ 
\bottomrule
\end{tabular*}
\label{tab:results_benchmark_cec}
\end{table}

The results are competitive with some black-box methods, e.g. Tables 2 and 3 of \cite{li2013benchmark}. However, we expect FS-NVA-GM not to be at the same level for more complex benchmark functions, as the NVA-M framework has been primarily designed to exploit first and second-order information. The black-box algorithm FS-NVA-GM may then be too simple in its current tested implementation, and in particular, limited by the allowed number of function calls.

\end{appendix}

\end{document}